\documentclass[runningheads]{llncs}
\pdfoutput=1

\usepackage{hyperref}
\usepackage[nocompress]{cite}

\usepackage{amsmath, amssymb}
\DeclareSymbolFontAlphabet{\amsmathbb}{AMSb}

\usepackage{graphicx}
\usepackage{cleveref}
\usepackage{mleftright}
\usepackage{cases}
\usepackage{empheq}
\usepackage{nicefrac}
\usepackage{stmaryrd}
\usepackage{bm}
\usepackage{bbm}
\usepackage{mathtools}
\usepackage{listings}
\usepackage[colorinlistoftodos]{todonotes}
\usepackage{caption}
\usepackage{subcaption}
\usepackage{tabularx}
\usepackage{booktabs}
\usepackage{scrextend}
\usepackage{threeparttable}
\usepackage{xspace}
\usepackage{siunitx}
\usepackage{multirow}

\usepackage{enumitem}
\usepackage{xspace}
\usepackage{microtype}

\usepackage{tikz} %
\usepackage{pgfplots} %
\usetikzlibrary{pgfplots.groupplots}
\usepgfplotslibrary{external,fillbetween,colorbrewer}
\DeclareUnicodeCharacter{2212}{−}
\usetikzlibrary{patterns,patterns.meta,shapes.arrows,calc,backgrounds,decorations,decorations.markings,angles,quotes}
\pgfplotsset{compat=newest}

\pgfplotsset{compat=1.11, width=10.5cm, height=7cm,
/pgfplots/ybar legend/.style={
    /pgfplots/legend image code/.code={%
        \draw[##1,/tikz/.cd,yshift=-0.25em]
        (0cm,0cm) rectangle (3pt,0.8em);},},
}

\definecolor{red}{rgb}{0.745,0.192,0.102}
\definecolor{darkgreen}{RGB}{34,161,55}
\definecolor{ruhuisstijlrood}{rgb}{0.745,0.192,0.102}
\definecolor{ruhuisstijlzwart}{rgb}{0,0,0}
\definecolor{ruhuisstijlwit}{rgb}{0.98,0.98,0.98}

\definecolor{plotblue}{rgb}{0.1,0.498039215686275,0.9549019607843137}

\spnewtheorem{assumption}{Assumption}{\bfseries}{\itshape}
\spnewtheorem{problemBold}{Problem}{\bfseries}{\itshape}

\crefname{equation}{Eq.}{Eqs.}
\crefname{pluralequation}{Eqs.}{Eqs.}

\crefname{algorithm}{Algorithm}{Algorithm}

\crefname{figure}{Fig.}{Figs.}
\crefname{pluralfigure}{Figs.}{Figs.}

\crefname{section}{Sect.}{Sects.}
\crefname{pluralsection}{Sects.}{Sects.}

\crefname{table}{Table}{Tables}
\crefname{pluraltable}{Tables}{Tables}

\crefname{definition}{Def.}{Defs.}
\crefname{pluraldefinition}{Defs.}{Defs.}

\crefname{theorem}{Theorem}{Theorems}
\crefname{pluraltheorem}{Theorems}{Theorems}

\crefname{lemma}{Lemma}{Lemmas}
\crefname{plurallemma}{Lemmas}{Lemmas}

\crefname{example}{Example}{Examples}
\crefname{pluralexample}{Examples}{Examples}

\crefname{problem}{Problem}{Problems}
\crefname{pluralproblem}{Problems}{Problems}

\crefname{problemBold}{Problem}{Problems}
\crefname{pluralproblemBoldm}{Problems}{Problems}

\crefname{assumption}{Assumption}{Assumptions}
\crefname{pluralassumption}{Assumptions}{Assumptions}

\crefname{remark}{Remark}{Remarks}
\crefname{pluralremark}{Remarks}{Remarks}

\crefname{proposition}{Proposition}{Propositions}
\crefname{pluralproposition}{Propositions}{Propositions}

\crefname{property}{Property}{Properties}
\crefname{pluralproperty}{Properties}{Properties}

\crefname{corollary}{Corollary}{Corollaries}
\crefname{pluralcorollary}{Corollaries}{Corollaries}

\crefname{appendix}{App.}{Appendices}
\crefname{pluralappendix}{Appendices}{Appendices}

\DeclareMathOperator*{\argmax}{arg\,max}

\usepackage{fontawesome5}

\usepackage{algorithm}
\usepackage[noend]{algpseudocode}

\algrenewcommand\algorithmicindent{1.0em}%

\newcommand{\relu}{\mathrm{ReLU}}

\newcommand{\satprob}{\mathrm{Pr}}
\newcommand{\Prob}{\mathbb{P}}
\newcommand{\Exp}{\mathbb{E}}
\newcommand{\R}{\mathbb{R}}

\newcommand{\Nzero}{\mathbb{N}_0}

\newcommand{\weight}{w}
\newcommand{\weightsys}{\mathcal{W}}
\newcommand{\Id}{\mathrm{Id}}

\newcommand{\loss}{\mathcal{L}}
\newcommand{\tuple}[1]{\langle #1 \rangle}
\newcommand{\Vmin}{V_{\mathrm{min}}}
\newcommand{\Vmax}{V_{\mathrm{max}}}
\newcommand{\Vlb}{V_{\mathrm{LB}}}
\newcommand{\Vub}{V_{\mathrm{UB}}}
\newcommand{\cell}{\mathrm{cell}}
\newcommand{\suggmesh}{\lambda}

\newcommand*{\nn}{\ensuremath{\mathcal{A}}}
\newcommand{\network}{{T^{\nn}}}

\newcommand{\lipnetwork}{L_{\nn, \weightsys}}

\newcommand{\weightsysalt}{\widetilde{\mathcal{W}}}
\newcommand{\Kalt}{\widetilde{K}}
\newcommand{\lipnetworkalt}{L_{\nn, \weightsysalt}}

\newcommand{\networkleqn}{{T^{\nn}_{\leq n}}}
\newcommand{\networkleqnprime}{T^{\nn'}_{\leq n}}
\newcommand{\weightalt}{\widetilde{\weight}}
\newcommand{\wsysalt}{\weightsysalt}

\newcommand*{\dtss}{\ensuremath{\mathcal{S}}}
\newcommand*{\DTSS}{\ensuremath{\tuple{\X, \X_0, \U, \Noise, \noisedist, f}}}
\newcommand{\X}{\mathcal{X}}
\newcommand{\Xdisc}{\widetilde{\mathcal{X}}}
\newcommand{\U}{\mathcal{U}}
\newcommand{\Noise}{\mathcal{N}}
\newcommand{\state}{\mathbf{x}}
\newcommand{\statedisc}{\tilde{\state}}
\newcommand{\control}{\mathbf{u}}
\newcommand{\policy}{\pi}
\newcommand{\noise}{\omega}
\newcommand{\dynamics}{f}
\newcommand{\noisedist}{\mu}

\newcommand*{\xTarget}{\X_T}
\newcommand*{\xUnsafe}{\X_U}

\definecolor{color1}{RGB}{55,126,184} %
\definecolor{color2}{RGB}{228,26,28} %
\definecolor{color3}{RGB}{77,175,74} %
\definecolor{color4}{RGB}{152,78,163} %
\definecolor{color5}{RGB}{255,127,0} %
\definecolor{color6}{rgb}{0.5, 1.0, 0.83} %
\definecolor{color7}{rgb}{1.0, 0.0, 1.0} %
\definecolor{color8}{rgb}{0.66, 0.66, 0.66} %

\usepackage{graphicx}

\usetikzlibrary{external}
\newif\iftikzexternal
\global\tikzexternaltrue

\renewcommand{\paragraph}[1]{\medskip\noindent\emph{#1}\,\,}
\renewcommand{\subsubsection}[1]{\medskip\noindent\textbf{#1}\,\,}

\usepackage{hyperref}

\makeatletter
\def\orcidID#1{\smash{\href{http://orcid.org/#1}{\protect\raisebox{-1.25pt}{\protect\includegraphics{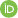}}}}}
\makeatother

\usepackage[firstpage]{draftwatermark} %

\SetWatermarkAngle{0}
\SetWatermarkText{\raisebox{12.5cm}{
\hspace{0.1cm}
\href{https://doi.org/10.5281/zenodo.15214887}{\includegraphics{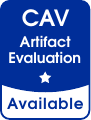}}
\hspace{9cm}
\includegraphics{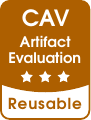}
}}

\newif\ifappendix
\global\appendixtrue

\begin{document}

\title{Policy Verification in Stochastic Dynamical Systems Using Logarithmic Neural Certificates
}
\titlerunning{Policy Verification in Stoch. Dyn. Sys. Using Logarithmic Neural Certificates}

\author{ 
Thom Badings\inst{1,2}\thanks{Equal contribution.}\orcidID{0000-0002-5235-1967}
\and
Wietze Koops\inst{2,3,4\star}\orcidID{0000-0001-9945-1992}
\and\\
Sebastian Junges\inst{2}\orcidID{0000-0003-0978-8466}
\and
Nils Jansen\inst{2,5}\orcidID{0000-0003-1318-8973}
}

 \institute{
University of Oxford, Oxford, United Kingdom
\\
\email{thom.badings@cs.ox.ac.uk} 
\and
Radboud University, Nijmegen, the Netherlands
\and
Lund University, Lund, Sweden
\and
University of Copenhagen, Copenhagen, Denmark
\and
Ruhr-University Bochum, Bochum, Germany
}

\maketitle              %

\begin{abstract}
We consider the verification of neural network policies for discrete-time stochastic systems with respect to reach-avoid specifications. We use a learner-verifier procedure that learns a certificate for the specification, represented as a neural network. Verifying that this neural network certificate is a so-called reach-avoid supermartingale (RASM) proves the satisfaction of a reach-avoid specification. Existing approaches for such a verification task rely on computed Lipschitz constants of neural networks. These approaches struggle with large Lipschitz constants, especially for reach-avoid specifications with high threshold probabilities. We present two key contributions to obtain smaller Lipschitz constants than existing approaches. First, we introduce logarithmic RASMs (logRASMs), which take exponentially smaller values than RASMs and hence have lower theoretical Lipschitz constants. Second, we present a fast method to compute tighter upper bounds on Lipschitz constants based on weighted norms. Our empirical evaluation shows we can consistently verify the satisfaction of reach-avoid specifications with probabilities as high as~$99.9999\%$.
\end{abstract}

\section{Introduction}
\label{sec:Introduction}
Feed-forward neural networks are widely used in reinforcement learning (RL) to represent policies for autonomous control systems operating in continuous and nonlinear environments~\cite{DBLP:journals/corr/LillicrapHPHETS15,DBLP:journals/arcras/Recht19,DBLP:conf/icml/DuanCHSA16}.
To deploy such policies in safety-critical domains, it is crucial to provide guarantees about their (closed-loop) behavior~\cite{DBLP:journals/corr/AmodeiOSCSM16}.
The development of techniques that provide such guarantees is an ongoing research effort~\cite{DBLP:journals/corr/abs-2405-06624}. 
In this paper, we study (nonlinear) stochastic dynamical systems, which are ubiquitous in control theory~\cite{khalil2002nonlinear,Bertsekas.Shreve78} and AI~\cite{DBLP:journals/arc/BusoniuBTKP18,bertsekas2019reinforcement} for modeling control tasks in uncertain environments.
The operational model of such a \emph{discrete-time stochastic system (DTSS)} is a Markov decision process (MDP) with a continuous state and action space, and with the transition function defined by stochastic difference equations.
We aim to prove that a \emph{reach-avoid specification} is satisfied, i.e., that the probability of reaching a set of goal states without visiting unsafe states is above some threshold~\cite{DBLP:journals/automatica/SummersL10}.
More precisely, we study the following verification problem (see \cref{fig:feedback_loop}): Given a DTSS and a neural network policy for this DTSS, check whether the policy satisfies a given reach-avoid specification.

\begin{figure}[!t]
\centering
\parbox{0.46\textwidth}{
    \centering
    \iftikzexternal
        \includegraphics[scale=.9]{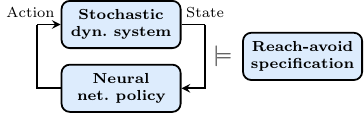}%
    \else
        \include{figures/feedback_loop}%
    \fi
}\hfill
\parbox{0.51\textwidth}{
    \centering
    \iftikzexternal
        \includegraphics[scale=.9]{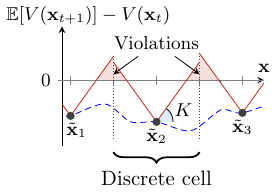}%
    \else
        \include{figures/Lipschitz}%
    \fi
}
\parbox{0.46\textwidth}{
    \caption{We verify that a neural network policy deployed on a DTSS satisfies a reach-avoid specification.}
    \label{fig:feedback_loop}
}\hfill
\parbox{0.51\textwidth}{
    \caption{The verifier computes the expected decrease of a candidate RASM $V$ on discrete states $\statedisc_i$ and uses a Lipschitz constant $K$ to generalize to all other states.}
    \label{fig:Lipschitz}
}
\end{figure}

\paragraph{Certificate functions.}
We follow the common paradigm of \emph{verification by finding a certificate}.
A certificate is a function satisfying conditions such that its \emph{existence} implies satisfaction of a specification.
Classical certificates include ranking functions~\cite{Floyd1993} (to prove termination in program loops~\cite{DBLP:conf/cav/BradleyMS05,DBLP:conf/vmcai/PodelskiR04}) and Lyapunov functions (to prove stability in non-stochastic systems~\cite{khalil2002nonlinear}).
In this paper, we build upon ranking supermartingales~\cite{DBLP:conf/cav/ChakarovS13,DBLP:conf/popl/FioritiH15}, specifically \emph{reach-avoid supermartingale} (RASM) certificates for DTSSes~\cite{DBLP:conf/aaai/ZikelicLHC23}.
A RASM is a function from the system's states to real values which (among other conditions) must decrease \emph{in expectation} at every step under the DTSS's dynamics. Hence, a RASM induces a \emph{supermartingale}~\cite{DBLP:books/daglib/0073491,DBLP:journals/tac/PrajnaJP07}.
The existence of a RASM proves the satisfaction of a reach-avoid specification.
Standard approaches to finding certificates mostly use optimization over restrictive templates, e.g., low-degree polynomials~\cite{DBLP:conf/rss/SteinhardtT11,DBLP:journals/automatica/SantoyoDC21}.
Thus, we follow the recent trend of representing certificates as neural networks instead~\cite{DBLP:conf/nips/ChangRG19,DBLP:journals/csysl/AbateAGP21,DBLP:conf/corl/RichardsB018,DBLP:conf/nips/ZhouQSL22,DBLP:conf/nips/00010KV23} (collectively called \emph{neural certificates}~\cite{DBLP:journals/trob/DawsonGF23}).

\paragraph{Learner-verifier framework.}
An effective method to learn a RASM is to use a counterexample-guided framework as in \cref{fig:CEGIS}.
Such a framework iterates between (1) a \emph{learner} that trains a neural network as a candidate certificate and (2) a \emph{verifier} that either proves the validity of the candidate or returns counterexamples that disprove that the candidate is a RASM~\cite{DBLP:conf/cav/AbateDKKP18,DBLP:conf/tacas/ChatterjeeHLZ23}.
We initialize the learner with a policy (without guarantees) learned from any common RL algorithm.
Then, the learner trains the neural network certificate and updates the (initial) policy based on counterexamples that it receives from the verifier. 

\paragraph{Challenges in verifying RASMs.}
Recall that the verifier in \cref{fig:CEGIS} must prove that the candidate RASM decreases in expectation with the DTSS's dynamics.
This \emph{expected decrease condition} is shown by the blue line in \cref{fig:Lipschitz}, so we must show that this line is strictly negative in every state $\state$.
Because the state space is continuous, existing RASM verifiers check conditions on a discretization of the state space (points $\statedisc_1, \statedisc_2, \statedisc_3$ in \cref{fig:Lipschitz}) and use Lipschitz continuity of the policy and the RASM (with Lipschitz constant $K$ in \cref{fig:Lipschitz}) to generalize to the entire state space.\footnote{For other neural certificates, different verifiers have been used, e.g., using satisfiability modulo theories (SMT) solving, which we discuss in the related work (\cref{sec:Related}).}
However, this approach has two main limitations.
First, specifications with very high threshold probabilities necessarily require RASMs with infeasibly large Lipschitz constants.
Second, computing the (smallest) Lipschitz constant of a neural network exactly is intractable \cite{DBLP:conf/nips/VirmauxS18}, so existing RASM verifiers use loose upper bounds instead.
Consequently, applying the framework to safety-critical domains, where high levels of assurance are needed, remains elusive.

\begin{figure}[t!]
	\centering
        \iftikzexternal
            \includegraphics[scale=0.8]{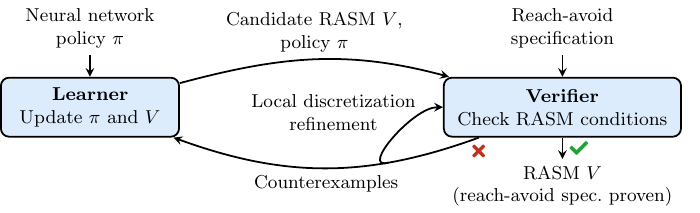}%
        \else
            \include{figures/CEGIS}%
        \fi
        \vspace{-3mm}
        \caption{Overview of the learner-verifier framework for finding a RASM.}%
	\label{fig:CEGIS}%
\end{figure}%

\paragraph{Our approach to improving RASM verifiers.}
In this paper, we propose novel techniques that address these two limitations in verifying RASMs represented as neural networks.
Our method can verify reach-avoid specifications with threshold probabilities as high as $99.9999\%$, which is enabled by two key novel aspects:
\begin{enumerate}
    \item \textbf{Logarithmic RASMs.}
    Instead of training the neural network to satisfy the RASM conditions from~\cite{DBLP:conf/aaai/ZikelicLHC23}, we consider \emph{the logarithm of these conditions}.
    Like a (standard) RASM, the resulting certificate, which we call a \emph{logarithmic RASM} (or \emph{logRASM}), proves the satisfaction of a reach-avoid specification.
    A logRASM takes exponentially smaller values than a RASM, leading to a lower theoretical Lipschitz constant of the certificate.
    \item \textbf{Tighter bounds on Lipschitz constants.}
    We use \emph{weighted norm systems} instead of standard norms to compute Lipschitz constants for neural networks.
    In combination with \emph{averaged activation operators}~\cite{combettes2020lipschitz}, our method leads to significantly tighter global bounds on Lipschitz constants for neural networks.
    These better Lipschitz constants improve the performance of RASM verifiers without sacrificing the efficiency of the verifier.
\end{enumerate}
We embed our techniques in the learner-verifier framework depicted in \cref{fig:CEGIS}.
We further accelerate the framework using a local refinement scheme for the verifier that only refines the discretization at points where necessary, similar to ideas from~\cite{DBLP:conf/atva/AnsaripourCHLZ23,DBLP:journals/csysl/MathiesenCL23}.
Specifically, when a discretized point violates the RASM conditions, we can determine if the violation could be mitigated by further refining that point.
Following this intuition, we locally refine the discretization while avoiding unnecessary computations in cases where refinement cannot fix violations.

\paragraph{Contributions.}
In summary, we present novel techniques for verifying neural network policies in stochastic dynamical systems with reach-avoid specifications.
Our approach combines the use of logarithmic RASMs as certificates with tighter upper bounds on Lipschitz constants of neural networks.
Our experiments confirm that we can verify specifications with probability bounds orders of magnitude higher than the state-of-the-art.

\section{Problem Statement}
\label{sec:Problem}

We study discrete-time stochastic (nonlinear dynamical) systems, which can be seen as a concise representation of MDPs with continuous state and action spaces:

\begin{definition}[DTSS]
    \label{def:DTSS}
    A \emph{discrete-time stochastic system (DTSS)} is a tuple $\dtss \coloneqq \DTSS$, where $\X \subseteq \R^d$ is the (continuous) \emph{state space}, $\X_0 \subseteq \X$ is a set of \emph{initial states}, $\U \subseteq \R^m$ is the (continuous) \emph{action space}, $\Noise \subseteq \R^p$ is the \emph{noise space}, $\noisedist \colon \mathcal{B}_{\mathcal{N}} \to [0,1]$ is a \emph{probability measure} on the Borel $\sigma$-algebra $\mathcal{B}_\Noise$ on $\Noise$, and $f \colon \X \times \U \times \mathcal{N} \to \X$ is the \emph{transition function}.
\end{definition}
The stochasticity of a DTSS $\dtss$ is modeled by the probability space $(\mathcal{N}, \mathcal{B}_{\mathcal{N}}, \noisedist)$ (see, e.g.,~\cite{Durrett96} for details).
The state $\state_t$ of a DTSS is defined recursively over discrete steps $t \in \Nzero$ as
\begin{equation}
    \state_{t+1} = \dynamics(\state_t, \control_t, \noise_t), 
    \quad \state_0 \in \X_0,
\label{eq:dynamical_system}
\end{equation}
where $\noise_t \sim \noisedist$. An execution of the DTSS $\dtss$ is an infinite sequence $(\state_t, \control_t, \noise_t)_{t \in \Nzero}$ of state-action-disturbance triples that satisfy \cref{eq:dynamical_system} for all $t \in \Nzero$.

\paragraph{Policy.}
A (memoryless deterministic) policy $\policy \colon \X \to \U$ chooses actions in a DTSS such that $\control_t \coloneqq \policy(\state_t)$ in \cref{eq:dynamical_system}.
Fixing a policy $\policy$ and an initial state $\state_0 \in \X_0$ defines an \emph{induced Markov process} in the probability space of all executions~\cite{Bertsekas.Shreve78,DBLP:books/wi/Puterman94}.
We denote the probability measure on this probability space by $\Prob^\policy_{\state_0}$. 

\paragraph{Reach-avoid specification.}
For an induced Markov process, we want to evaluate the probability of reaching a \emph{target set} $\xTarget \subseteq \X$ before reaching an \emph{unsafe set} $\xUnsafe \subseteq \X$.
Formally, this \emph{reach-avoid probability} $\satprob^\policy_{\state_0}(\xTarget, \xUnsafe)$ is defined as
\begin{equation}
    \label{eq:reachavoid_prob}
    \satprob^\policy_{\state_0}(\xTarget, \xUnsafe) \coloneqq
    \Prob^\policy_{\state_0} \big\{ 
    \exists t \in \Nzero \, : \,
    \state_t \in \xTarget 
    \wedge
    ( \forall t' \in \{0,\ldots,t\} : \state_{t'} \notin \xUnsafe )
    \big\}.
\end{equation}
Intuitively, $\satprob^\policy_{\state_0}(\xTarget, \xUnsafe)$ is the probability that, from initial state $\state_0$, the system eventually reaches $\xTarget$ while never reaching $\xUnsafe$ before. 
\begin{definition}[Reach-avoid specification]
Given a DTSS as in \cref{def:DTSS} and a policy $\policy$, a \emph{reach-avoid specification} is a triple $\tuple{\xTarget, \xUnsafe, \rho}$ and is satisfied if $\satprob^\policy_{\state_0}(\xTarget, \xUnsafe) \geq \rho$ for all $\state_0 \in \X_0$.
\end{definition}

\paragraph{Formal problem.}
The following verification problem is central to this paper:
\begin{problemBold}[Policy verification]\label{problem}
    Given a DTSS $\dtss$ with a policy $\policy$, verify whether the reach-avoid specification $\tuple{\xTarget, \xUnsafe, \rho}$ is satisfied. %
\end{problemBold}

In this paper, we make the following standard assumptions~\cite{Abate2008probabilisticSystems,khalil2002nonlinear,DBLP:conf/nips/BerkenkampTS017}. These assumptions ensure that the reach-avoid probability is well-defined~\cite{Bertsekas.Shreve78}.

\begin{assumption}
    \label{assumptions} For a DTSS $\dtss = \DTSS$, we assume that:
    \begin{enumerate}[topsep=3pt]
    \item The transition function $f$ and the policy $\policy$ are Lipschitz continuous;
    \item The sets $\X$, $\X_0$, $\xTarget$, $\xUnsafe$ and $\U$ are Borel measurable;
    \item The sets $\X$ and $\mathcal{N}$ are compact (i.e., closed and bounded).
    \end{enumerate}
\end{assumption}

\section{Verifying Reach-Avoid Specifications Using RASMs}
\label{sec:Discretization}

In this section, we fix a DTSS~$\DTSS$ as in \cref{def:DTSS}, a policy $\policy$, and a reach-avoid specification $\tuple{\xTarget, \xUnsafe, \rho}$.
We recap the certificate, called a \emph{reach-avoid supermartingale} (RASM), and the verification procedure, proposed by~\cite{DBLP:conf/aaai/ZikelicLHC23} to solve \cref{problem}. 
This section deviates from~\cite{DBLP:conf/aaai/ZikelicLHC23} in one aspect (see \cref{remark:deviation}).

\begin{definition}[RASM]\label{def:rasm}
    A continuous function $V\colon\X \to \R_{\geq 0}$ is a \emph{reach-avoid supermartingale (RASM)} (for a fixed reach-avoid specification) if: %
    \begin{enumerate}[topsep=3pt]
        \item \emph{Initial condition:} \; $V(\state) \leq 1$ for all $\state \in \X_0$;
        \item \emph{Safety condition:} \; $V(\state) \geq \frac1{1-\rho}$ for all $\state \in \xUnsafe$;
        \item \emph{Expected decrease condition:} There exists $\epsilon > 0$ such that for all $\state \in \X \setminus \xTarget$ with  $V(\state) < \frac1{1-\rho}$, we have $\Exp_{\noise \sim \noisedist}\left[V(f(\state, \policy(\state), \noise))\right] \leq V(\state) - \epsilon$.
    \end{enumerate}
\end{definition}
A RASM $V$ associates each state $\state \in \X$ with a non-negative value and decreases in expectation at every step in the dynamics.
To reach an unsafe state from any initial state, the value $V(\state_t)$ needs to increase from at most 1 to at least $\frac{1}{1-\rho}$ along the execution. Since $V$ decreases in expectation with every step, this happens with probability at most $1-\rho$.
This intuitively shows why the existence of a RASM implies that the reach-avoid specification in \cref{problem} is satisfied:
\begin{theorem}[\!\!\!\cite{DBLP:conf/aaai/ZikelicLHC23}, proof in {\ifappendix \cref{proof:rasm}\else \cite[App.~B.2]{Badings_CAV25_extended}\fi}]\label{thm:rasm}
    If there exists a RASM for the reach-avoid specification, then this specification is satisfied. 
\end{theorem}

\subsection{Verifying RASMs by Discretization}
Since the state space $\X$ is continuous, it is not feasible to check the conditions from \cref{def:rasm} on individual points $\state \in \X$. 
Instead, we check slightly stronger versions of the conditions from \cref{def:rasm} on a \emph{discretization} of the state space into rectangular cells.
Concretely, for a given \emph{mesh size} $\tau > 0$, define $\cell^\tau_\infty(\state) = \{\state' : \| \state - \state' \|_{\infty} \leq \tau/d\}$, where $\| \cdot \|_\infty$ denotes the $\infty$-norm and $d$ is the dimension of the state space $\X \subseteq \mathbb{R}^d$.\footnote{We divide by $d$ in the definition of $\cell^\tau_\infty(\state)$ to ensure that $\|\state - \state'\|_1 \leq \tau$ for all $\state' \in \cell^\tau_\infty(\state)$, which we need later in \cref{def:drasm}.} We allow different mesh sizes for cells around different centers $\state$.
A discretization of $\X$ must cover $\X$ as follows:

\begin{definition}[Discretization of $\X$]
    \label{def:discretization}
    A \emph{discretization} of $\X$ is a finite set of points $\Xdisc$ together with a \emph{mesh size} $\tau_{\statedisc}$ for each $\statedisc \in \Xdisc$ such that for every $\state \in \X$ there exists an $\statedisc \in \Xdisc$ such that $\state \in \cell_\infty^{\tau_{\statedisc}}(\statedisc)$.
\end{definition}
\paragraph{Lipschitz constants.}
To generalize results on a discretization to the full state space, we use Lipschitz continuity (using 1-norms). We say that $L_g$ is a \emph{Lipschitz constant} of a function $g$ if $\| g(x) - g(x') \|_1 \leq L_g \| x - x'\|_1$ for all $x, x'$ in the domain of $g$. All Lipschitz constants in this paper will be with respect to the 1-norm. Throughout the paper, we write $L_f$ and $L_\policy$ for the Lipschitz constants of the dynamics $f$ and the policy $\policy$, respectively. 

\paragraph{Conditions on the discretization.}
We define a stronger version of the RASM conditions, such that the satisfaction of these stronger conditions on each point $\statedisc \in \Xdisc$ from a discretization of $\X$ implies the satisfaction of the conditions in \cref{def:rasm}.
Toward these conditions, we define for any $\statedisc \in \Xdisc$
\[
\Vmin(\statedisc) \,=\!\! \min\limits_{\state \in \cell^{\tau_{\statedisc}}_\infty(\statedisc)}\! V(\state) \qquad \text{and} \qquad \Vmax(\statedisc) \,=\!\! \max\limits_{\state \in \cell^{\tau_{\statedisc}}_\infty(\statedisc)}\! V(\state)%
\]%
as the min/max of $V$ within each  $\cell^{\tau_{\statedisc}}_\infty(\statedisc)$.
Computing $\Vmin(\statedisc)$ and $\Vmax(\statedisc)$ analytically is in general not possible, but using interval bound propagation (IBP)~\cite{gowal2018effectiveness} we can compute bounds $\Vlb(\statedisc)$ and $\Vub(\statedisc)$ satisfying \[\Vlb(\statedisc) \leq \Vmin(\statedisc) \leq V(\statedisc) \leq \Vmax(\statedisc) \leq \Vub(\statedisc)\] for all $\statedisc \in \Xdisc$.
These bounds obtained from IBP are generally tighter than those computed using Lipschitz constants.

\paragraph{Discrete RASM.}
Concretely, the satisfaction of the conditions in \cref{def:rasm} is then implied by the following conditions on the discretization.
\begin{definition}[Discrete RASM]\label{def:drasm}
    Let $V\colon\X \to \R_{\geq 0}$ be Lipschitz continuous with Lipschitz constant $L_V$.  Let\footnote{In our implementation, we use a slightly improved definition of $K$; see {\ifappendix \cref{app:splitlip}\else \cite[App.~A.1]{Badings_CAV25_extended}\fi}.} $K = L_V L_f \left(L_\pi + 1\right)$ and let $\Xdisc$ be a~discretization with a mesh size $\tau_{\statedisc}$ for each $\statedisc \in \Xdisc$. Then, $V$ is a \emph{discrete RASM} for $\Xdisc$
    if: %
    \begin{enumerate}[topsep=3pt]
    \item \emph{Initial condition:} \;  $\Vub(\statedisc) \leq 1$ for all $\statedisc \in \Xdisc$ with $\cell_\infty^{\tau_{\statedisc}}(\statedisc) \cap \X_0 \neq \emptyset$.
    \item \emph{Safety condition:} \;  $\Vlb(\statedisc) \geq \frac1{1-\rho}$ for all $\statedisc \in \Xdisc$ with $\cell_\infty^{\tau_{\statedisc}}(\statedisc) \cap \xUnsafe \neq \emptyset$.
    \item \emph{Expected decrease condition:}
    \begin{equation}\label{eq:expdecrcond}
        \Exp_{\noise \sim \noisedist}\left[V(f(\statedisc, \policy(\statedisc), \noise))\right] < \Vlb(\statedisc) - \tau_{\statedisc} K
    \end{equation}for all $\statedisc \in \Xdisc$ with $\cell_\infty^{\tau_{\statedisc}}(\statedisc) \cap (\X \setminus \xTarget) \neq \emptyset$ and  $\Vlb(\statedisc) < \frac1{1-\rho}$.
\end{enumerate}
\end{definition}

\begin{remark}
\label{remark:deviation}
We deviate slightly from~\cite{DBLP:conf/aaai/ZikelicLHC23}, which instead uses $V(\statedisc) - \tau (K + L_V)$ on the right-hand side of \cref{eq:expdecrcond}.
Since IBP usually gives tighter bounds than Lipschitz continuity, our version of \cref{eq:expdecrcond} is (slightly) easier to satisfy. 
\end{remark}

\noindent Verifying the conditions in \cref{def:drasm} is sufficient to show that $V$ is a RASM:

\begin{lemma}[proof in {\ifappendix \cref{proof:drasm}\else \cite[App.~B.3]{Badings_CAV25_extended}\fi}]\label{lem:drasm} 
Every discrete RASM is also a RASM. 
\end{lemma}

Computing the expected value in \cref{eq:expdecrcond} exactly is generally infeasible.
Instead, we bound  this expectation from above by discretizing the noise space $\mathcal{N}$ into a collection of cells $\mathcal{C}$ (which is possible since $\mathcal{N}$ is compact), such that
$
\Exp_{\noise \sim \noisedist}\left[V(f(\statedisc, \policy(\statedisc), \noise))\right] \leq \sum_{C \in \mathcal{C}} \Prob(\noise \in C) \sup_{\noise \in C} \left[V(f(\statedisc, \policy(\statedisc), \noise)) \right],
$
where we again use IBP to upper bound $\sup_{\noise \in C} \left[V(f(\statedisc, \policy(\statedisc), \noise))\right]$ for each cell $C \in \mathcal{C}$. 

\paragraph{Shape of a discrete RASM.} 
Finally, we make some remarks about the typical shape of a discrete RASM $V$, especially in the context where we try to minimize its Lipschitz constant $K$ (to speed up verification). Since $V$ is Lipschitz continuous, it is differentiable almost everywhere. Recall that the negative gradient vector $-\nabla V$ points in the direction that $V$ decreases fastest. Due to the expected decrease condition, $-\nabla V$ will therefore typically roughly point in the direction the state moves under the dynamics. Moreover, since the expected decrease condition requires a fixed decrease of $\tau_{\statedisc} K$, the slope $\|\nabla V\|$ will be roughly constant, at least in regions where the step size under the dynamics is roughly constant.

\subsection{Challenges in Verifying RASMs}
The existing verification procedure based on the discrete RASM conditions in \cref{def:drasm} is generally conservative and computationally expensive.
The scalability of the procedure is especially limited by the large Lipschitz constant $L_V$ of any RASM.
Concretely, the initial condition in \cref{def:rasm} requires a value of at least $\frac{1}{1-\rho}$ in all $\state \in \X_U$, while the safety condition requires a value of at most $1$ in all $\state \in \X_0$.
Thus, these conditions on the RASM imply that
\[
L_V \geq \frac1{\text{dist}(\X_0, \X_U)} \left(\frac1{1-\rho}-1\right),
\]
where $\text{dist}(\X_0, \X_U) = \inf_{(\state_0, \state_u)\in \X_0 \times \xUnsafe} \|\state_0-\state_u\|_1$ is the smallest distance between an initial and an unsafe state. 
Since a large $L_V$ leads to a large $K$, verifying the conditions in \cref{def:drasm} requires a fine discretization (i.e., a discretization with small mesh size $\tau$), which is hence computationally expensive. 
Moreover, a higher $L_V$ is required for specifications with a higher threshold probability $\rho$.

The limitations of large Lipschitz constants are exacerbated since the RASM $V$ and policy $\policy$ are neural networks.
In particular, computing the (smallest) Lipschitz constant of a neural network exactly is intractable~\cite{DBLP:conf/nips/VirmauxS18}.
Hence, approaches such as~\cite{DBLP:conf/aaai/ZikelicLHC23} use loose upper bounds on these Lipschitz constants instead.

Concretely, our key contributions address these challenges by (1) taking the logarithm of the RASM conditions (\cref{sec:Log}), which reduces the lower bound on $L_V$ to $\frac1{\text{dist}(\X_0, \X_U)} \log\big(\frac1{1-\rho}\big)$, and (2) computing tighter bounds on the Lipschitz constant of neural networks (\cref{sec:Lipschitz}).

\section{Logarithmic RASMs} \label{sec:Log}

We now turn to our first main contribution, which is proposing the notion of \emph{logarithmic RASMs}, and providing a less conservative method for checking that a function is a logarithmic RASM using a discretization. 
Our starting point is to take the (natural) logarithm of the RASM conditions in \cref{def:rasm}: 

\begin{definition}[logRASM]\label{def:lograsm}
    A continuous function $V\colon\X \to \R$ is a \emph{logarithmic RASM (logRASM)} if: %
    \begin{enumerate}[topsep=3pt]
        \item \emph{Initial condition:} \;  $V(\state) \leq 0$ for all $\state \in \X_0$;
        \item \emph{Safety condition:} \;  $V(\state) \geq \log\big(\frac1{1-\rho}\big)$ for all $\state \in \xUnsafe$;
        \item \emph{Expected decrease condition:} There exists $\epsilon > 0$ such that for all $\state \in \X \setminus \xTarget$ with  $V(\state) < \log\big(\frac1{1-\rho}\big)$, we have $\log \Exp_{\noise \sim \noisedist}\big[\!\exp\!\big(V(f(\state, \policy(\state), \noise))\big)\!\big] \leq V(\state) - \epsilon$.
    \end{enumerate}
\end{definition}
\noindent The exponential of a logRASM is indeed a RASM:
\begin{lemma}[proof in {\ifappendix \cref{proof:lograsm}\else \cite[App.~B.4]{Badings_CAV25_extended}\fi}] \label{lem:lograsm}
If $V$ is a logRASM, then $ \exp\!\big(V\big)$ is a~RASM.
\end{lemma}
The threshold of the safety condition \cref{def:lograsm} is only $\log\big(\frac1{1-\rho}\big)$, which is much smaller (and thus easier to satisfy) than $\frac1{1-\rho}$ in \cref{def:rasm}.
On the other hand, for the expected decrease condition we now have to bound $\log \Exp_{\noise \sim \noisedist}\left[\exp(V(\state_{t+1}))\right]$ in \cref{def:lograsm}, which by Jensen's inequality is always larger than the $\Exp_{\noise \sim \noisedist}\left[V(\state_{t+1})\right]$ in \cref{def:rasm}, and thus the condition of \cref{def:lograsm} is easier to satisfy. Nevertheless, in practice, the positive effect of the (exponentially) smaller threshold is larger.

\paragraph{Conditions on the discretization.} Next, we show how we can check that a function $V\colon\X \to \R$ is a logRASM by checking stronger conditions on a discretization.

\begin{definition}[Discrete logRASM]\label{def:dlograsm}
    Let $V\colon\X \to \R$ be Lipschitz continuous with Lipschitz constant $L_V$.  Let $K = L_V L_f \left(L_\pi + 1\right)$ and let $\Xdisc$ be a~discretization with mesh sizes $\tau_{\statedisc}$. Then, $V$ is a \emph{discrete logRASM} for $\Xdisc$
    if the~following~hold: 
    \begin{enumerate}[topsep=3pt]
    \item \emph{Initial condition:}\, $\Vub(\statedisc) \leq 0$ for all $\statedisc \in \Xdisc$ with $\cell_{\infty}^{\tau_{\statedisc}}(\statedisc) \cap \X_0 \neq \emptyset$.
    \item \emph{Safety condition:}\, $\Vlb(\statedisc) \geq \log\big(\frac1{1-\rho}\big)$ for all $\statedisc \in \Xdisc$ with $\cell_\infty^{\tau_{\statedisc}}(\statedisc) \cap \xUnsafe \neq \emptyset$.
    \item \emph{Expected decrease condition:}
    \begin{equation}\label{eq:expdecrcond_new}
        \log \Exp_{\noise \sim \noisedist}\Big[\exp\!\big(V(f(\statedisc, \policy(\statedisc), \noise))\big)\Big] < \Vlb(\statedisc) - \tau_{\statedisc} K
    \end{equation}for all $\statedisc \in \Xdisc$ such that $\cell_\infty^{\tau_{\statedisc}}(\statedisc) \cap (\X \setminus \xTarget) \neq \emptyset$ and  $\Vlb(\statedisc) < \log\big(\frac1{1-\rho}\big)$.
\end{enumerate}
\end{definition}

The following theorem is the main result of this section and shows that the existence of a discrete logRASM implies the existence of a RASM.

\begin{theorem}[proof in {\ifappendix \cref{proof:logdrasm}\else \cite[App.~B.5]{Badings_CAV25_extended}\fi}] \label{thm:logdrasm}
If $V$ is a discrete logRASM for a discretization $\Xdisc$, then $ \exp\!\big(V\big)$ is a RASM.
\end{theorem}

We now sketch the proof of \cref{thm:logdrasm}. The main difference compared to \cref{lem:drasm} lies in the expected decrease condition. To show that \cref{eq:expdecrcond_new} implies the expected decrease condition in \cref{def:rasm} for $\exp(V)$, we first note that $V(f(\state, \policy(\state), \noise)) \leq V(f(\statedisc, \policy(\statedisc), \noise)) + \tau_{\statedisc} K$ by Lipschitz continuity.
  Hence, 
   \begin{align*} \Exp_{\noise \sim \noisedist} \big[\exp\!\big(V(f(\state, \policy(\state), \noise))\big)\big] 
     &\leq \Exp_{\noise \sim \noisedist} \big[\exp\!\big(V(f(\statedisc, \policy(\statedisc), \noise))+\tau_{\statedisc} K\big)\big] 
  \\ &= e^{\tau_{\statedisc} K}\Exp_{\noise \sim \noisedist} \big[\exp\!\big(V(f(\statedisc, \policy(\statedisc), \noise))\big)\big] 
  \\ &< e^{\tau_{\statedisc} K} e^{ - \tau_{\statedisc} K} \exp (\Vlb(\statedisc)) \leq \exp(V(\state)).
  \end{align*}
Finally, we obtain the $-\epsilon$ in the expected decrease condition from \cref{def:rasm} using a compactness argument (see {\ifappendix \cref{proof:logdrasm} for details\else \cite[App.~B.5]{Badings_CAV25_extended}\fi}), from which \cref{thm:logdrasm} follows.

The main contribution of \cref{def:dlograsm} and \cref{thm:logdrasm} lies in \cref{eq:expdecrcond_new}. Notably, proving that \cref{eq:expdecrcond_new} is sufficient for showing that the expected decrease condition holds effectively exploits the \emph{local} Lipschitz constant of $\exp$, rather than the global Lipschitz constant of $\exp(V)$. Indeed, if we would directly adapt the proof of \cref{lem:drasm}, we would obtain the condition
\begin{equation} \label{eq:expdecrcond_weak}
        \Exp_{\noise \sim \noisedist}\big[\exp\big(V(f(\statedisc, \policy(\statedisc), \noise))\big)\big] < \exp(\Vlb(\statedisc)) - \tau_{\statedisc} K'
\end{equation}
where $K' = \frac1{1-\rho} K$ is a Lipschitz constant of $\exp(V)$, where we use that we can cap any RASM at $\frac1{1-\rho}$.  The following lemma shows that our novel condition \cref{eq:expdecrcond_new} is always a weaker (i.e., better) condition than \cref{eq:expdecrcond_weak}.

\begin{lemma}[proof in {\ifappendix \cref{proof:better}\else \cite[App.~B.6]{Badings_CAV25_extended}\fi}] \label{lem:better}
Let $K' = \frac1{1-\rho} K > 0$. If $\Vlb(\statedisc) < \log\big(\frac1{1-\rho}\big)$, then $\exp(\Vlb(\statedisc)) - \tau_{\statedisc} K' < \exp(\Vlb(\statedisc) - \tau_{\statedisc} K)$.
\end{lemma}

\begin{example}
As a concrete example, consider $\rho = 0.9999$, $K' = 20$, and $\Vlb(\statedisc) = 5$. Then $\exp(\Vlb(\statedisc)) - \tau_{\statedisc} K' \approx -51.6 < 145.5 \approx \exp(\Vlb(\statedisc) - \tau_{\statedisc} K)$, showing that \cref{eq:expdecrcond_new} is much easier to satisfy than \cref{eq:expdecrcond_weak}. To obtain a right hand side of 145.5 in \cref{eq:expdecrcond_weak} we would require $\tau_{\statedisc} \approx 1.5 \cdot 10^{-5}$.  Hence, our new approach allows a discretization that is more than 60 times coarser (in each dimension).
\end{example}

\paragraph{Shape of a discrete logRASM.}
Although both a discrete RASM and the exponential of a discrete logRASM yield a RASM, they typically look quite different. Due to the fixed decrease of $\tau_{\statedisc} K$ required by the expected decrease condition~\eqref{eq:expdecrcond_new}, also a discrete logRASM generally has a nearly constant slope $\| \nabla V \|$, similarly to a discrete RASM. However, a RASM $V'$ has an exponentially larger Lipschitz constant $K$ and thus an (exponentially) larger slope than a logRASM $V$. After taking the exponential, the RASM $\exp(V)$ has a similar slope as $V'$ at points where $V'$ and  $\exp(V)$ are large, but $\exp(V)$ has a smaller slope at points where $V'$ and  $\exp(V)$ are small. %

\section{Tighter Lipschitz Constants for Neural Networks}
\label{sec:Lipschitz}

Our second main contribution is a novel method to compute tighter Lipschitz constants for feed-forward neural networks. 
In particular, to obtain tighter \emph{global} Lipschitz constants, we combine the use of \emph{weighted 1-norms} defined by $\|x\| =  \sum_i \weight_i \lvert x_i\rvert$ (for weights $\weight_i > 0$) with \emph{averaged activation operators}~\cite{combettes2020lipschitz}. 

We first provide some intuition on why using weighted norms leads to tighter Lipschitz constants. For the standard, unweighted norm, the Lipschitz constant provides the same upper bound on the change of a function in each direction. In contrast, weighted norms allow for different bounds in different directions. When composing functions (e.g., different neural network layers), the method with standard norm therefore assumes the same upper bound on the change in all directions, while our method accounts for the upper bound on each individual direction. Since these bounds are generally tighter for all but the maximal direction, our method computes tighter Lipschitz constants.

We consider feed-forward neural networks with linear layers:

\begin{definition}[Neural network] \label{def:neural}
\!\!An $(n{+}1)$-layer (feed-forward) \emph{neural network} with dimensions $m_k$ $(0 \leq k \leq n)$ is a sequence of tuples $\nn \coloneqq  (\tuple{A_k, b_k, R_k})_{k=1}^n$, where $A_k \in \R^{m_k \times m_{k-1}}$ are \emph{matrices},\footnote{We use $A$ rather than the standard $W$ for the matrices of the neural network to avoid confusion with the weights from weighted norms.} $b_k \in \R^{m_k}$ are \emph{biases}, and $R_k \colon \R^{m_k} \to \R^{m_k}$ are \emph{activation functions}. The operator $\network \colon \R^{m_0} \to \R^{m_n}$ corresponding~to~$\nn$ maps $x_0$ to $x_n$,  where $x_k$ is defined recursively by $x_{k} = R_{k}\left( A_k x_{k-1} + b_k\right)$ for $1 \leq k \leq n$.
\end{definition}
\begin{assumption}
    \label{assumption:activation_lip}
    The Lipschitz constant of each activation function $R_k$ is~1.
\end{assumption}
\cref{assumption:activation_lip} is satisfied by common activation functions such as ReLU, Softplus, tanh, and sigmoid. It is straightforward to generalize our results to any Lipschitz continuous activation functions, but we do not pursue this here. In the following, we use the notation introduced in \cref{def:neural} for the components of the neural network. 

\subsection{Weighted Norms} \label{subsec:weighted}

A weighted $1$-norm of dimension $m$ is a function from $\R^m$ to $\R$ defined by $\| x \| = \sum_{i=1}^m w_i \| x_i \|$, where $w \in \R_{>0}^m$.
By combining a weighted $1$-norm for each layer of a neural network, we obtain a \emph{weight system}.
\begin{definition}[Weight system]
A \emph{weight system} $\weightsys$ for an $(n\!+\!1)$-layer neural network consists of a weighted 1-norm $\| x \|_\weightsys^k = \sum_{i=1}^{m_k} \weight_i^k |x_i|$ for each layer $0 \leq k \leq n$, where $\weight^k \in \R_{>0}^{m_k}$ and $\max_i \weight_i^k = 1$.\footnote{We assume w.l.o.g.\ that the max.\ weight is 1: we may rescale all weights (and thus the Lipschitz bound).} 
\end{definition}
Given weighted norms $\| \cdot \|_\weightsys^k$ and $\| \cdot \|_\weightsys^\ell$ on $\R^{m_k}$ and $\R^{m_\ell}$, we define the weighted norm for a matrix $M \in \R^{m_\ell \times m_k}$ as $\| M \|^{k,\ell}_{\weightsys} = \sup \Big\{ \frac{\| Mx \|^\ell_{\weightsys}}{\| x \|^k_{\weightsys}} \!~\Big|~\! x \in \R^{m_k}, x \neq 0\Big\}$. The next lemma shows how we compute $\| M \|^{k,\ell}_{\weightsys}$ in practice.
\begin{lemma}[proof in {\ifappendix \cref{proof:matrixnorm}\else \cite[App.~B.7]{Badings_CAV25_extended}\fi}]\label{lem:matrixnorm}
Let $M \in \R^{m_\ell \times m_k}$ be a matrix with entries $M_{ij}$. Equip the space $\R^{m_k}$ with the norm $\|x\|_{\weightsys}^k = \sum_{i=1}^{m_k} \weight_i^k \lvert x_i\rvert $, and the space $\R^{m_\ell}$ with the norm $\|x\|_{\weightsys}^{\ell} = \sum_{i=1}^{m_\ell} \weight_i^\ell \lvert x_i\rvert $. Then the corresponding matrix norm satisfies 
$
\|M\|^{k, \ell}_{\weightsys} \;=\; \max\limits_{1 \leq j \leq m_k}  \left[\frac{1}{\weight_j^k} \sum_{i=1}^{m_\ell} \weight_i^\ell  \left\lvert M_{ij} \right\rvert \right]$.
\end{lemma}
We now define the Lipschitz bound of a neural network for a weight system. 
\begin{definition}[Lipschitz bound]
The \emph{Lipschitz bound} of a neural network~$\nn$ for a weight system~$\weightsys$ is  $\lipnetwork = \prod_{\ell=1}^n\! \| A_\ell \|^{\ell-1, \ell}_{\weightsys}$.
\end{definition}
The Lipschitz bound $\lipnetwork$ is indeed a Lipschitz constant of the operator $\network$:
\begin{lemma}[proof in {\ifappendix \cref{proof:lipprod}\else \cite[App.~B.8]{Badings_CAV25_extended}\fi}]\label{lem:lipprod} 
Let $\weightsys$ be a weight system. Then $\lipnetwork$ is a Lipschitz constant of $\network$, i.e.\ $\|\network(x) - \network(x')\|^n_{\weightsys} \leq  \lipnetwork \|x - x'\|^0_{\weightsys}$ for all $x, x' \in \R^{m_0}$. If additionally $w^n_i=1$ for all $1 \leq i \leq m_n$,  then $\lipnetwork$ is a Lipschitz constant of $\network$ for the standard (unweighted) 1-norm, i.e.\ $\|\network(x) - \network(x')\| \leq \lipnetwork \|x - x'\|$ for all $x, x' \in \R^{m_0}$.
\end{lemma}
By choosing the same unweighted norm for each layer, one may recover the Lipschitz bound from \cite{DBLP:journals/corr/SzegedyZSBEGF13}, which corresponds to the approach presented in \cite{DBLP:conf/aaai/ZikelicLHC23}. We now show that choosing weights aptly can lead to smaller Lipschitz bounds:
\begin{figure}[t!]
	\centering
        \scalebox{0.74}{
\begin{tikzpicture}
\begin{axis}[
axis lines = center,
xmin = -2,
ymin = -2,
xmax = 2,
ymax = 2,
xtick={-2,-1,0,1,2},
xticklabels={$-1$,$-0.5$,,$0.5$,$1$},
ytick={-2,-1,0,1,2},
yticklabels={$-1$,$-0.5$,,$0.5$,$1$},
legend style={at={(1.5,0.62)},anchor=south east,legend columns=1,column sep=0.5em},
width = 6.4 cm,
height = 4 cm
]
\addplot [
domain=0:2, 
samples=100, 
color=black,
]
{1-x/2};
\addlegendentry{ours, $|x^2| \leq 1$}
\addplot [
domain=0:1, 
samples=100, 
color=red,
]
{1-1*x};
\addlegendentry{\cite{DBLP:journals/corr/SzegedyZSBEGF13}, $|x^2| \leq 1$}
\addplot [
domain=0:2, 
samples=100, 
color=black,
]
{-1+x/2};
\addplot [
domain=-2:0, 
samples=100, 
color=black,
]
{-1-x/2};
\addplot [
domain=-2:0, 
samples=100, 
color=black,
]
{1+x/2};

\addplot [
domain=0:1, 
samples=100, 
color=red,
]
{-1+1*x};
\addplot [
domain=-1:0, 
samples=100, 
color=red,
]
{-1-1*x};
\addplot [
domain=-1:0, 
samples=100, 
color=red,
]
{1+1*x};

\end{axis}

\begin{scope}[shift={(8,0)}]
\begin{axis}[
axis lines = center,
xmin = -2,
ymin = -2,
xmax = 2,
ymax = 2,
xtick={-2,-1,0,1,2},
xticklabels={$\nicefrac{-2}{6}$,$\nicefrac{-1}{6}$,,$\nicefrac{1}{6}$,$\nicefrac{2}{6}$},
ytick={-2,-1,0,1,2},
yticklabels={$\nicefrac{-2}{6}$,$\nicefrac{-1}{6}$,,$\nicefrac{1}{6}$,$\nicefrac{2}{6}$},
legend style={at={(1.4,0.62)},anchor=south east,legend columns=1,column sep=0.5em},
width = 6.4 cm,
height = 4 cm
]
\addplot [
domain=0:1, 
samples=100, 
color=black,
]
{2-2*x};
\addlegendentry{ours, $|x^2| \leq 1$}
\addplot [
domain=0:0.6, 
samples=100, 
color=red,
]
{0.6-x};
\addlegendentry{\cite{DBLP:journals/corr/SzegedyZSBEGF13}, $|x^2| \leq 1$}
\addplot [
domain=0:1, 
samples=100, 
color=black,
]
{-2+2*x};
\addplot [
domain=-1:0, 
samples=100, 
color=black,
]
{-2-2*x};
\addplot [
domain=-1:0, 
samples=100, 
color=black,
]
{2+2*x};

\addplot [
domain=0:0.6, 
samples=100, 
color=red,
]
{-0.6+x};
\addplot [
domain=-0.6:0, 
samples=100, 
color=red,
]
{-0.6-x};
\addplot [
domain=-0.6:0, 
samples=100, 
color=red,
]
{0.6+x};
\end{axis}
\end{scope}
\end{tikzpicture}
}%
\vspace{-30pt}%
        \caption{On the left: The region such that we prove that the input $x^1$ of the hidden layer maps to output $x^2$ with $|x^2| \leq 1$. On the right: The region such that the input $x^0$ maps to output $x^2$ with $|x^2| \leq 1$. In black: our approach, in red: \cite{DBLP:journals/corr/SzegedyZSBEGF13}.}%
	\label{fig:weightednorms}%
\end{figure}%
\begin{example} \label{example:weighted}
Consider a neural network with 3 layers (1 hidden layer), matrices $A_1 = \Big({\small\begin{matrix} 4 & -1 \\[-2pt] -1 & 1 \end{matrix}}\Big)$ and $A_2 = \begin{pmatrix} 1 & 2 \end{pmatrix}$, biases $b_1 = \Big({\small\begin{matrix} 0  \\[-2pt] 0 \end{matrix}}\Big)$ and $b_2 = 0$, and ReLU activation functions. Define a weight system $\weightsys$ by $w_1^0 = 1$, $w_2^0 = \tfrac12$, $w_1^1 = \tfrac12$, $w_2^1 = 1$, and $w_1^2 = 1$. 
Then \cref{lem:matrixnorm} yields $\lipnetwork = \|A_1\|^{0,1}_{\weightsys}\|A_2\|^{1,2}_{\weightsys} = 3 \cdot 2 = 6$.
In contrast, the Lipschitz bound from \cite{DBLP:journals/corr/SzegedyZSBEGF13} is $\|A_1\|\|A_2\| = 5 \cdot 2 = 10$. While both approaches compute a bound of 2 corresponding to $A_2$, our approach using the weighted norms records that the effect of the first neuron is only $w_1^1 = \tfrac12$ times~$\|A_2\|=2$, which in turn yields a tighter bound for $A_1$, namely 3 instead of 5. We illustrate these better bounds for each layer in \cref{fig:weightednorms}. For the first layer, we only get a better bound for one of the directions, but then using this, we get a better bound for both directions and hence a better Lipschitz constant for the second layer.
\end{example}
Next, we show how to compute a weight system $\weightsys$ such that $\lipnetwork$ is lowest among all weight systems, for given weights on the output layer.\footnote{In practice, we set the weights on the output layer all to 1, since in the end we are interested in Lipschitz bounds for the unweighted 1-norm.} We call such a weight system \emph{optimal}.  Observe that the Lipschitz bound decreases when the weights on the input layer increase. 
This observation motivates the following optimality criterion for weights, which is based on the product of the Lipschitz bound and the weights on the input layer.
\begin{definition}[Optimality]
A weight system $\weightsys$ is \emph{optimal} for output weights $\weight^n$ if the Lipschitz bound satisfies $\lipnetwork\weight^0_j \leq \lipnetworkalt\weightalt^0_j $ for all $1 \leq j \leq m_0$ and all weight systems $\weightsysalt$ with output weights $\weight^n$, where $\weightalt^0$ are the input~weights~of~$\weightsysalt$.
\end{definition} 
\begin{lemma}[proof in {\ifappendix \cref{proof:optlem}\else \cite[App.~B.9]{Badings_CAV25_extended}\fi}]\label{lem:optlem}
If $\weightsys$ is optimal for output weights $\weight^n$, then $\lipnetwork \leq \lipnetworkalt$ for all weight systems $\weightsysalt$ with output weights $\weight^n$.
\end{lemma}
\begin{algorithm}[t]
\caption{Computing optimal weights.}\label{alg:weights}
\begin{algorithmic}
\Require{Output weights $\weight^n$ for output layer $n$, matrices $A_k \in \R^{m_k \times m_{k-1}}$ ($1 \leq k \leq n$) as in \cref{def:neural}.}
\Ensure{Input weights $w^0$ and a Lipschitz bound $K$ such that $(w^0, K)$ is optimal.}
\For{$\ell = n, \dots, 1$}
\State{$K_\ell$ $\gets$ $\!\!\!\max\limits_{1 \leq j \leq m_{\!\ell\!-\!1}} \!\!\!\sum_{i=1}^{m_\ell} \!\weight_i^\ell \! \left\lvert (A_\ell)_{ij} \right\rvert$  \Comment{Lipschitz constant $\| A_\ell \|^{\ell-1, \ell}_{\weightsys}$ if $w^{\ell-1}_j = 1$ for all $j$}}
\For{$j = 1, \dots, m_{\ell-1}$}
\State{$\weight_j^{\ell\!-\!1}$ $\!\gets$ $\!\frac1{K_\ell}\!\sum_{i=1}^{m_\ell} \! \weight_i^\ell \! \left\lvert (A_\ell)_{ij} \right\rvert$\quad\Comment{Smallest weight for which Lipschitz constant is $K_\ell$}}
\EndFor
\EndFor
\Return{$w^0$, $\prod_{\ell=1}^n K_\ell$} \Comment{Return input weights and Lipschitz bound $K$}
\end{algorithmic}
\end{algorithm}

We now explain how \cref{alg:weights} computes an optimal weight system.
Given weights $\weight^\ell_i$ for the space $\R^{m_{\ell}}$, we can set the normalized weights $\weight_j^k$ in \cref{lem:matrixnorm} proportional to $\sum_{i=1}^{m_\ell} \weight_i^\ell  \left\lvert M_{ij} \right\rvert$, which implies that the maximum in \cref{lem:matrixnorm} is attained for all $1 \leq j \leq m_k$. \cref{alg:weights} starts from given output weights~$w^n_i$ and iteratively computes weights $\weight_j^{\ell-1}$ given weights $\weight_i^\ell$ in this way. Then the maximum in \cref{lem:matrixnorm} is attained by all $1 \leq j \leq m_{\ell-1}$ for the matrix $M = A_\ell$.

\begin{theorem}[Correctness of \cref{alg:weights}; proof in {\ifappendix \cref{proof:optweights}\else \cite[App.~B.10]{Badings_CAV25_extended}\fi}] \label{thm:optweights}
Let output weights $\weight^n$ be given. 
Then the weights $\weight^\ell_j$ computed using  Algorithm \ref{alg:weights} are optimal for output weights $\weight^n$.
\end{theorem}

In practice, we set the output weights $\weight^n_i = 1$ for all $i$. 
Then \cref{alg:weights} computes a Lipschitz constant of $\network$ for the unweighted 1-norm, cf.\ \cref{lem:lipprod}.

\subsection{Averaged Activation Operators}\label{app:averagedactivationoperators}

Next, we explain how to combine weighted norms with \emph{averaged activation operators} \cite{baillon1978asymptotic,combettes2020lipschitz} to compute even tighter Lipschitz constants. Let $L_{\network}$ denote the Lipschitz constant of the neural network operator $\network$.

\begin{definition}
An $\alpha$-\emph{averaged activation operator} $(0 < \alpha < 1)$ is an operator $R \colon \R \rightarrow \R$ that satisfies $R = (1-\alpha) \Id + \alpha Q$ for some  $Q \colon \R \rightarrow \R$ with Lipschitz constant $1$ and identity function $\Id$.
\end{definition}

\noindent Since $\relu(x) = \tfrac12 x + \tfrac12|x|$, the $\relu$ is $\tfrac12$-averaged.
For simplicity, we only use $\tfrac12$-averaged activation operators.  
We extend a result of~\cite{combettes2020lipschitz} to weighted norms:

\begin{theorem}[proof in {\ifappendix \cref{proof:averagedactivation}\else \cite[App.~B.11]{Badings_CAV25_extended}\fi}] \label{thm:averagedactivation}
Consider an $(n+1)$-layer network with $\tfrac12$-averaged activation operators $R_k$.  Let  $\weightsys$ be a corresponding weight system. Let $S_n = \{(k_1, k_2, \ldots, k_r) \in \mathbb{N}_0^r \mid 0 \leq r \leq n-1, \, 1 \leq k_1 < k_2 < \dots < k_r \leq n-1\}$. Then, the Lipschitz constant $L_{\network}$ of the neural network operator $\network$ satisfies
\[
L_{\network} \leq \frac1{2^{n-1}} \sum_{(k_1, k_2, \ldots, k_r) \in S_n}\;\left[\prod_{\ell=1}^{r+1} \| A_{k_\ell} \ldots A_{k_{\ell-1}+1} \|^{k_{\ell-1}, k_\ell}_{\weightsys} \right],
\]
where we set $k_0 = 0$ and $k_{r+1} = n$.
\end{theorem}

\noindent For $n=2$, this yields $L_\network \leq \tfrac12 \big( \|A_2 A_1\|_{\weightsys}^{0,2} +\|A_2\|_{\weightsys}^{1,2}  \|A_1\|_{\weightsys}^{0,1} \big)$, which (by the submultiplicativity of the matrix norm) is smaller than $\|A_1\|_{\weightsys}^{1,2}  \|A_0\|_{\weightsys}^{0,1}$. 

In the general case, the submultiplicativity of the matrix norm implies that each of the $2^{n-1}$ summands in the sum is at most $\prod_{\ell=1}^n \|A_\ell\|_{\weightsys}^{\ell-1,\ell}$. 
Hence, the bound in \cref{thm:averagedactivation} is (for given weights $\weightsys$) tighter than the bound $\prod_{\ell=1}^n \| A_\ell \|^{\ell-1, \ell}_{\weightsys}$. 
The intuition for the result is that for the `identity part' of the averaged activation operator, we take the matrix product inside the matrix norm (which gives a smaller result than taking the product of the matrix norms).

The fact that \cref{thm:averagedactivation}  yields tighter bounds does not contradict the optimality in \cref{thm:optweights}, since \cref{thm:optweights} only applies if the formula $\prod_{\ell=1}^n \| A_\ell \|^{\ell-1, \ell}_{\weightsys}$ is used, while the bound in \cref{thm:averagedactivation} is always smaller for given weights.

\begin{example}  Consider the network introduced in \cref{example:weighted}, for which we have $A_2A_1 = \begin{pmatrix} 2 & 1 \end{pmatrix}$. Then just using averaged activation operators (as in \cite{combettes2020lipschitz}) yields a bound of $L_\network \leq \tfrac12 \left( \|A_2 A_1\| +\|A_2\|  \|A_1\| \right) = \tfrac12 \left( 2 + 2 \cdot 5 \right) = 6$, while using both weighted norms and averaged activation operators (\cref{thm:averagedactivation}) yields a bound of $L_\network \leq \tfrac12 \big(\|A_2 A_1\|_{\weightsys}^{0,2} +\|A_2\|_{\weightsys}^{1,2}  \|A_1\|_{\weightsys}^{0,1}\big)  = \tfrac12 \left( 2 + 2 \cdot 3 \right) = 4$.
\end{example}

\newcommand{\meshinit}{\tau_{\text{init}}}
\newcommand{\meshmin}{\tau_{\text{min}}}

\section{Learner-Verifier Framework}
\label{sec:Implementation}

Following \cite{DBLP:conf/aaai/ZikelicLHC23}, we implement our techniques from \cref{sec:Log,sec:Lipschitz} in the learner-verifier framework from \cref{fig:CEGIS}.
Given an initial policy $\policy$ (which we assume to be a neural network), the learner trains the certificate $V$ to be a logRASM.
The verifier checks whether $V$ is a discrete logRASM (as per \cref{def:dlograsm}), and thus whether $\exp(V)$ is a RASM.
Checking these conditions involves the Lipschitz constants $L_\policy$ and $L_V$, which we compute using our techniques from \cref{sec:Lipschitz}.
We terminate and return the certificate $V$ upon satisfaction of these conditions.
If the conditions are not satisfied, the verifier either refines the discretization or returns counterexamples to the learner. %
As the learner also updates the policy $\policy$, we effectively solve the following problem:
\begin{problemBold}[Policy synthesis]\label{problem2}
    Given a DTSS $\dtss$, compute a policy $\policy$ such that the reach-avoid specification $\tuple{\xTarget, \xUnsafe, \rho}$ is satisfied.
\end{problemBold}%
Termination of the learner-verifier implies that we have solved \cref{problem2}.

\newcommand{\Linit}{0}
\newcommand{\Lunsafe}{\text{U}}
\newcommand{\Lexp}{{\Exp}}

\subsubsection{Verifier.}
Recall that the verifier checks the discrete logRASM conditions from \cref{def:dlograsm} on a discretization $\Xdisc$ of the state space.
When the verifier finds a point $\statedisc \in \Xdisc$ that violates these conditions, we either decrease the mesh size $\tau_{\statedisc}$ of $\statedisc$ (to try and mitigate the violation) or return $\statedisc$ as a counterexample to the learner.

\paragraph{Local refinement.}
Decreasing the mesh size $\tau_{\statedisc}$ of the point $\statedisc$ can only mitigate a violation if $\statedisc$ is not a \emph{hard violation} of the logRASM conditions.
Hard violations are points $\statedisc \in \Xdisc$ that already suffice to prove that the current candidate certificate $V$ is not a logRASM.
For example, consider a point $\statedisc \in \Xdisc \cap \X_0$ that violates the discrete logRASM initial condition, i.e., $\Vub(\statedisc) > 0$.
If the logRASM initial condition is also violated, i.e., $V(\statedisc) > 0$, then $V$ cannot be a logRASM, so $\statedisc$ is a hard violation.
We use an analogous argument for the other conditions.

We iteratively refine the discretization as long as \emph{none of the violations} are hard violations, similar to common abstraction refinement schemes~\cite{DBLP:journals/jacm/ClarkeGJLV03,DBLP:conf/formats/DierksKL07,DBLP:conf/hybrid/TiwariK02}.
Specifically, we split the set $\cell^{\tau_{\statedisc}}_\infty(\statedisc)$ associated with each (non-hard) violation $\statedisc $ into multiple smaller cells whose mesh size $\tau_{\statedisc}$ is reduced by a factor of $C \in (0,1)$.
In the context of supermartingale certificates, such refinements are also used by~\cite{DBLP:conf/atva/AnsaripourCHLZ23}.
As a novel aspect, we observe that the reduction in $\tau_{\statedisc}$ needed to mitigate a violation depends on the degree to which a condition is violated, so we use a different factor $C$ for each violation. We discuss in {\ifappendix \cref{app:suggestedmesh}\else \cite[App.~A.2]{Badings_CAV25_extended}\fi} how we compute informed values for $C$.
Importantly, the verifier only still needs to check the discrete logRASM conditions for the points associated with these new cells: points~$\statedisc$ that are not a violation cannot become a violation due to a discretization with a smaller mesh size $\tau_{\statedisc}$. 

\paragraph{Counterexamples.}
When the verifier finds at least one hard violation, we stop the refinement and return \emph{all violations} $\Xdisc' \subseteq \Xdisc$ to the learner.
These violations of the initial, safety, and expected decrease conditions are, respectively, added to three sets of counterexamples, denoted by $C_{\Linit}$, $C_{\Lunsafe}$, and $C_{\Lexp}$.
However, if there are many violations, these counterexample sets become large.
Thus, we implement these sets as buffers of a fixed size and, in each iteration, randomly replace a fixed fraction of the samples with new counterexamples.

\subsubsection{Learner.}
The learner trains the certificate $V$ and the policy $\policy$ on a differentiable version of the logRASM conditions in \cref{def:lograsm}.
The learner minimizes the loss function
$\loss(\pi, V) = \loss_\Linit(V) +\loss_\Lunsafe(V) + \alpha \cdot \loss_\Lexp(\pi, V),$
with hyperparameter $\alpha \in \R_{\geq 0}$, and where each term models a differentiable version of a logRASM condition:

{\small
\vspace{-13pt}
\begin{align*}
    \loss_\Linit(V) &= \max_{\state \in P_0}\left\{\max\{V(\state) + \varepsilon, 0\}\right\},
    \\ 
    \loss_\Lunsafe(V) &= \frac{1}{\log(\tfrac{1}{1-\rho})} \max_{\state \in P_\Lunsafe}\left\{\max\big\{\log\big(\tfrac{1}{1-\rho}\big) - V(\state) + \varepsilon, 0\big\}\right\},
    \\
    \loss_\Lexp(\pi, V) &= \frac{1}{|P_\Lexp|} \!\sum_{\state \in P_\Lexp} \!\max\!\bigg\{
         \!\log\!\bigg[\frac1N\!\sum_{\noise_i \sim d}\! \exp\!\big[V(f(\state, \policy(\state), \noise_i))\big]\!\bigg] {-} V(\state) {+} \tau K' {+} \varepsilon'
    , 0  \bigg\}.
\end{align*}}%
The points $P_\Linit$, $P_\Lunsafe$, and $P_\Lexp$ over which we check the conditions consist of randomly sampled points (which are freshly sampled each epoch) and the respective counterexamples $C_{\Linit}$, $C_{\Lunsafe}$, and $C_{\Lexp}$ returned by the verifier in previous iterations.
The loss $\loss_\Lexp(\pi, V)$ approximates the expected decrease condition over a finite number $N$ of noise samples, $\noise_i \sim d$.
The terms $\varepsilon, \varepsilon' \in \R_{\geq 0}$ ensure that a loss of zero implies that the logRASM conditions are strictly satisfied at the points in the sets $P_\Linit$, $P_\Lunsafe$, and $P_\Lexp$.
Finally, $K' = K + L_V = L_V(L_f(L_\policy+1)+1)$ is the Lipschitz constant of the function $\state \mapsto \log \Exp_{\noise \sim \noisedist}[\exp(V(f(\state, \policy(\state), \noise)))] - V(\state)$, and $\tau$ is a \emph{loss mesh} size chosen specifically for the problem. %
\section{Empirical Evaluation}
\label{sec:Empirical}
We perform numerical experiments to answer the following questions about our techniques, implemented in the learner-verifier framework described in \cref{sec:Implementation}:

\begin{enumerate}[nosep, leftmargin=.7cm]
    \item[Q1:] Can our methods be used to verify reach-avoid specifications with high probability bounds in challenging benchmarks?
    \item[Q2:] Is our learner-verifier framework robust to deviations in the input policy?
    \item[Q3:] How does our method for computing Lipschitz constants (\cref{sec:Lipschitz}) compare to other methods for computing Lipschitz constants of neural networks?
\end{enumerate}

\paragraph{Setup.}
All experiments are run on a server running Debian, with an AMD Ryzen Threadripper PRO 5965WX CPU, 512~GB of RAM, and an  NVIDIA GeForce RTX~4090 GPU.
Our Python implementation uses JAX~\cite{jax2018github} (v0.4.26) with GPU acceleration.
The policy and certificate neural networks both consist of 3 hidden layers of 128 neurons each.
See {\ifappendix \cref{app:Hyperparameters}\else \cite[App.~C.2]{Badings_CAV25_extended}\fi} for all hyperparameters.
Our implementation is available at \url{https://doi.org/10.5281/zenodo.15214887}.

\subsection*{Q1. Verifying Reach-Avoid Specifications}
We compare learner-verifier frameworks that implement different combinations of our verifier techniques: \texttt{logRASM+Lip} is our proposed verifier as described in \cref{sec:Implementation} (i.e., using both logRASMs and improved Lipschitz bounds), \texttt{logRASM} only uses logRASMs, \texttt{Lip} only uses improved Lipschitz bounds, and the \texttt{baseline} uses neither.
Since \texttt{Lip} and \texttt{baseline} train a (standard) RASM, these learner-verifiers use a different loss function (cf. {\ifappendix \cref{app:lossrasm}\else \cite[App.~C.3]{Badings_CAV25_extended}\fi}) based on the RASM conditions.
The verifier in the \texttt{baseline} checks (except for \cref{remark:deviation}) the same discrete RASM conditions as in~\cite{DBLP:conf/aaai/ZikelicLHC23}, but our learner-verifier framework differs in several algorithmic aspects.
To obtain a fairer comparison between the cases, we use our own implementation as a baseline that we can also run on the same hardware.
However, the baseline results are generally competitive with those in~\cite{DBLP:conf/aaai/ZikelicLHC23}.

\paragraph{Benchmarks.}
We consider all benchmarks from~\cite{DBLP:conf/aaai/ZikelicLHC23} (\texttt{linear-sys}, \texttt{pendulum}, and \texttt{collision-avoid}), as well as a version of \texttt{linear-sys} with a more challenging layout.
These four benchmarks have 2D state spaces.
In addition, to assess the limits of our approach, we consider more challenging benchmarks with 3D and 4D state spaces.
We consider reach-avoid specifications with different probability bounds ranging from $\rho=0.8$ to $0.999999$.
We pretrain all policies with proximal policy optimization (PPO)~\cite{DBLP:journals/corr/SchulmanWDRK17} for 100{,}000 steps, which takes less than $30$ seconds per instance (except for \texttt{drone4D} and \texttt{planar-robot}, which are trained for $1$ and $10$ million steps, respectively).
We use a loss function that also penalizes high Lipschitz constants.
For details on the benchmarks, we refer to~{\ifappendix \cref{app:Models}\else \cite[App.~C.1]{Badings_CAV25_extended}\fi}.

{
\setlength{\tabcolsep}{2.2pt}
\begin{table*}[t]
\centering
\caption{Average runtimes (in sec.) and st.dev. over 10 seeds (timeout of $\SI{30}{\minute}$; $d$ and $m$ are the state and action space dimensions). See {\ifappendix \cref{app:Hyperparameters}\else \cite[App.~C.2]{Badings_CAV25_extended}\fi} for the hyperparameters. An instance is considered as failed if 3 or more seeds time out.
}\label{tab:main}

\scalebox{0.8}{
\begin{threeparttable}
\begin{tabular}{@{}llllllllll@{}}
\toprule
& & & & \multicolumn{6}{c}{{Probability bound $\rho$}} \\
\cmidrule(lr){5-10}
Benchmark & $d$ & $m$ & Learner-verifier & \multicolumn{1}{c}{0.8} & \multicolumn{1}{c}{0.9} & \multicolumn{1}{c}{0.99} & \multicolumn{1}{c}{0.999} & \multicolumn{1}{c}{0.9999} & \multicolumn{1}{c}{0.999999} \\
\midrule\multirow{4}{*}{\texttt{linear-sys}} & \multirow{4}{*}{{2}} & \multirow{4}{*}{{1}} & \texttt{logRASM+Lip (ours)\!\!} & $\hphantom{00}$$47 \,\scriptstyle{\pm 5}$${}^{}$  & $\hphantom{00}$$50 \,\scriptstyle{\pm 6}$${}^{}$  & $\hphantom{00}$$52 \,\scriptstyle{\pm 6}$${}^{}$  & $\hphantom{00}$$50 \,\scriptstyle{\pm 7}$${}^{}$  & $\hphantom{00}$$51 \,\scriptstyle{\pm 8}$${}^{}$  & $\hphantom{00}$$42 \,\scriptstyle{\pm 6}$${}^{}$  \\
& & & \texttt{logRASM} & $\hphantom{00}$$54 \,\scriptstyle{\pm 1}$${}^{}$  & $\hphantom{00}$$53 \,\scriptstyle{\pm 1}$${}^{}$  & $\hphantom{00}$$52 \,\scriptstyle{\pm 1}$${}^{}$  & $\hphantom{00}$$52 \,\scriptstyle{\pm 1}$${}^{}$  & $\hphantom{00}$$52 \,\scriptstyle{\pm 1}$${}^{}$  & $\hphantom{00}$$51 \,\scriptstyle{\pm 2}$${}^{}$  \\
& & & \texttt{Lip} & $\hphantom{00}$$45 \,\scriptstyle{\pm 3}$${}^{}$  & $\hphantom{00}$$42 \,\scriptstyle{\pm 5}$${}^{}$  & $\hphantom{00}$$79 \,\scriptstyle{\pm 18}$${}^{}$  & $\hphantom{0}$$180 \,\scriptstyle{\pm 66}$${}^{}$  & $\hphantom{0}$$545 \,\scriptstyle{\pm 117}$${}^{*}$  & \multicolumn{1}{c}{--} \\
& & & \texttt{baseline} & $\hphantom{00}$$88 \,\scriptstyle{\pm 10}$${}^{}$  & $\hphantom{00}$$89 \,\scriptstyle{\pm 4}$${}^{}$  & $\hphantom{0}$$308 \,\scriptstyle{\pm 157}$${}^{}$  & $\hphantom{0}$$699 \,\scriptstyle{\pm 224}$${}^{}$  & \multicolumn{1}{c}{--} & \multicolumn{1}{c}{--} \\
\midrule\multirow{4}{*}{\begin{tabular}{@{}l@{}}\texttt{linear-sys} \\ (hard layout)\!\end{tabular}} & \multirow{4}{*}{{2}} & \multirow{4}{*}{{1}} & \texttt{logRASM+Lip (ours)\!\!} & $\hphantom{0}$$103 \,\scriptstyle{\pm 9}$${}^{}$  & $\hphantom{0}$$109 \,\scriptstyle{\pm 7}$${}^{}$  & $\hphantom{0}$$110 \,\scriptstyle{\pm 5}$${}^{}$  & $\hphantom{0}$$127 \,\scriptstyle{\pm 4}$${}^{}$  & $\hphantom{0}$$138 \,\scriptstyle{\pm 25}$${}^{}$  & $\hphantom{0}$$175 \,\scriptstyle{\pm 25}$${}^{*}$  \\
& & & \texttt{logRASM} & $\hphantom{0}$$283 \,\scriptstyle{\pm 40}$${}^{}$  & $\hphantom{0}$$386 \,\scriptstyle{\pm 73}$${}^{}$  & $\hphantom{0}$$668 \,\scriptstyle{\pm 151}$${}^{}$  & \multicolumn{1}{c}{--} & \multicolumn{1}{c}{--} & \multicolumn{1}{c}{--} \\
& & & \texttt{Lip} & \multicolumn{1}{c}{--} & \multicolumn{1}{c}{--} & \multicolumn{1}{c}{--} & \multicolumn{1}{c}{--} & \multicolumn{1}{c}{--} & \multicolumn{1}{c}{--} \\
& & & \texttt{baseline} & \multicolumn{1}{c}{--} & \multicolumn{1}{c}{--} & \multicolumn{1}{c}{--} & \multicolumn{1}{c}{--} & \multicolumn{1}{c}{--} & \multicolumn{1}{c}{--} \\
\midrule\multirow{4}{*}{\texttt{pendulum}} & \multirow{4}{*}{{2}} & \multirow{4}{*}{{2}} & \texttt{logRASM+Lip (ours)\!\!} & $\hphantom{00}$$77 \,\scriptstyle{\pm 10}$${}^{}$  & $\hphantom{00}$$71 \,\scriptstyle{\pm 2}$${}^{}$  & $\hphantom{00}$$85 \,\scriptstyle{\pm 2}$${}^{}$  & $\hphantom{00}$$99 \,\scriptstyle{\pm 11}$${}^{}$  & $\hphantom{0}$$107 \,\scriptstyle{\pm 11}$${}^{}$  & $\hphantom{0}$$137 \,\scriptstyle{\pm 43}$${}^{}$  \\
& & & \texttt{logRASM} & $\hphantom{0}$$226 \,\scriptstyle{\pm 23}$${}^{}$  & $\hphantom{0}$$229 \,\scriptstyle{\pm 26}$${}^{}$  & $\hphantom{0}$$216 \,\scriptstyle{\pm 13}$${}^{}$  & $\hphantom{0}$$221 \,\scriptstyle{\pm 30}$${}^{}$  & $\hphantom{0}$$239 \,\scriptstyle{\pm 28}$${}^{}$  & $\hphantom{0}$$218 \,\scriptstyle{\pm 7}$${}^{}$  \\
& & & \texttt{Lip} & $\hphantom{0}$$108 \,\scriptstyle{\pm 7}$${}^{}$  & $\hphantom{0}$$191 \,\scriptstyle{\pm 27}$${}^{}$  & \multicolumn{1}{c}{--} & \multicolumn{1}{c}{--} & \multicolumn{1}{c}{--} & \multicolumn{1}{c}{--} \\
& & & \texttt{baseline} & $\hphantom{0}$$721 \,\scriptstyle{\pm 168}$${}^{}$  & \multicolumn{1}{c}{--} & \multicolumn{1}{c}{--} & \multicolumn{1}{c}{--} & \multicolumn{1}{c}{--} & \multicolumn{1}{c}{--} \\
\midrule\multirow{4}{*}{\begin{tabular}{@{}l@{}}\texttt{collision-} \\ \texttt{avoid}\end{tabular}} & \multirow{4}{*}{{2}} & \multirow{4}{*}{{2}} & \texttt{logRASM+Lip (ours)\!\!} & $\hphantom{00}$$69 \,\scriptstyle{\pm 1}$${}^{**}$  & $\hphantom{00}$$68 \,\scriptstyle{\pm 2}$${}^{}$  & $\hphantom{00}$$94 \,\scriptstyle{\pm 5}$${}^{}$  & $\hphantom{0}$$108 \,\scriptstyle{\pm 2}$${}^{}$  & $\hphantom{0}$$122 \,\scriptstyle{\pm 2}$${}^{}$  & $\hphantom{0}$$137 \,\scriptstyle{\pm 3}$${}^{}$  \\
& & & \texttt{logRASM} & $\hphantom{0}$$107 \,\scriptstyle{\pm 1}$${}^{}$  & $\hphantom{0}$$116 \,\scriptstyle{\pm 10}$${}^{}$  & $\hphantom{0}$$147 \,\scriptstyle{\pm 10}$${}^{}$  & $\hphantom{0}$$170 \,\scriptstyle{\pm 9}$${}^{}$  & $\hphantom{0}$$188 \,\scriptstyle{\pm 11}$${}^{}$  & $\hphantom{0}$$227 \,\scriptstyle{\pm 3}$${}^{}$  \\
& & & \texttt{Lip} & $\hphantom{0}$$117 \,\scriptstyle{\pm 6}$${}^{}$  & $\hphantom{0}$$152 \,\scriptstyle{\pm 12}$${}^{}$  & $\hphantom{0}$$391 \,\scriptstyle{\pm 42}$${}^{}$  & \multicolumn{1}{c}{--} & \multicolumn{1}{c}{--} & \multicolumn{1}{c}{--} \\
& & & \texttt{baseline} & $\hphantom{0}$$252 \,\scriptstyle{\pm 16}$${}^{**}$  & \multicolumn{1}{c}{--} & \multicolumn{1}{c}{--} & \multicolumn{1}{c}{--} & \multicolumn{1}{c}{--} & \multicolumn{1}{c}{--} \\
\midrule\multirow{4}{*}{\begin{tabular}{@{}l@{}}\texttt{triple-} \\ \texttt{integrator}\end{tabular}} & \multirow{4}{*}{{3}} & \multirow{4}{*}{{1}} & \texttt{logRASM+Lip (ours)\!\!} & $\hphantom{0}$$793 \,\scriptstyle{\pm 180}$${}^{*}$  & $\hphantom{0}$$700 \,\scriptstyle{\pm 258}$${}^{}$  & $\hphantom{0}$$630 \,\scriptstyle{\pm 114}$${}^{}$  & $\hphantom{0}$$675 \,\scriptstyle{\pm 156}$${}^{}$  & $\hphantom{0}$$597 \,\scriptstyle{\pm 111}$${}^{}$  & $\hphantom{0}$$606 \,\scriptstyle{\pm 108}$${}^{}$  \\
& & & \texttt{logRASM} & \multicolumn{1}{c}{--} & $\hphantom{}$$1394 \,\scriptstyle{\pm 148}$${}^{}$  & $\hphantom{}$$1397 \,\scriptstyle{\pm 134}$${}^{**}$  & \multicolumn{1}{c}{--} & $\hphantom{}$$1396 \,\scriptstyle{\pm 228}$${}^{}$  & \multicolumn{1}{c}{--} \\
& & & \texttt{Lip} & $\hphantom{}$$1430 \,\scriptstyle{\pm 182}$${}^{**}$  & \multicolumn{1}{c}{--} & \multicolumn{1}{c}{--} & \multicolumn{1}{c}{--} & \multicolumn{1}{c}{--} & \multicolumn{1}{c}{--} \\
& & & \texttt{baseline} & \multicolumn{1}{c}{--} & \multicolumn{1}{c}{--} & \multicolumn{1}{c}{--} & \multicolumn{1}{c}{--} & \multicolumn{1}{c}{--} & \multicolumn{1}{c}{--} \\
\midrule\multirow{4}{*}{\begin{tabular}{@{}l@{}}\texttt{planar-} \\ \texttt{robot}\end{tabular}} & \multirow{4}{*}{{3}} & \multirow{4}{*}{{2}} & \texttt{logRASM+Lip (ours)\!\!} & $\hphantom{0}$$326 \,\scriptstyle{\pm 44}$${}^{}$  & $\hphantom{0}$$380 \,\scriptstyle{\pm 89}$${}^{}$  & $\hphantom{0}$$341 \,\scriptstyle{\pm 58}$${}^{}$  & $\hphantom{0}$$341 \,\scriptstyle{\pm 94}$${}^{}$  & $\hphantom{0}$$491 \,\scriptstyle{\pm 99}$${}^{}$  & \multicolumn{1}{c}{--} \\
& & & \texttt{logRASM} & $\hphantom{0}$$720 \,\scriptstyle{\pm 262}$${}^{}$  & \multicolumn{1}{c}{--} & \multicolumn{1}{c}{--} & \multicolumn{1}{c}{--} & \multicolumn{1}{c}{--} & \multicolumn{1}{c}{--} \\
& & & \texttt{Lip} & \multicolumn{1}{c}{--} & \multicolumn{1}{c}{--} & \multicolumn{1}{c}{--} & \multicolumn{1}{c}{--} & \multicolumn{1}{c}{--} & \multicolumn{1}{c}{--} \\
& & & \texttt{baseline} & \multicolumn{1}{c}{--} & \multicolumn{1}{c}{--} & \multicolumn{1}{c}{--} & \multicolumn{1}{c}{--} & \multicolumn{1}{c}{--} & \multicolumn{1}{c}{--} \\
\midrule\multirow{4}{*}{\texttt{drone4D}} & \multirow{4}{*}{{4}} & \multirow{4}{*}{{2}} & \texttt{logRASM+Lip (ours)\!\!} & $\hphantom{0}$$665 \,\scriptstyle{\pm 282}$${}^{**}$  & $\hphantom{0}$$656 \,\scriptstyle{\pm 164}$${}^{}$  & $\hphantom{0}$$765 \,\scriptstyle{\pm 276}$${}^{*}$  & $\hphantom{0}$$873 \,\scriptstyle{\pm 124}$${}^{}$  & \multicolumn{1}{c}{--} & \multicolumn{1}{c}{--} \\
& & & \texttt{logRASM} & \multicolumn{1}{c}{--} & \multicolumn{1}{c}{--} & \multicolumn{1}{c}{--} & \multicolumn{1}{c}{--} & \multicolumn{1}{c}{--} & \multicolumn{1}{c}{--} \\
& & & \texttt{Lip} & \multicolumn{1}{c}{--} & \multicolumn{1}{c}{--} & \multicolumn{1}{c}{--} & \multicolumn{1}{c}{--} & \multicolumn{1}{c}{--} & \multicolumn{1}{c}{--} \\
& & & \texttt{baseline} & \multicolumn{1}{c}{--} & \multicolumn{1}{c}{--} & \multicolumn{1}{c}{--} & \multicolumn{1}{c}{--} & \multicolumn{1}{c}{--} & \multicolumn{1}{c}{--} \\
\bottomrule
\end{tabular}%

\begin{tablenotes}
        \raggedright
        \item[*] One timeout out of ten seeds; \,\, ${}^{**}$ Two timeouts out of ten seeds.
\end{tablenotes}
\end{threeparttable}%
}%
\vspace{-0.2em}
\end{table*}%
}%

\paragraph{Solving \cref{problem2}.}
We show that our method reliably learns verified policies with only minor parameter tuning on individual benchmarks.
Each instance is run on $10$ seeds and is considered as failed when $3$ or more seeds do not terminate within a $30$~minute timeout.
We run our learner-verifier framework with the same hyperparameters across all 2D benchmarks; for the 3D and 4D benchmarks, we only slightly tune hyperparameters to adapt to these higher dimensions (see {\ifappendix \cref{app:Hyperparameters}\else \cite[App.~C.2]{Badings_CAV25_extended}\fi} for details).
The average times required to find a valid (log)RASM are presented in \cref{tab:main}
(excluding the time to train input policies).
For all benchmarks, our new method is able to consistently verify (much) \emph{higher probability bounds} $\rho$ ($99.9999\%$ for all 2D benchmarks) at \emph{lower run times} than the other learner-verifiers.
The best bounds successfully verified by our baseline are slightly lower than the values from \cite{DBLP:conf/aaai/ZikelicLHC23}.
However, we use a lower timeout (30 minutes instead of 3 hours) and consider an instance failed if $>2/10$ seeds failed, whereas \cite{DBLP:conf/aaai/ZikelicLHC23} reports the highest bound successfully verified.
Finally, the results for the 3D and 4D benchmarks clearly show that our method scales to benchmarks that were out of reach for the baseline.

\paragraph{Learned logRASMs.}
Four logRASMs learned using our method are shown in \cref{fig:RASMs}.
Especially the \texttt{linear-sys} (hard layout) benchmark requires a logRASM with a non-trivial shape, illustrating the usefulness of neural networks to represent certificates.
For a RASM with the same bound of $\rho = 0.999999$, the learner would train the certificate to have values up to at least $10^6$, which is required to satisfy the safety condition ($V(\state) \geq \frac{1}{1-\rho} = 10^6$).
By contrast, the learned logRASMs in \cref{fig:RASMs} only take values between $-20$ and $75$, making them easier to learn.

\begin{figure}[t!]
\centering
\includegraphics[height=2.5cm]{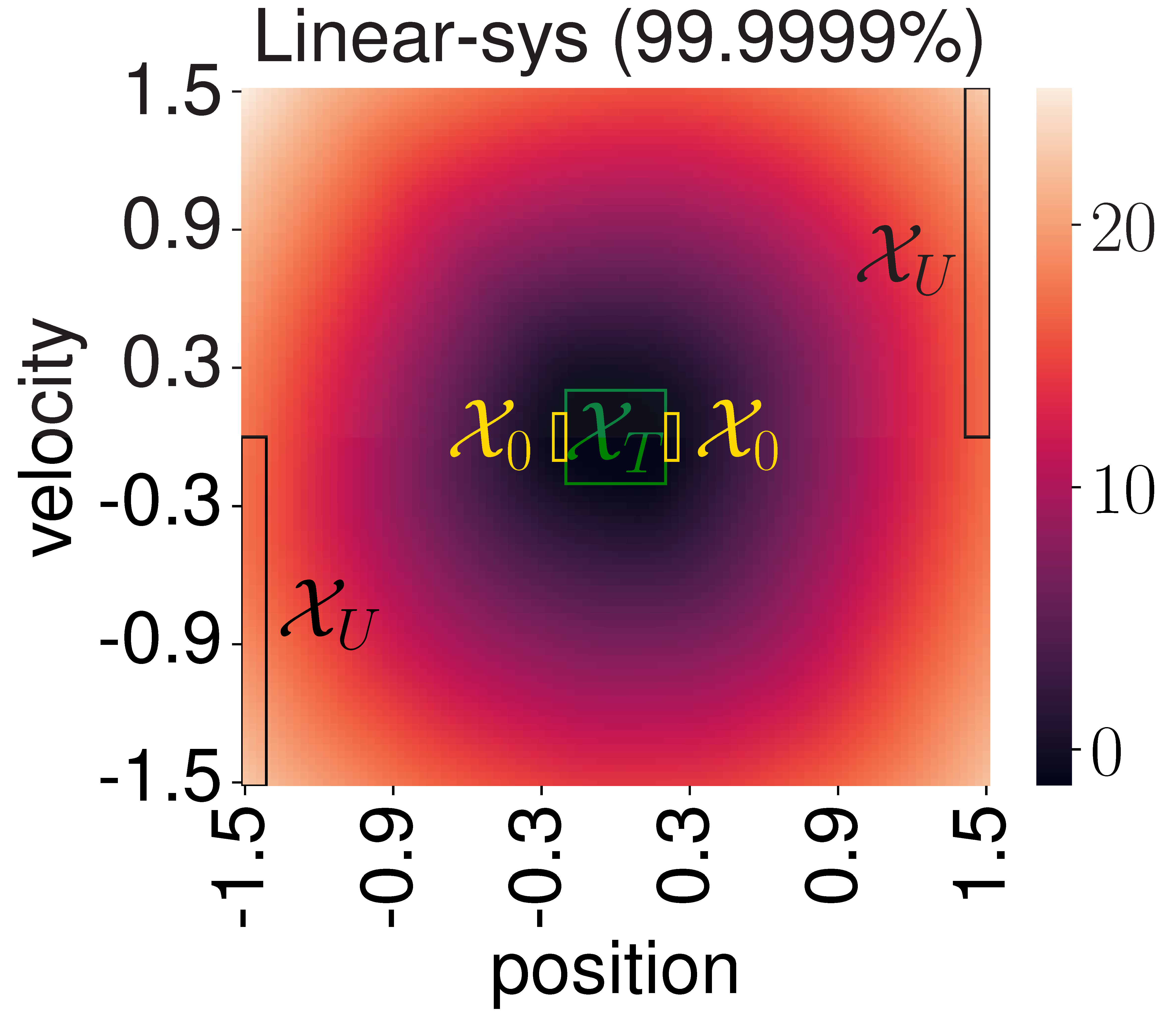}
\hfill
\includegraphics[height=2.5cm]{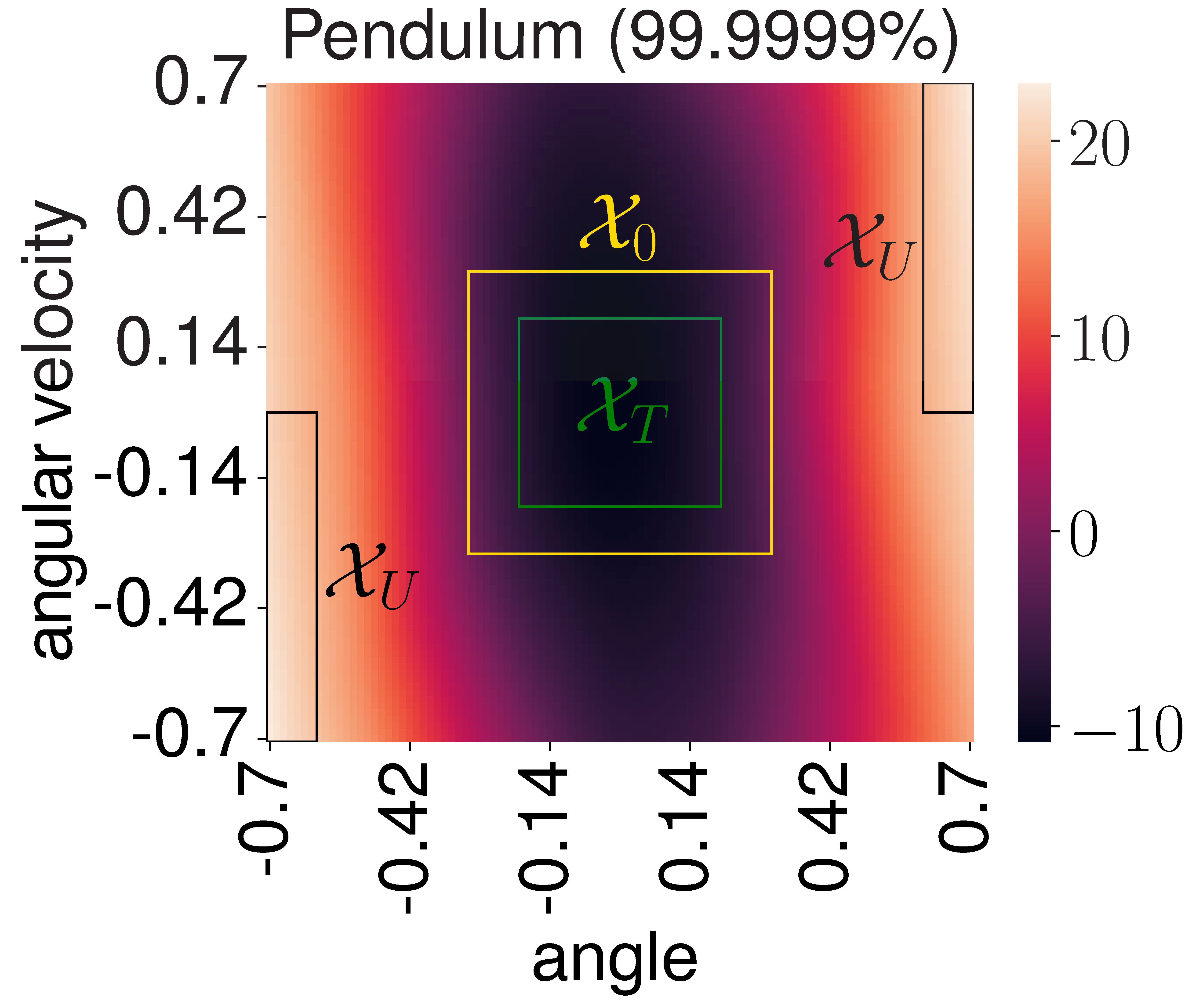}
\hfill
\includegraphics[height=2.5cm]{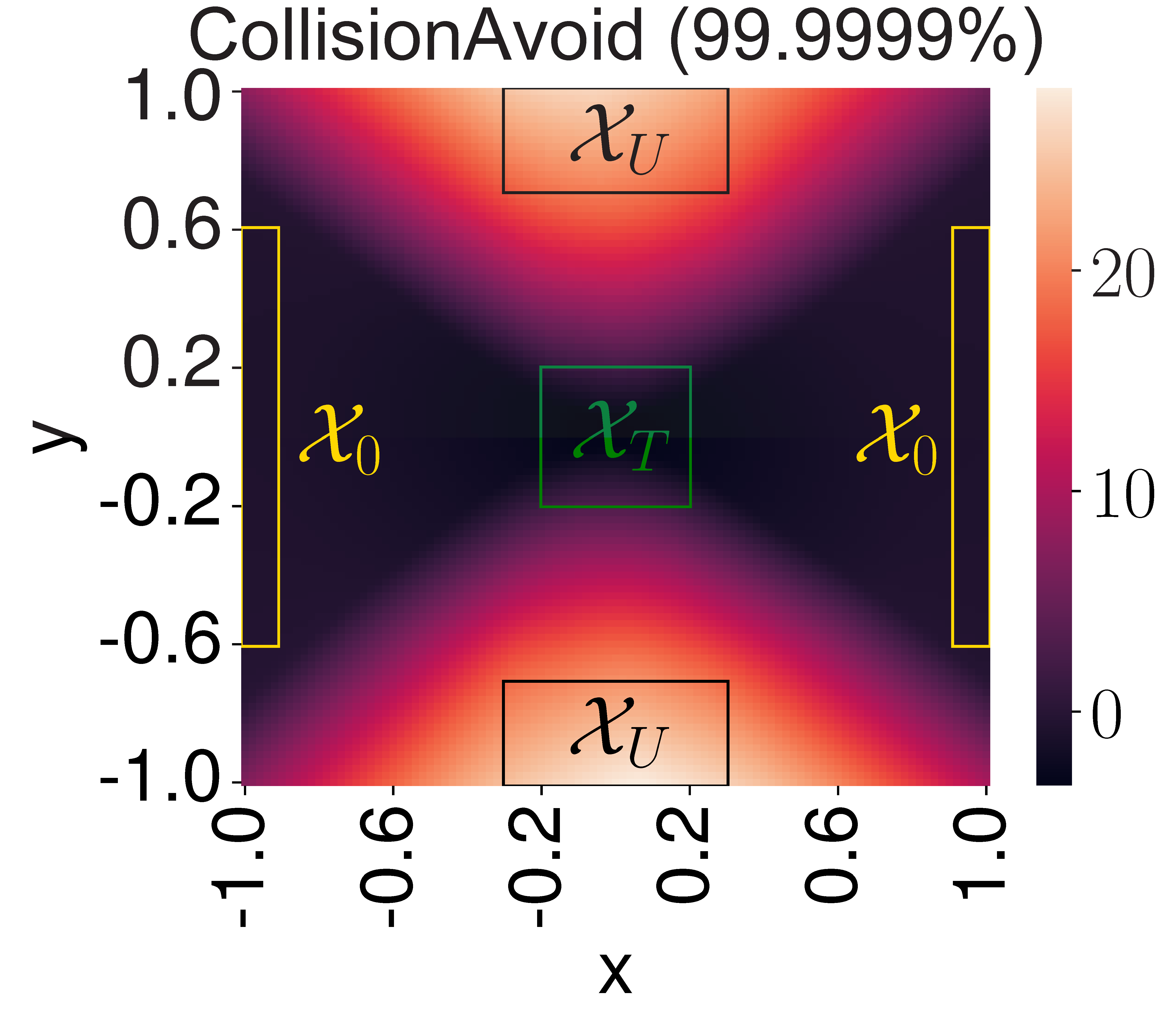}
\hfill
\includegraphics[height=2.5cm]{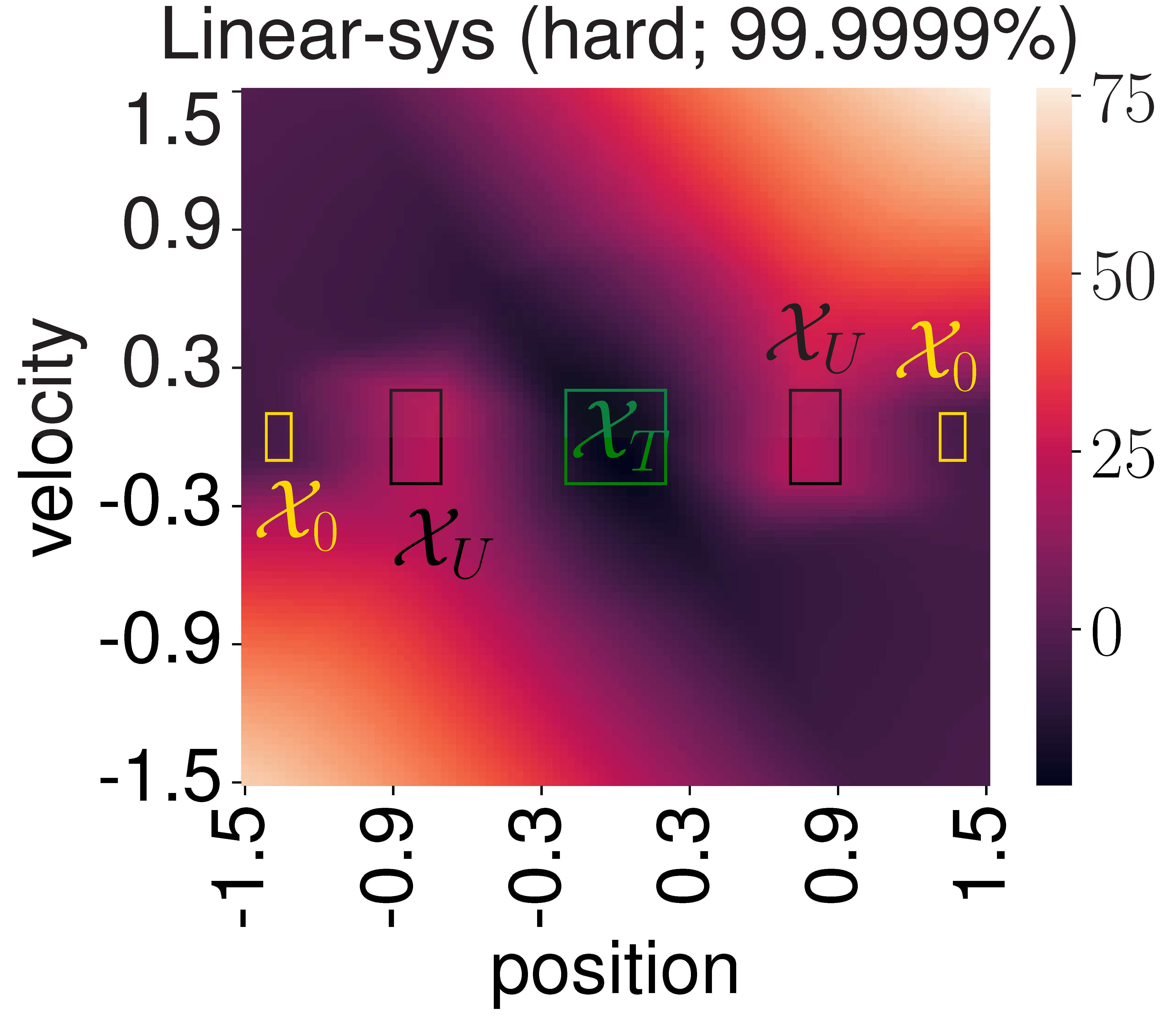}%
\caption{The logRASMs learned using our new method ($\texttt{logRASM+Lip}$).}
\label{fig:RASMs}
\end{figure}

{
\setlength{\tabcolsep}{2.2pt}
\begin{table*}[t]
\centering
\caption{Runtimes (in seconds) for verifying reach-avoid specifications (with probability $\rho=0.999999$) on input policies trained with several RL algorithms for different numbers of steps (avgs. and st.dev. over 10 seeds; timeout of $\SI{30}{\minute}$).}\label{tab:SB3}

\scalebox{0.85}{
\begin{threeparttable}
\begin{tabular}{@{}llllllllll@{}}
\toprule
& & \multicolumn{4}{c}{{$\alpha = 10, \, \tau = 0.0005$}} & \multicolumn{4}{c}{{$\alpha = 0.1, \, \tau = 0.001$}} \\
\cmidrule(lr){3-6}\cmidrule(lr){7-10} 
Benchmark & Steps & TRPO & TQC & SAC & A2C & TRPO & TQC & SAC & A2C \\
\midrule
\multirow{3}{*}{\texttt{linear-sys}} & $\num{1e+04}$ & $\hphantom{0}$$55 \,\scriptstyle{\pm 1}$ & $\hphantom{}$$135 \,\scriptstyle{\pm 37}$ & $\hphantom{}$$143 \,\scriptstyle{\pm 48}$ & $\hphantom{}$$112 \,\scriptstyle{\pm 32}$ & $\hphantom{0}$$86 \,\scriptstyle{\pm 2}$ & $\hphantom{}$$171 \,\scriptstyle{\pm 1}$ & $\hphantom{}$$167 \,\scriptstyle{\pm 9}$ & $\hphantom{}$$166 \,\scriptstyle{\pm 14}$ \\
& $\num{1e+05}$ & $\hphantom{0}$$66 \,\scriptstyle{\pm 18}$ & $\hphantom{}$$180 \,\scriptstyle{\pm 7}$ & $\hphantom{}$$177 \,\scriptstyle{\pm 18}$ & $\hphantom{}$$115 \,\scriptstyle{\pm 26}$ & $\hphantom{0}$$89 \,\scriptstyle{\pm 9}$ & $\hphantom{}$$185 \,\scriptstyle{\pm 19}$ & $\hphantom{}$$176 \,\scriptstyle{\pm 5}$ & $\hphantom{}$$170 \,\scriptstyle{\pm 1}$ \\
& $\num{1e+06}$ & $\hphantom{0}$$67 \,\scriptstyle{\pm 1}$ & $\hphantom{}$$173 \,\scriptstyle{\pm 3}$ & $\hphantom{}$$170 \,\scriptstyle{\pm 2}$ & $\hphantom{0}$$91 \,\scriptstyle{\pm 43}$ & $\hphantom{}$$109 \,\scriptstyle{\pm 23}$ & $\hphantom{}$$334 \,\scriptstyle{\pm 276}$${}^{*}$ & $\hphantom{}$$338 \,\scriptstyle{\pm 193}$${}^{*}$ & $\hphantom{}$$134 \,\scriptstyle{\pm 46}$ \\
\midrule\multirow{3}{*}{\begin{tabular}{@{}l@{}}\texttt{linear-sys} \\ (hard layout)\end{tabular}} & $\num{1e+04}$ & $\hphantom{}$$192 \,\scriptstyle{\pm 23}$ & $\hphantom{}$$240 \,\scriptstyle{\pm 3}$ & $\hphantom{}$$245 \,\scriptstyle{\pm 19}$ & $\hphantom{}$$226 \,\scriptstyle{\pm 31}$ & $\hphantom{}$$170 \,\scriptstyle{\pm 3}$ & $\hphantom{}$$246 \,\scriptstyle{\pm 41}$ & $\hphantom{}$$242 \,\scriptstyle{\pm 55}$ & $\hphantom{}$$242 \,\scriptstyle{\pm 20}$ \\
& $\num{1e+05}$ & $\hphantom{}$$188 \,\scriptstyle{\pm 3}$ & $\hphantom{}$$245 \,\scriptstyle{\pm 22}$ & $\hphantom{}$$236 \,\scriptstyle{\pm 3}$ & $\hphantom{}$$238 \,\scriptstyle{\pm 3}$ & $\hphantom{}$$237 \,\scriptstyle{\pm 16}$${}^{**}$ & $\hphantom{}$$173 \,\scriptstyle{\pm 19}$ & $\hphantom{}$$163 \,\scriptstyle{\pm 1}$ & $\hphantom{}$$256 \,\scriptstyle{\pm 40}$ \\
& $\num{1e+06}$ & $\hphantom{}$$212 \,\scriptstyle{\pm 39}$ & $\hphantom{}$$314 \,\scriptstyle{\pm 15}$ & $\hphantom{}$$316 \,\scriptstyle{\pm 18}$ & $\hphantom{}$$219 \,\scriptstyle{\pm 25}$ & $\hphantom{}$$264 \,\scriptstyle{\pm 34}$ & \multicolumn{1}{c}{--} & \multicolumn{1}{c}{--} & $\hphantom{}$$261 \,\scriptstyle{\pm 56}$ \\
\midrule\multirow{3}{*}{\texttt{pendulum}} & $\num{1e+04}$ & $\hphantom{}$$219 \,\scriptstyle{\pm 20}$${}^{*}$ & $\hphantom{}$$206 \,\scriptstyle{\pm 13}$ & \multicolumn{1}{c}{--} & $\hphantom{}$$196 \,\scriptstyle{\pm 21}$ & $\hphantom{}$$365 \,\scriptstyle{\pm 127}$ & $\hphantom{}$$652 \,\scriptstyle{\pm 312}$ & $\hphantom{}$$543 \,\scriptstyle{\pm 262}$ & $\hphantom{}$$415 \,\scriptstyle{\pm 184}$ \\
& $\num{1e+05}$ & \multicolumn{1}{c}{--} & $\hphantom{}$$279 \,\scriptstyle{\pm 42}$${}^{**}$ & $\hphantom{}$$295 \,\scriptstyle{\pm 66}$${}^{**}$ & $\hphantom{}$$193 \,\scriptstyle{\pm 24}$${}^{*}$ & $\hphantom{}$$427 \,\scriptstyle{\pm 180}$ & $\hphantom{}$$708 \,\scriptstyle{\pm 190}$ & $\hphantom{}$$400 \,\scriptstyle{\pm 161}$ & $\hphantom{}$$447 \,\scriptstyle{\pm 240}$ \\
& $\num{1e+06}$ & $\hphantom{}$$267 \,\scriptstyle{\pm 39}$${}^{*}$ & \multicolumn{1}{c}{--} & \multicolumn{1}{c}{--} & \multicolumn{1}{c}{--} & $\hphantom{}$$525 \,\scriptstyle{\pm 259}$ & $\hphantom{}$$498 \,\scriptstyle{\pm 194}$ & $\hphantom{}$$334 \,\scriptstyle{\pm 45}$ & $\hphantom{}$$496 \,\scriptstyle{\pm 251}$ \\
\midrule\multirow{3}{*}{\begin{tabular}{@{}l@{}}\texttt{collision-} \\ \texttt{avoid}\end{tabular}} & $\num{1e+04}$ & $\hphantom{}$$133 \,\scriptstyle{\pm 7}$ & $\hphantom{}$$194 \,\scriptstyle{\pm 7}$ & $\hphantom{}$$199 \,\scriptstyle{\pm 7}$ & $\hphantom{}$$197 \,\scriptstyle{\pm 12}$ & $\hphantom{}$$167 \,\scriptstyle{\pm 18}$ & $\hphantom{}$$175 \,\scriptstyle{\pm 2}$ & $\hphantom{}$$176 \,\scriptstyle{\pm 3}$ & $\hphantom{}$$183 \,\scriptstyle{\pm 21}$ \\
& $\num{1e+05}$ & $\hphantom{}$$101 \,\scriptstyle{\pm 3}$ & $\hphantom{}$$199 \,\scriptstyle{\pm 7}$ & $\hphantom{}$$198 \,\scriptstyle{\pm 8}$ & $\hphantom{}$$190 \,\scriptstyle{\pm 13}$ & $\hphantom{0}$$95 \,\scriptstyle{\pm 1}$ & $\hphantom{}$$176 \,\scriptstyle{\pm 2}$ & $\hphantom{}$$178 \,\scriptstyle{\pm 2}$ & $\hphantom{}$$174 \,\scriptstyle{\pm 2}$ \\
& $\num{1e+06}$ & $\hphantom{}$$191 \,\scriptstyle{\pm 28}$ & $\hphantom{}$$246 \,\scriptstyle{\pm 22}$ & $\hphantom{}$$276 \,\scriptstyle{\pm 30}$ & $\hphantom{}$$194 \,\scriptstyle{\pm 29}$ & $\hphantom{}$$169 \,\scriptstyle{\pm 16}$ & \multicolumn{1}{c}{--} & \multicolumn{1}{c}{--} & $\hphantom{}$$170 \,\scriptstyle{\pm 15}$ \\
\bottomrule
\end{tabular}

\begin{tablenotes}
        \raggedright
        \item[*] One timeout out of ten seeds; \,\, ${}^{**}$ Two timeouts out of ten seeds.
\end{tablenotes}
\end{threeparttable}
}

\end{table*}
}

\subsection*{Q2. Robustness to Input Policies}
We consider the same benchmarks as in \cref{tab:main} (with $\rho = 0.999999$ and with our \texttt{logRASM+Lip} learner-verifier) but now pretrain input policies using the \texttt{Stable-Baselines3}~\cite{stable-baselines3} implementation of the RL algorithms TRPO, TQC, SAC, and A2C (with default parameters; see {\ifappendix \cref{app:Models}\else \cite[App.~C.1]{Badings_CAV25_extended}\fi} for the loss functions) for either $10^4$, $10^5$ or $10^6$ steps.
Since we use these implementations unchanged, we now do not train for a lower Lipschitz constant (in contrast to the PPO-trained policies for Q1), but for a (state-based) reward function. Note that the RL reward maximization may not be able to fully capture the nature of a reach-avoid specification.
Each instance is run on 10 seeds and is considered failed when 3 or more seeds do not terminate within 30 minutes.

The run times in \cref{tab:SB3} show that our method is generally agnostic to the policy training algorithm.
We observe that training the policies longer tends to slightly increase the time to verify the policy, which can be a sign of over-training policies to maximize rewards.
Moreover, the values of the hyperparameters $\alpha$ and $\tau$ in the loss function~(cf.\ \cref{sec:Implementation}) influence the performance on individual benchmarks (e.g., we cannot reliably verify all \texttt{pendulum} policies for $\alpha =10$, $\tau = 0.0005$).
In conclusion, our method is reasonably robust against the input policy, but finding common hyperparameters for all benchmarks is difficult.

\subsection*{Q3. Comparison of Lipschitz Constants}

We demonstrate the need for our efficient method to compute Lipschitz constants when solving \cref{problem,problem2}.
We compare our techniques from \cref{sec:Lipschitz} against the anytime algorithm LipBaB~\cite{DBLP:conf/icann/BhowmickDR21}, a competitive solver for computing global Lipschitz constants, on the final policy and certificate networks (cf.\ {\ifappendix \cref{app:lipbab}\else \cite[App.~D]{Badings_CAV25_extended}\fi}). 
Our method takes \SI{0.2}{\second} to compute a Lipschitz constant (and only \SI{0.0002}{\second} when already JIT-compiled).
LipBaB returns a first Lipschitz constant after \SI{0.5}{\second} (which is, on average, 40\% larger than ours) and requires usually more than \SI{100}{\second} to compute a better Lipschitz constant than ours.
A typical benchmark requires 3--10 verifier iterations, each of which takes around \SI{20}{\second}, so better results from LipBaB may not outweigh the increase in verifier run time.
For example, even just using LipBaB in the final verifier-iteration would, on most benchmarks, more than double the total runtimes from \cref{tab:main}.

\subsection*{Discussion and Limitations}

Beyond the mentioned scalability limitation (w.r.t. the dimension of the state space) in Q1, our experiments do not address the following:
(1)~We did not consider the robustness w.r.t.\ the loss functions used for pretraining and the learner.
(2)~We did not consider multiplicative RASMs as introduced in \cite{DBLP:conf/nips/ZikelicLVCH23}, although our results would also apply to these RASMs.
(3)~We did not consider using only IBP for the expected decrease condition. This would require piecewise linear under- and overapproximations of the dynamics as proposed by \cite{DBLP:conf/nips/MazouzMRLL22,DBLP:journals/corr/MazouzBLL24}.

\section{Related Work}
\label{sec:Related}

Policy verification/synthesis for stochastic dynamical systems has largely been addressed using two approaches.
The first is to generate a model-based~\cite{DBLP:journals/ejcon/0001SRHH12,DBLP:journals/tac/ZamaniEMAL14,DBLP:journals/tac/LahijanianAB15,DBLP:journals/jair/BadingsRAPPSJ23} or data-driven~\cite{gracia2024,DBLP:conf/hybrid/JacksonLFL21} abstraction (e.g., as a finite Markov decision process) and use probabilistic model checking on this abstraction.
The second approach (which we take in this work) is to find a certificate function that implies the satisfaction of a specification.
These approaches differ from typical objectives in constrained~\cite{DBLP:conf/icml/AchiamHTA17,altman2021constrained} and safe RL~\cite{DBLP:conf/nips/BerkenkampTS017,DBLP:conf/aaai/AlshiekhBEKNT18}, which mostly focus on maximizing rewards while satisfying constraints on expected costs or safety in exploration~\cite{DBLP:journals/arcras/BrunkeGHYZPS22,DBLP:journals/jmlr/GarciaF15}.

Certificates are used in several areas, e.g., Lyapunov~\cite{khalil2002nonlinear,dequeiroz2000lyapunov} and control barrier functions~\cite{DBLP:journals/tac/AmesXGT17,DBLP:journals/csysl/LindemannD19,DBLP:conf/rss/ChoiCTS20} in control, and ranking functions~\cite{Floyd1993,DBLP:conf/cav/BradleyMS05,DBLP:conf/vmcai/PodelskiR04} in program analysis.
For stochastic systems, the value of the certificate along trajectories needs to be a supermartingale~\cite{DBLP:journals/tac/PrajnaJP07,DBLP:journals/automatica/Clark21,DBLP:journals/tac/JagtapSZ21}.
Besides the RASMs~\cite{DBLP:conf/aaai/ZikelicLHC23,DBLP:conf/nips/ZikelicLVCH23} we build upon in this paper,~\cite{neustroev2024} uses neural supermartingale certificates for continuous-time stochastic systems, and~\cite{DBLP:conf/cav/AbateGR24} proposes certificates for $\omega$-regular properties in stochastic systems but makes restrictive assumptions to achieve a practical algorithm.
Supermartingales are also used to analyze termination~\cite{DBLP:conf/cav/ChakarovS13,DBLP:journals/pacmpl/AgrawalC018,DBLP:conf/popl/ChatterjeeNZ17,DBLP:conf/cav/AbateGR20} and reachability~\cite{DBLP:journals/toplas/TakisakaOUH21} of probabilistic programs.
Various recent papers represent such certificates
as neural networks~\cite{DBLP:conf/nips/ChangRG19,DBLP:journals/csysl/AbateAGP21,DBLP:conf/corl/RichardsB018,DBLP:conf/nips/ZhouQSL22,DBLP:conf/nips/00010KV23,DBLP:journals/trob/DawsonGF23}.
The resulting candidate certificate (i.e., the neural network) can be verified using satisfiability modulo theories (SMT)~\cite{DBLP:conf/concur/AbateEGPR23,DBLP:conf/hybrid/AbateAEGP21}, branch-and-bound~\cite{DBLP:journals/csysl/MathiesenCL23}, or (like our approach) discretization and leveraging Lipschitz continuity~\cite{DBLP:conf/aaai/ZikelicLHC23,DBLP:conf/aaai/LechnerZCH22}.
Yet, all of these approaches are computationally expensive: SMT does not scale to large neural networks, whereas branch-and-bound and discretization do not scale with the state space dimension.

\subsection*{Neural Network Robustness and Lipschitz Constants}
The use of Lipschitz constants as a measure of neural network stability and robustness was pioneered by \cite{DBLP:journals/corr/SzegedyZSBEGF13}, who propose the product of the Lipschitz constants of each layer as an upper bound for the Lipschitz constant of the network. This bound is fast to compute, but also very loose. Recently, there has been significant work in devising methods for computing tighter \emph{global} and \emph{local} Lipschitz constants.
However, these methods are not designed for the large number of calls that our learner-verifier framework requires. Since the Lipschitz constant appears in the loss function, we need to recompute it for every batch and every epoch, leading to roughly 1{,}000 Lipschitz computations per learner-verifier iteration. Hence, even spending 20 ms on each Lipschitz computation would slow down our iterations by a factor two. Besides speed, another requirement is that the Lipschitz computation is differentiable, so that effects of weight updates on the Lipschitz constant are taken into account in the gradient of the loss function.

We now discuss why existing methods from the literature are (despite yielding tighter bounds on Lipschitz constants) less suited to our needs. Algorithms that compute {global} Lipschitz constants include~LipBaB \cite{DBLP:conf/nips/FazlyabRHMP19} and methods using semidefinite programming~\cite{DBLP:conf/icann/BhowmickDR21,DBLP:conf/iclr/Wang0HAZCJ24}. However, these methods are not differentiable, and have running times on the order of seconds per call. 
Methods for computing {local} Lipschitz constants include analytical bounds from \cite{avant2023analytical},  LiPopt \cite{DBLP:conf/iclr/GomezRC20}, LipMIP \cite{DBLP:conf/nips/JordanD20}, FastLin and FastLip \cite{DBLP:conf/icml/WengZCSHDBD18}, GenBaB\cite{DBLP:journals/corr/ShiJKJHZ24}, Recurjac \cite{DBLP:conf/nips/ShiW0KH22,DBLP:conf/aaai/ZhangZH19}. The analytical methods from \cite{avant2023analytical} are fast, but only apply relative to a fixed base point rather than within some region, which makes them unusable in our stochastic context. Out of the local methods, FastLin and FastLip \cite{DBLP:conf/icml/WengZCSHDBD18} are the fastest, but running times of 5 ms per call are still too slow in our context for a local method. 
 
We utilize results from~\cite{combettes2020lipschitz} in \cref{app:averagedactivationoperators}, which is to our knowledge the only method (besides \cite{DBLP:journals/corr/SzegedyZSBEGF13}) that can compute global Lipschitz constants sufficiently fast. We note that \cite{DBLP:journals/corr/abs-2412-01783}, which trains neural networks to certify the relation between two systems, also uses~\cite{combettes2020lipschitz} to compute their Lipschitz constants (and could therefore improve their results by using the method proposed in this paper instead). Anisotropic certification~\cite{DBLP:journals/tmlr/EirasATKDGB22} is similar to our weighted norms, but does not include an algorithm to compute optimal weights.

Besides approaches to bound Lipschitz constants, training networks to have a small Lipschitz constant is studied by~\cite{DBLP:conf/scalespace/BungertRRST21,DBLP:journals/ml/GoukFPC21,DBLP:journals/csysl/PauliKBKA22}. 
However, for our purposes, we need an upper bound of the Lipschitz constant, and training a network to have a low Lipschitz constant does not guarantee that an upper bound for that Lipschitz constant computed with a particular method is also small.
Another approach to neural network robustness is interval bound propagation (IBP), a technique to propagate interval inputs through neural networks~\cite{gowal2018effectiveness}.
Finally, a different line of research considers the adversarial robustness of neural networks~\cite{DBLP:conf/pkdd/BiggioCMNSLGR13,DBLP:journals/corr/SzegedyZSBEGF13,DBLP:conf/iclr/KurakinGB17a,DBLP:conf/cav/HuangKWW17,katz2017reluplex,DBLP:conf/sp/GehrMDTCV18}. 
We refer to the survey articles~\cite{DBLP:journals/csr/HuangKRSSTWY20,zuhlke2024adversarial} for a comprehensive overview of verification and robustness of neural networks.

\section{Conclusion}
\label{sec:Conclusion}

We presented two contributions to improve the verification of policies in stochastic systems using reach-avoid supermartingales (RASMs).
First, our logRASMs take exponentially lower values and hence have lower (theoretical) Lipschitz constants than (standard) RASMs.
Second, we compute tight bounds on Lipschitz constants by integrating the novel idea of weighted norms with averaged activation operators.
Our experiments show that our techniques allow the verification of reach-avoid specifications with much higher probability bounds than the state-of-the-art.

Future work includes generalizing our method for computing bounds on Lipschitz constants to broader classes of neural networks. 
In addition, while this work focuses on the verifier, improving the learner and the choice of counterexamples can improve the overall performance of the learner-verifier framework. 
Finally, we wish to investigate the robustness of the learner-verifier against perturbations in the system dynamics and the specification.

\begin{credits}
\subsubsection{\ackname}
This research has been funded by the ERC Starting Grant 101077178 (DEUCE), the EPSRC grant EP/Y028872/1 (Mathematical Foundations of Intelligence: An ``Erlangen Programme'' for AI), the Wallenberg AI, Autonomous Systems and Software Program (WASP) funded by the Knut and Alice Wallenberg Foundation, the NWO Veni Grant ProMiSe (222.147), and the NWO grant NWA.1160.18.238 (PrimaVera).

\subsubsection{\discintname}
The authors have no competing interests to declare that are relevant to the content of this article.
\end{credits}

\bibliographystyle{splncs04}
\bibliography{references}

\ifappendix

\newpage
\appendix

\section{Further Algorithmic Details}

\subsection{Split Lipschitz Constant of Dynamics}
\label{app:splitlip}

In this appendix, we describe an improvement over the formula $K = L_V L_f (L_\pi + 1)$ by analyzing the Lipschitz constant of the dynamics function $f \colon \X \times \U \times \mathcal{N} \rightarrow \X$ more carefully. Note that for $L_f$, we are interested in changes in the inputs $\state \in \X$ and $\control \in \U$, but take $\noise \in \mathcal{N}$ fixed. Hence, the Lipschitz constant $L_f$ satisfies
\[
\| f(\state', \control', \noise) -  f(\state, \control, \noise) \| \leq L_f\| (\state', \control') -  (\state, \control) \|.
\]
for all (fixed) $\noise \in \mathcal{N}$ and all $(\state, \control), (\state', \control') \in \X \times \U$.

We now compute two `split' Lipschitz constants: one Lipschitz constant $L_{f, \state}$ corresponding to changes in the state $\state$ (but keeping the action fixed), and one Lipschitz constant $L_{f, \control}$ corresponding to changes in the control $\control$ (but keeping the state fixed). Formally, we have 
\[
\| f(\state', \control, \noise) -  f(\state, \control, \noise) \| \leq L_{f, \state}\| \state' -  \state \|
\]
for fixed $\control \in \U$ and $\noise \in \mathcal{N}$, and all $\state, \state' \in \X$. Similarly, we have 
\[
\| f(\state, \control', \noise) -  f(\state, \control, \noise) \| \leq L_{f, \control}\| \control' -  \control \|
\]
for fixed $\state \in \X$ and $\noise \in \mathcal{N}$, and all $\control, \control' \in \U$. For a given dynamics function $f$, we can compute (upper bounds for) $L_{f, \state}$ and $L_{f, \control}$. Note that we always have $L_{f, \state} \leq L_f$, and $L_{f, \control} \leq L_f$. 

We now explain why we can replace $L_f (L_\pi + 1)$ by $L_{f, \state} + L_{f, \control} L_\pi$, as bound for the Lipschitz constant of the function $\state \mapsto f(\state, \policy(\state), \noise)$ for all fixed $\noise \in \mathcal{N}$. 

Fix $\noise \in \mathcal{N}$, and let $\state, \state' \in \X$ be given. Then $\| \policy(\state) - \policy(\state')\| \leq L_{\pi}\| \state - \state' \|$. Now we use the triangle inequality to separate the changes in $\state$ from those in $\control = \policy(\state)$:
\begin{align*}
\| f(\state, &\policy(\state), \noise) - f(\state', \policy(\state'), \noise) \| \\ &\leq \| f(\state, \policy(\state), \noise) - f(\state', \policy(\state), \noise) \|  + \| f(\state', \policy(\state), \noise) - f(\state', \policy(\state'), \noise) \|  \\
&\leq L_{f, \state}\| \state' -  \state \| +  L_{f, \control}\| \policy(\state) -  \policy(\state') \| \\
&\leq (L_{f, \state} + L_{f, \control} L_\pi) \| \state' -  \state \|,
\end{align*}
which shows that $L_{f, \state} + L_{f, \control} L_\pi$ is an upper bound for the Lipschitz constant of the function $\state \mapsto f(\state, \policy(\state), \noise)$. Hence, $L_V(L_{f, \state} + L_{f, \control} L_\pi)$ is an upper bound for the Lipschitz constant of the function $\state \mapsto V(f(\state, \policy(\state), \noise))$. 

Finally, note that the inequalities $L_{f, \state} \leq L_f$ and $L_{f, \control} \leq L_f$ imply that $L_V (L_{f, \state} + L_{f, \control} L_\pi) \leq L_V L_f (L_\pi + 1)$, so replacing $K=L_V L_f (L_\pi + 1)$ by $K = L_V (L_{f, \state} + L_{f, \control} L_\pi)$ is indeed an improvement. 

In practice, we have implemented this improvement in our method and all verifier variants for which we report results (including the baseline).

\subsection{Suggested Mesh Size}\label{app:suggestedmesh}

Recall from \cref{sec:Implementation} that refining the mesh size by a fixed factor $C \in (0,1)$ may be unnecessary to mitigate violations of the RASM conditions.
In this section, we explain how we compute a \emph{suggested mesh} for each of the points $\statedisc \in \Xdisc$ that violate the expected decrease condition.
This suggested mesh is used as an upper bound on the factor by which we refine the discretization.

Concretely, let $\statedisc \in \Xdisc$ be a soft violation (as defined in \cref{sec:Implementation}) and denote its current mesh size by $\tau_{\statedisc}$.
Thus, $\statedisc$ is associated with the set $\cell_\infty^{\tau_{\statedisc}}(\statedisc)$.
For each such point, the suggested mesh $\suggmesh_{\statedisc}$ is computed as
\begin{equation}
    \label{eq:suggested_mesh}
    \suggmesh_{\statedisc} =  \max\left\{0.8 \cdot\frac{V(\statedisc)-E(\statedisc)}{K},   \frac{\Vlb(\statedisc)-E(\statedisc)}{K}  \right\},
\end{equation}
where we write $E(\statedisc) = \log\Exp_{\noise \sim \noisedist}\left[\exp(V(f(\statedisc, \policy(\statedisc), \noise)))\right]$.
Then, our local refinement scheme splits the set $\cell_\infty^{\tau_{\statedisc}}(\statedisc)$ associated with the point $\statedisc$ into smaller cells with mesh size $\max(C \tau_{\statedisc}, \lambda_{\statedisc})$.
Hence, we refine the mesh size by the maximum of the fixed factor $C$ and the suggested mesh size $\suggmesh_{\statedisc}$.

\paragraph{Intuition for the suggested mesh size.} We now explain the intuition behind \cref{eq:suggested_mesh}. The requirement to check the cell containing $\statedisc$ from \cref{eq:expdecrcond_new} is that 
\[
E(\statedisc) < \Vlb(\statedisc) - \tau_{\statedisc} K.
\]
Hence, if we refine the mesh to $\lambda_{\statedisc}  < \frac{\Vlb(\statedisc)-E(\statedisc)}{K}$ (the second term in the maximum in \cref{eq:suggested_mesh}), then the condition will be guaranteed to hold for the new cell containing $\statedisc$. However, refining the point $\statedisc$ introduces new points in the discretization in the set $\cell_\infty^{\tau_{\statedisc}}(\statedisc)$, and for these points, this suggested mesh is not necessarily sufficient. This suggested mesh is somewhat conservative, however, since it does not take into account that also $\Vlb(\statedisc)$ would become larger when decreasing the mesh. The first term in the maximum in \cref{eq:suggested_mesh} takes this into account by replacing $\Vlb(\statedisc)$ by $V(\statedisc)$. We now multiply by $0.8$ to take into account that the required mesh might be lower for nearby points. 

\paragraph{Initial and safety conditions.} For the initial and safety conditions, computing a suggested mesh size is difficult due to the use of IBP for computing $\Vlb(\statedisc)$ and $\Vub(\statedisc)$.
In particular, the suggested mesh computation in \cref{eq:suggested_mesh} relies on Lipschitz constants, which are more conservative than the bounds we obtain using IBP.
As a result, adapting \cref{eq:suggested_mesh} for the initial and safety conditions would lead to suggested meshes that are too conservative (i.e., too low).
Thus, we simply use the fixed factor $C$ for refining violations of the initial and safety conditions. Since in practice the most difficult violations are violations of the expected decrease condition, this is not a major limitation.

\section{Proofs}\label{app:proofs}

\subsection{Preliminary Properties}

We first recall several standard properties of norms and Lipschitz constants. Throughout the appendix, superscripts on norms denote the index of the norm.

\cref{prop:lipcomp,prop:lipadd} are standard properties of Lipschitz constants.

\begin{property}\label{prop:lipcomp}
If $f \colon \R^{m_k} \rightarrow \R^{m_{\ell}},g \colon \R^{m_j} \rightarrow \R^{m_k}$ are Lipschitz continuous with Lipschitz constants $L_f$ and $L_g$ with respect to given norms  $\| \cdot \|^j$,  $\| \cdot \|^k$, and  $\| \cdot \|^\ell$ defined on $\R^{m_j}$, $\R^{m_k}$, and $\R^{m_\ell}$ respectively, then $f\circ g \colon \R^{m_j} \rightarrow \R^{m_{\ell}}$ is Lipschitz continuous with Lipschitz constant $L_f L_g$.
\end{property}

\begin{property}\label{prop:lipadd}
If $f, g \colon \R^{m_k} \rightarrow \R^{m_{\ell}}$ are Lipschitz continuous with Lipschitz constants $L_f$ and $L_g$ with respect to given norms  $\| \cdot \|^k$, and  $\| \cdot \|^\ell$ defined on $\R^{m_k}$ and $\R^{m_\ell}$ respectively, and $\alpha, \beta \in \R$, then $\alpha f + \beta g \colon \R^{m_k} \rightarrow \R^{m_{\ell}}$ is Lipschitz continuous with Lipschitz constant $|\alpha| L_f + |\beta| L_g$.
\end{property}

\Cref{prop:lipaffine} states that the operator norm of a matrix $A$ is a Lipschitz constant of a corresponding affine function $x \mapsto Ax + b$, where $b$ is some (bias) vector. 

\begin{property} \label{prop:lipaffine}
Let $A \in \R^{m_{\ell} \times m_k}$ be a matrix and let $b \in \R^{m_{\ell}}$ be a vector. 
Equip the input space with the norm $\| \cdot \|^k$ and the output space with the norm $\| \cdot \|^\ell$, and let the corresponding operator norm be given by 
\[\| A \|^{k,\ell} = \sup \Big\{\frac{\| A x \|^\ell}{\| x \|^k} ~\Big|~ x \in \R^{m_k},  x \neq 0\Big\}.\] For these norms, the affine function $\R^{m_k} \rightarrow \R^{m_{\ell}}$ defined by $x \mapsto Ax + b$ is Lipschitz continuous with Lipschitz constant $\|A\|^{k, \ell}$.
\end{property}
\begin{proof}
Let $x, x' \in \R^{m_k}$ be given. If $x = x'$, then $\| (Ax+b) - (Ax'+b)\|^\ell = 0 = \|A\|^{k, \ell}\|x-x'\|^k$ trivially. Now assume that $x \neq x'$, then $x-x'\neq 0$. Hence, 
\[\| (Ax+b) - (Ax'+b)\|^\ell = \|A(x-x')\|^{\ell} \leq \|A\|^{k, \ell}\|x-x'\|^k,\]
by the definition of the operator norm. This shows that $x \mapsto Ax + b$ is Lipschitz continuous with Lipschitz constant $\|A\|^{k, \ell}$.
\end{proof}

\Cref{prop:actfuns} shows that an activation function applying a scalar function componentwise inherits the Lipschitz constant of the scalar function in any weighted $1$-norm.

\begin{property} \label{prop:actfuns}
Let $R \colon \R \rightarrow \R$ be a function with Lipschitz constant $L$, i.e.\ $|R(x) - R(x')| \leq L|x-x'|$. Then the vectorized function $R' \colon \R^v \rightarrow \R^v$ applying $R$ componentwise has Lipschitz constant $L$ when the input and output space are both equipped with the same weighted $1$-norm.
\end{property}
\begin{proof}
Let $\|x\|_{\weightsys} =  \sum_{i=1}^v w_i \lvert x_i\rvert$ for weights $w_i > 0$ be the norm on the input and output space. For any $x, x' \in \R^v$, it holds that 
\begin{align*}
\|R'(x) - R'(x')\|_{\weightsys} &=  \sum_{i=1}^v w_i \lvert (R'(x) - R'(x'))_i\rvert =  \sum_{i=1}^v w_i \lvert R(x_i) - R(x'_i)\rvert\\ & \leq \sum_{i=1}^v w_i L\lvert x_i - x'_i\rvert = L  \sum_{i=1}^v w_i \lvert x_i - x'_i\rvert = L \|x - x'\|_{\weightsys},
\end{align*}
which shows that $R'$ is Lipschitz continuous with Lipschitz constant $L$. 
\end{proof}

Finally, \cref{prop:normdecomp} shows how we can decompose $1$-norms. 

\begin{property}\label{prop:normdecomp}
Equip $\R^{v_1+v_2}$, $\R^{v_1}$ and $\R^{v_2}$ with a weighted $1$-norm, such that the first $v_1$ weights on $\R^{v_1+v_2}$ coincide with the weights on $\R^{v_1}$ and the last $v_2$ weights on $\R^{v_1+v_2}$ coincide with the weights on $\R^{v_2}$.
Let $x = (y,z) \in \R^{v_1+v_2}$ with $y\in \R^{v_1}$ and $z \in \R^{v_2}$ be given. Then $\|x\|_{\weightsys} \leq \|y\|_{\weightsys} + \|z\|_{\weightsys}$. %
\end{property}

\begin{proof}
The triangle inequality implies
\[ 
\|x\|_{\weightsys} = \|(y, 0) + (0, z)\|_{\weightsys} \leq \|(y, 0)\|_{\weightsys} + \|(0, z)\|_{\weightsys} = \|y\|_{\weightsys} + \|z\|_{\weightsys},
\]
where the last equality holds since from the definition of the weighted $1$-norm it follows that 0 components do not contribute to the norm. 
\end{proof}

\setcounter{theorem}{0}
\setcounter{lemma}{0}

\subsection{Proof of \cref{thm:rasm}} \label{proof:rasm}

We give a short proof of \cref{thm:rasm}. Another proof is given in \cite{DBLP:conf/aaai/ZikelicLHC23}.

\begin{theorem}
    If there exists a RASM for the reach-avoid specification $\tuple{\xTarget, \xUnsafe, \rho}$, then this specification is satisfied.
\end{theorem}

\begin{proof} Fix a policy $\policy$ and an initial state $\state_0 \in \X_0$. We consider the stochastic process $(\state_t)_{t \in \Nzero}$ which is defined recursively by $\state_{t+1} = f(\state_t, \policy(\state_t), \noise_t)$. Let $\mathcal{F}_t$ be the natural filtration corresponding to $(\state_t)_{t \in \Nzero}$. Let \[\sigma = \min\{t\in\Nzero \colon \state_t \in \xTarget \vee V(\state_t) \geq \tfrac1{1-\rho}\}.\] Then $\sigma$ is a stopping time, since $\{\sigma = t\} \in \mathcal{F}_{t}$ for all $t \in \Nzero$; intuitively, this means that whether the event occurs is determined by the values of $\state_{t'}$ for $t' \leq t$.  

The expected decrease condition implies that $(V(\state_{\min\{t,\sigma\}}))_{t \in \Nzero}$ is a supermartingale with respect to the filtration $(\mathcal{F}_t)_{t \in \Nzero}$.  Namely, if $t < \sigma$, then \begin{align*} \Exp[V(\mathbf{x}_{\min\{t+1,\sigma\}}) \mid \mathcal{F}_t] &= \Exp[V(\mathbf{x}_{t+1}) \mid \mathcal{F}_t] = \Exp[V(f(\state_t, \policy(\state_t), \noise_t)) \mid \mathbf{x}_t] \\ &= \Exp_{\noise_t\sim\noisedist}[V(f(\state_t, \policy(\state_t), \noise_t))] \leq V(\state_t) = V(\mathbf{x}_{\min\{t,\sigma\}}), \end{align*}
while if $t \geq \sigma$ then $\Exp[V(\mathbf{x}_{\min\{t+1,\sigma\}}) \mid \mathcal{F}_t] = V(\mathbf{x}_{\sigma}) = V(\mathbf{x}_{\min\{t,\sigma\}})$.

We now show that $\sigma < \infty$ almost surely  (i.e., with probability 1). For this we use that $V$ decreases in expectation by $\epsilon$ each step until $\sigma$ occurs, but remains nonnegative. This implies that 
\[
0 \leq \Exp[V(\state_{\sigma})] = \Exp[V(\state_0) - \epsilon \sigma], 
\]
so $\Exp[\sigma] \leq \frac{\Exp[V(\state_0)]}{\epsilon}$. Hence, Markov's inequality implies $\mathbb{P}[\sigma \geq t] \leq \frac{\Exp[V(\state_0)]}{\epsilon t}$ and taking the limit $t \rightarrow \infty$ then shows that $\sigma < \infty$ almost surely.

Since $(V(\state_{\min\{t,\sigma\}}))_{t \in \Nzero}$ is a supermartingale and $\sigma < \infty$ almost surely, the optional stopping theorem implies that $\Exp[V(\state_\sigma)] \leq \Exp[V(\state_0)]$. Moreover, we have $\tfrac1{1-\rho} \mathbb{P}\left[V(\state_\sigma) \geq \tfrac1{1-\rho}\right] \leq \Exp[V(\state_\sigma)]$ by Markov's inequality (using that $V$ is nonnegative), and  $\Exp[V(\state_0)] \leq 1$ by the initial condition. Together, this shows that $\mathbb{P}\left[V(\state_\sigma) \geq \tfrac1{1-\rho}\right] \leq (1-\rho)\Exp[V(\state_\sigma)] \leq (1-\rho)\Exp[V(\state_0)] \leq 1-\rho$. 
Since $\state_\sigma\in \xTarget \vee V(\state_\sigma) \geq \frac1{1-\rho}$ holds, this implies $\mathbb{P}\left[\state_\sigma \in \xTarget\right] \geq \rho$. 
Since $V(\state_t) < \frac1{1-\rho}$ for $t < \sigma$ by the definition of $\sigma$, the safety condition guarantees that $\state_t \not\in \xUnsafe$ for $t < \sigma$.   Hence, we conclude that \[\satprob^\policy_{\state_0}(\xTarget, \xUnsafe) \geq \mathbb{P}\left[\state_\sigma \in \xTarget \wedge (\forall t < \sigma :\state_t \not\in \xUnsafe) \right] \geq \rho,\] as required to show that the reach-avoid specification $\tuple{\xTarget, \xUnsafe, \rho}$ is satisfied.
\end{proof}

\subsection{Proof of \cref{lem:drasm}} \label{proof:drasm}

\begin{lemma}
Every discrete RASM is also a RASM.
\end{lemma}

\begin{proof}
Let $V \colon \X \rightarrow \R_{\geq 0}$ be a discrete RASM. Then $V$ is Lipschitz continuous and hence continuous.  We proceed by showing that each of the three conditions in the definition of a discrete RASM (\cref{def:drasm}) implies the corresponding condition in the definition of a RASM (\cref{def:rasm}). \medskip

    \textbf{(1)} \emph{Initial condition}: Since $V$ is a discrete RASM, it holds that $\Vub(\statedisc) \leq 1$ for all $\statedisc \in \Xdisc$ such that $\cell_\infty^{\tau_{\statedisc}}(\statedisc) \cap \X_0 \neq \emptyset$. Now let $\state \in \X_0$ be given. Since $\Xdisc$ is a discretization of $\X$, there exists a point $\statedisc \in \Xdisc$ such that $\state \in \cell_\infty^{\tau_{\statedisc}}(\statedisc)$. Then $\state \in \cell_\infty^{\tau_{\statedisc}}(\statedisc) \cap \X_0$, so $\cell_\infty^{\tau_{\statedisc}}(\statedisc) \cap \X_0 \neq \emptyset$ and hence
    \[
    V(\state) \leq \Vmax(\statedisc) \leq \Vub(\statedisc) \leq 1.
    \]
    Since $\state \in \X_0$ was arbitrary, we conclude that $ V(\state) \leq 1$ for all $\state \in \X_0$. Hence, we conclude that the initial condition \textbf{(1)} from \cref{def:rasm} holds.

    \textbf{(2)} \emph{Safety condition}: Since $V$ is a discrete RASM, it holds that  $\Vlb(\statedisc) \geq \frac1{1-\rho}$ for all $\statedisc \in \Xdisc$ such that $\cell_\infty^{\tau_{\statedisc}}(\statedisc) \cap \xUnsafe \neq \emptyset$. Now let $\state \in \xUnsafe$ be given. Since $\Xdisc$ is a discretization of $\X$, there exists a point $\statedisc \in \Xdisc$ such that $\state \in \cell_\infty^{\tau_{\statedisc}}(\statedisc)$. Then $\state \in \cell_\infty^{\tau_{\statedisc}}(\statedisc) \cap \xUnsafe$, so $\cell_\infty^{\tau_{\statedisc}}(\statedisc) \cap \xUnsafe \neq \emptyset$ and hence
    \[
    V(\state) \geq \Vmin(\statedisc) \geq \Vlb(\statedisc) \geq \tfrac1{1-\rho}.
    \]
    Since $\state \in \xUnsafe$ was arbitrary, we conclude that $ V(\state) \geq \frac1{1-\rho}$ for all $\state \in \xUnsafe$. Hence, we conclude that the safety condition \textbf{(2)} from \cref{def:rasm} holds. \medskip

    \textbf{(3)} \emph{Expected decrease condition}: Since $V$ is a discrete RASM, it holds that 
    \begin{equation*}
        \Exp_{\noise \sim \noisedist}\left[V(f(\statedisc, \policy(\statedisc), \noise))\right] < \Vlb(\statedisc) - \tau_{\statedisc} K
    \end{equation*}for all $\statedisc \in \Xdisc$ such that $\cell_\infty^{\tau_{\statedisc}}(\statedisc) \cap (\X \setminus \xTarget) \neq \emptyset$ and  $\Vlb(\statedisc) < \frac1{1-\rho}$. For each point $\statedisc \in \Xdisc$ in the discretization, define $\epsilon_{\statedisc}$ by \[\epsilon_{\statedisc} = \Vlb(\statedisc) - \tau_{\statedisc} K - \Exp_{\noise \sim \noisedist}\left[V(f(\statedisc, \policy(\statedisc), \noise))\right] > 0.\]
    Since $\Xdisc$ is finite, $\epsilon := \min\limits_{\statedisc \in \Xdisc} \epsilon_{\statedisc} > 0$. Hence, there exists an $\epsilon > 0$ such that \smallskip
    \begin{equation*}
        \Exp_{\noise \sim \noisedist}\left[V(f(\statedisc, \policy(\statedisc), \noise))\right] \leq \Vlb(\statedisc) - \epsilon - \tau_{\statedisc} K
    \end{equation*}for all $\statedisc \in \Xdisc$ such that $\cell_\infty^{\tau_{\statedisc}}(\statedisc) \cap (\X \setminus \xTarget) \neq \emptyset$ and  $\Vlb(\statedisc) < \frac1{1-\rho}$.
    
    Now let $\state \in \X$ be a point such that $\state \in \X \setminus \xTarget$ and  $V(\state) < \frac1{1-\rho}$. Since $\Xdisc$ is a discretization of $\X$, there exists a point $\statedisc \in \Xdisc$ such that $\state \in \cell_\infty^{\tau_{\statedisc}}(\statedisc)$. Then $\state \in \cell_\infty^{\tau_{\statedisc}}(\statedisc) \cap (\X \setminus \xTarget)$, which implies $\cell_\infty^{\tau_{\statedisc}}(\statedisc) \cap (\X \setminus \xTarget) \neq \emptyset$, and $\Vlb(\statedisc) \leq \Vmin(\statedisc) \leq V(\state) < \frac1{1-\rho}$. Hence, we have \[\Exp_{\noise \sim \noisedist}\left[V(f(\statedisc, \policy(\statedisc), \noise))\right] \leq \Vlb(\statedisc) - \epsilon - \tau_{\statedisc} K.\]

      Fix an $\noise$. Since $\state \in \cell_\infty^{\tau_{\statedisc}}(\statedisc)$, we have $\|\state-\statedisc\|_{\infty} \leq \tau/d$, where $d$ is the dimension of the state space $\X$. It follows that $\|\state-\statedisc\|_{1} \leq \tau$. Hence, the Lipschitz continuity of $\policy$ implies $\| \policy(\state) - \policy(\statedisc) \|_1 \leq \tau_{\statedisc} L_\policy$, so 
  \[\|(\state, \policy(\state), \noise) - (\statedisc, \policy(\statedisc), \noise)\|_1 \leq  \| \state - \statedisc\|_1 + \|\policy(\state) - \policy(\statedisc)\|_1 \leq \tau_{\statedisc} (L_\policy+1),\]
  by \cref{prop:normdecomp}. 
  Since $V$ and $f$ have Lipschitz constants $L_V$ and $L_f$, this implies
  \[\left|V(f(\state, \policy(\state), \noise)) - V(f(\statedisc, \policy(\statedisc), \noise))\right| \leq  \tau_{\statedisc} L_V L_f (L_\policy+1) = \tau_{\statedisc} K.\]
  Hence, we have $V(f(\state, \policy(\state), \noise)) \leq V(f(\statedisc, \policy(\statedisc), \noise)) + \tau_{\statedisc} K$ for each $\noise$. Taking the expectation over $\noise \sim \noisedist$ yields 
  \begin{align*}
    \Exp_{\noise \sim \noisedist}\left[V(f(\state, \policy(\state), \noise))\right] &\leq
    \Exp_{\noise \sim \noisedist}\left[V(f(\statedisc, \policy(\statedisc), \noise))\right] + \tau_{\statedisc} K \\ &\leq \Vlb(\statedisc) - \epsilon \leq V(\state) - \epsilon.
  \end{align*}
  Since $\state \in \X$ was an arbitrary point such that $\state \in \X \setminus \xTarget$ and  $V(\state) < \frac1{1-\rho}$, we conclude that there exists an $\epsilon > 0$ such that  $\Exp_{\noise \sim \noisedist}\left[V(f(\state, \policy(\state), \noise))\right] \leq V(\state) - \epsilon$ for all $\state \in \X$ such that $\state \in \X \setminus \xTarget$ and  $V(\state) < \frac1{1-\rho}$. Hence, the Expected decrease condition  \textbf{(3)} from \cref{def:rasm} holds.

  Hence, we conclude that $V$ is a RASM, which completes the proof.
\end{proof}

\subsection{Proof of \cref{lem:lograsm}} \label{proof:lograsm}

\newcommand{\Vlog}{V}
\newcommand{\Vnolog}{\widetilde{V}}

\begin{lemma}
If $V$ is a logRASM, then $ \exp\!\big(V\big)$ is a RASM.
\end{lemma}

\begin{proof}
Let $\Vlog$ be a logRASM, and write $\Vnolog = \exp(\Vlog)$. For the initial condition, note that $\Vlog(\state) \leq 0$ for all $\state \in \X_0$ implies $\Vnolog(\state) = \exp\big( \Vlog(\state)\big) \leq 1$ for all $\state \in \X_0$. 

Similarly, for the safety condition, $\Vlog(\state) \geq \log\big(\frac1{1-\rho}\big)$ for all $\state \in \xUnsafe$ implies $\Vnolog(\state) = \exp\big( \Vlog(\state)\big)  \geq \frac1{1-\rho}$ for all $\state \in \xUnsafe$. 

For the expected decrease condition, we first note that $\Vlog$ is continuous and $\X$ is compact, so Weierstrass’ Extreme Value Theorem implies that $\Vlog$ attains some global minimum $v$. Now let $\epsilon > 0$ satisfy 
\[
\log \Exp_{\noise \sim \noisedist}\left[\exp(\Vlog(f(\state, \policy(\state), \noise)))\right] \leq \Vlog(\state) - \epsilon
\]
for all $\state \in \X \setminus \xTarget$ with  $\Vlog(\state) \leq \log\big(\frac1{1-\rho}\big)$. Let 
$\epsilon' = e^v(1-e^{-\epsilon}) > 0$. Then 
\begin{align*}
\Exp_{\noise \sim \noisedist}\left[\Vnolog(f(\state, \policy(\state), \noise))\right] &\leq \exp\big(\Vlog(\state) - \epsilon\big) = \Vnolog(\state) e^{-\epsilon} \\ &= \Vnolog(\state) - \Vnolog(\state)(1-e^{-\epsilon}) \\ &\leq \Vnolog(\state) - e^v(1-e^{-\epsilon}) \\ &= \Vnolog(\state) - \epsilon'
\end{align*}
for all $\state \in \X \setminus \xTarget$ with  $\Vnolog(\state) \leq \frac1{1-\rho}$. Hence, the expected decrease condition is also satisfied. Finally, we note that $\Vnolog = \exp\big(\Vlog\big)$ is nonnegative and that $\Vnolog$ is continuous since $\Vlog$ and $\exp$ are continuous. We conclude that $\Vnolog$ is a RASM.
\end{proof}

\subsection{Proof of \cref{thm:logdrasm}} \label{proof:logdrasm}

\begin{theorem}
If $V$ is a discrete logRASM for a discretization $\Xdisc$, then $ \exp\big(V\big)$ is a RASM.
\end{theorem}

\begin{proof}
We follow the proof of \cref{lem:drasm}, except that we need some novel ideas for the expected decrease condition. 
Let $\Vlog \colon \X \rightarrow \R$ be a discrete logRASM, and write $\Vnolog = \exp(\Vlog)$. Then $\Vlog$ is Lipschitz continuous and hence continuous, so $\Vnolog$ is also continuous. Also, $\Vnolog$ is nonnegative.  We proceed by showing that each of the three conditions in the definition of a discrete logRASM (\cref{def:dlograsm}) implies the corresponding condition in the definition of a RASM (\cref{def:rasm}). 

    \textbf{(1)} \emph{Initial condition}: Since $\Vnolog$ is a discrete logRASM, it holds that $\Vub(\statedisc) \leq 0$ for all $\statedisc \in \Xdisc$ such that $\cell_\infty^{\tau_{\statedisc}}(\statedisc) \cap \X_0 \neq \emptyset$. Now let $\state \in \X_0$ be given. Since $\Xdisc$ is a discretization of $\X$, there exists a point $\statedisc \in \Xdisc$ such that $\state \in \cell_\infty^{\tau_{\statedisc}}(\statedisc)$. Then $\state \in \cell_\infty^{\tau_{\statedisc}}(\statedisc) \cap \X_0$, so $\cell_\infty^{\tau_{\statedisc}}(\statedisc) \cap \X_0 \neq \emptyset$ and hence
    \[
    \Vnolog(\state) = \exp(V(\state)) \leq \exp(\Vmax(\statedisc)) \leq \exp(\Vub(\statedisc)) \leq \exp(0) = 1.
    \]
    Since $\state \in \X_0$ was arbitrary, we conclude that $\Vnolog(\state) \leq 1$ for all $\state \in \X_0$. Hence, we conclude that the initial condition \textbf{(1)} from \cref{def:rasm} holds.

    \textbf{(2)} \emph{Safety condition}: Since $V$ is a discrete logRASM, it holds that  $\Vlb(\statedisc) \geq \log\big(\frac1{1-\rho}\big)$ for all $\statedisc \in \Xdisc$ such that $\cell_\infty^{\tau_{\statedisc}}(\statedisc) \cap \xUnsafe \neq \emptyset$. Now let $\state \in \xUnsafe$ be given. Since $\Xdisc$ is a discretization of $\X$, there exists a point $\statedisc \in \Xdisc$ such that $\state \in \cell_\infty^{\tau_{\statedisc}}(\statedisc)$. Then $\state \in \cell_\infty^{\tau_{\statedisc}}(\statedisc) \cap \xUnsafe$, so $\cell_\infty^{\tau_{\statedisc}}(\statedisc) \cap \xUnsafe \neq \emptyset$ and hence
    \[
    \Vnolog(\state) = \exp(V(\state)) \geq \exp(\Vmin(\statedisc)) \geq \exp(\Vlb(\statedisc)) \geq \exp\!\big(\!\log\!\big(\tfrac1{1-\rho}\big)\big) = \tfrac1{1-\rho}.
    \]
    Since $\state \in \xUnsafe$ was arbitrary, we conclude that $\Vnolog(\state) \geq \frac1{1-\rho}$ for all $\state \in \xUnsafe$. Hence, we conclude that the safety condition \textbf{(2)} from \cref{def:rasm} holds. \medskip

    \textbf{(3)} \emph{Expected decrease condition}: Since $V$ is a discrete logRASM, it holds that 
    \begin{equation*}
        \log \Exp_{\noise \sim \noisedist}\Big[\exp\!\big(V(f(\statedisc, \policy(\statedisc), \noise))\big)\Big] < \Vlb(\statedisc) - \tau_{\statedisc} K
    \end{equation*}
    for all $\statedisc \in \Xdisc$ such that $\cell_\infty^{\tau_{\statedisc}}(\statedisc) \cap (\X \setminus \xTarget) \neq \emptyset$ and  $\Vlb(\statedisc) < \log\big(\frac1{1-\rho}\big)$. For each point $\statedisc \in \Xdisc$ in the discretization, define $\epsilon_{\statedisc}$ by \[\epsilon_{\statedisc} = \Vlb(\statedisc) - \tau_{\statedisc} K -  \log \Exp_{\noise \sim \noisedist}\Big[\exp\!\big(V(f(\statedisc, \policy(\statedisc), \noise))\big)\Big] > 0.\]
    Since $\Xdisc$ is finite, $\epsilon' := \min\limits_{\statedisc \in \Xdisc} \epsilon_{\statedisc} > 0$. Hence, there exists an $\epsilon' > 0$ such that \smallskip
    \begin{equation*}
         \log \Exp_{\noise \sim \noisedist}\Big[\exp\!\big(V(f(\statedisc, \policy(\statedisc), \noise))\big)\Big] \leq \Vlb(\statedisc) - \epsilon - \tau_{\statedisc} K
    \end{equation*}for all $\statedisc \in \Xdisc$ such that $\cell_\infty^{\tau_{\statedisc}}(\statedisc) \cap (\X \setminus \xTarget) \neq \emptyset$ and  $\Vlb(\statedisc) < \log\big(\frac1{1-\rho}\big)$.
    
    Now let $\state \in \X$ be a point such that $\state \in \X \setminus \xTarget$ and  $\Vnolog(\state) < \frac1{1-\rho}$. Since $\Xdisc$ is a discretization of $\X$, there exists a point $\statedisc \in \Xdisc$ such that $\state \in \cell_\infty^{\tau_{\statedisc}}(\statedisc)$. Then $\state \in \cell_\infty^{\tau_{\statedisc}}(\statedisc) \cap (\X \setminus \xTarget)$, which implies $\cell_\infty^{\tau_{\statedisc}}(\statedisc) \cap (\X \setminus \xTarget) \neq \emptyset$, and $\Vlb(\statedisc) \leq \Vmin(\statedisc) \leq V(\state) = \log(\Vnolog(\state)) < \log\big(\frac1{1-\rho}\big)$. Hence, we have \[\log \Exp_{\noise \sim \noisedist}\Big[\exp\!\big(V(f(\statedisc, \policy(\statedisc), \noise))\big)\Big] \leq \Vlb(\statedisc) - \epsilon' - \tau_{\statedisc} K.\]
Fix an $\noise$. Since $\state \in \cell_\infty^{\tau_{\statedisc}}(\statedisc)$, we have $\|\state-\statedisc\|_{\infty} \leq \tau/d$, where $d$ is the dimension of the state space $\X$. It follows that $\|\state-\statedisc\|_{1} \leq \tau$. Hence, the Lipschitz continuity of $\policy$ implies $\| \policy(\state) - \policy(\statedisc) \|_1 \leq \tau_{\statedisc} L_\policy$, so 
  \[\|(\state, \policy(\state), \noise) - (\statedisc, \policy(\statedisc), \noise)\|_1 \leq  \| \state - \statedisc\|_1 + \|\policy(\state) - \policy(\statedisc)\|_1 \leq \tau_{\statedisc} (L_\policy+1),\]
  by \cref{prop:normdecomp}. 
  Since $V$ and $f$ have Lipschitz constants $L_V$ and $L_f$, this implies
  \[\left|V(f(\state, \policy(\state), \noise)) - V(f(\statedisc, \policy(\statedisc), \noise))\right| \leq  \tau_{\statedisc} L_V L_f (L_\policy+1) = \tau_{\statedisc} K.\]
  Hence, we have $V(f(\state, \policy(\state), \noise)) \leq V(f(\statedisc, \policy(\statedisc), \noise)) + \tau_{\statedisc} K$ for each $\noise$. 
 
\noindent We now come to the main novel part of the proof. We have
   \begin{align*} \Exp_{\noise \sim \noisedist} \left[\Vnolog(f(\state, \policy(\state), \noise))\right] 
     &= \Exp_{\noise \sim \noisedist} \Big[\exp\!\big(\Vlog(f(\state, \policy(\state), \noise))\big)\Big] 
  \\ &\leq \Exp_{\noise \sim \noisedist} \Big[\exp\big(\Vlog(f(\statedisc, \policy(\statedisc), \noise))+\tau_{\statedisc} K\big)\Big] 
  \\ &= e^{\tau_{\statedisc} K}\Exp_{\noise \sim \noisedist} \Big[\exp\big(\Vlog(f(\statedisc, \policy(\statedisc), \noise))\big)\Big] 
  \\ &\leq e^{\tau_{\statedisc} K} \exp\left(\Vlb(\statedisc) - \epsilon' - \tau_{\statedisc} K\right)
  \\ &\leq e^{\tau_{\statedisc} K} e^{ - \tau_{\statedisc} K} e^{-\epsilon'} \exp\left(V(\statedisc)\right) = e^{-\epsilon'} \Vnolog(\state). 
  \end{align*}
  Note that $\Vlog$ is continuous and $\X$ is compact, so Weierstrass’ Extreme Value Theorem implies that $\Vlog$ attains some global minimum $v$. Let 
$\epsilon = e^v(1-e^{-\epsilon'}) > 0$. Then this shows that
\[
\Exp_{\noise \sim \noisedist} \left[\Vnolog(f(\state, \policy(\state), \noise))\right] \leq \Vnolog(\state) - \epsilon.
\]
   Since $\state \in \X$ was an arbitrary point such that $\state \in \X \setminus \xTarget$ and  $\Vnolog(\state) < \frac1{1-\rho}$, we conclude that there exists an $\epsilon > 0$ such that  $\Exp_{\noise \sim \noisedist}\left[\Vnolog(f(\state, \policy(\state), \noise))\right] \leq \Vnolog(\state) - \epsilon$ for all $\state \in \X$ such that $\state \in \X \setminus \xTarget$ and  $\Vnolog(\state) < \frac1{1-\rho}$. Hence, the Expected decrease condition  \textbf{(3)} from \cref{def:rasm} holds.

  Hence, we conclude that $\Vnolog = \exp(\Vlog)$ is a RASM, which completes the proof.
\end{proof}

\subsection{Proof of \cref{lem:better}} \label{proof:better}

\begin{lemma}
Let $K' = \frac1{1-\rho} K > 0$. If $\Vlb(\statedisc) < \log\big(\frac1{1-\rho}\big)$, then \[\exp(\Vlb(\statedisc)) - \tau_{\statedisc} K' < \exp(\Vlb(\statedisc) - \tau_{\statedisc} K).\]
\end{lemma}

\begin{proof}
We have
\begin{align*}
\exp(\Vlb(\statedisc) - \tau_{\statedisc} K) 
    &= \exp(\Vlb(\statedisc))\exp( - \tau_{\statedisc} K) 
 \\ &\geq \exp(\Vlb(\statedisc)) (1 - \tau_{\statedisc} K) 
 \\ &= \exp(\Vlb(\statedisc)) - \exp(\Vlb(\statedisc))\tau_{\statedisc} K 
 \\ &> \exp(\Vlb(\statedisc)) - \tfrac1{1-\rho}\tau_{\statedisc} K 
 \\ &= \exp(\Vlb(\statedisc)) - \tau_{\statedisc} K',
\end{align*}
which is the required inequality.
\end{proof}

\subsection{Proof of \cref{lem:matrixnorm}}\label{proof:matrixnorm}

\begin{lemma}
Let $M \in \R^{m_\ell \times m_k}$ be a matrix with entries $M_{ij}$. Equip the space $\R^{m_k}$ with the norm $\|x\|_{\weightsys}^k = \sum_{i=1}^{m_k} \weight_i^k \lvert x_i\rvert $, and the space $\R^{m_\ell}$ with the norm $\|x\|_{\weightsys}^{\ell} = \sum_{i=1}^{m_\ell} \weight_i^\ell \lvert x_i\rvert $. Then the corresponding matrix norm satisfies 
\[
\|M\|^{k, \ell}_{\weightsys} \;=\; \max\limits_{1 \leq j \leq m_k}  \left[\frac{1}{\weight_j^k} \sum_{i=1}^{m_\ell} \weight_i^\ell  \left\lvert M_{ij} \right\rvert \right].\]
\end{lemma}

\begin{proof}
Note that we can write $\|x\|_{\weightsys}^{k} = \sum_{i=1}^{m_k} \weight_i^k \lvert x_i\rvert $ and $\|y\|_{\weightsys}^{\ell} = \sum_{i=1}^{m_\ell} \weight_i^{\ell} \lvert y_i\rvert $. Hence,
\begin{align*}
    \| A x \|_{\weightsys}^{\ell} &= \sum_{i=1}^{m_\ell} \weight_i^{\ell} \lvert(Ax)_i\rvert \leq  \sum_{i=1}^{m_\ell} \left[ \weight_i^{\ell} \sum\limits_{j=1}^{m_k} \lvert A_{ij}\rvert \lvert x_j\rvert \right] \\ &=  \sum\limits_{j=1}^{m_k}  \left(\frac{1}{\weight_j^k} \sum_{i=1}^{m_\ell} \weight_i^\ell \left\lvert A_{ij}\right\rvert \right) \weight_j^k\lvert x_j \rvert  \\ &\leq  \left(\max_{1 \leq j \leq m_k} \left[\frac{1}{\weight_j^k} \sum_{i=1}^{m_\ell} \weight_i^\ell \left\lvert A_{ij}\right\rvert \right]\right) \! \sum\limits_{j=1}^{m_k}  w_j^k \lvert x_j \rvert \\ &= \left(\max_{1 \leq j \leq m_k} \left[\frac{1}{\weight_j^k} \sum_{i=1}^{m_\ell} \weight_i^\ell \left\lvert A_{ij}\right\rvert \right]\right) \|x\|_{\weightsys}^{k}, 
\end{align*}
which shows that \[\|A\|_{\weightsys}^{k, \ell} = \sup \Big\{\frac{\| A x \|_{\weightsys}^\ell}{\| x \|_{\weightsys}^k} ~\Big|~ x \in \R^{m_k},  x \neq 0\Big\} \leq \max\limits_{1 \leq j \leq m_k} \left[\frac{1}{\weight_j^k} \sum_{i=1}^{m_\ell} \weight_i^\ell \left\lvert A_{ij}\right\rvert \right].\] 

Moreover, the inequality $\|Ax\|_{\weightsys}^\ell \leq \max\limits_{1 \leq j \leq m_k} \left[\frac{1}{\weight_j^k} \sum_{i=1}^{m_\ell} \weight_i^\ell \left\lvert A_{ij}\right\rvert \right] \|x\|_{\weightsys}^k$ holds with equality, if we choose a $j^* \in \argmax\limits_{1 \leq j \leq m_k} \left[\frac{1}{\weight_j^k} \sum_{i=1}^{m_\ell} \weight_i^\ell \left\lvert A_{ij}\right\rvert \right]$ and define $x$ by $x_{j^*} = 1$ and $x_j = 0$ for $j \neq j^*$. This shows that the operator norm $\|A\|_{\weightsys}^{k, \ell}$ is in fact equal to $\max\limits_{1 \leq j \leq m_k} \left[\frac{1}{\weight_j^k} \sum_{i=1}^{m_\ell} \weight_i^\ell \left\lvert A_{ij}\right\rvert \right]$.
\end{proof}

\subsection{Proof of \cref{lem:lipprod}} \label{proof:lipprod}

\begin{lemma}
Let $\weightsys$ be a weight system. Then $\lipnetwork$ is a Lipschitz constant of $\network$, i.e.\ $\|\network(x) - \network(x')\|^n_{\weightsys} \leq  \lipnetwork \|x - x'\|^0_{\weightsys}$ for all $x, x' \in \R^{m_0}$. If additionally $w^n_i=1$ for all $1 \leq i \leq m_n$,  then $\lipnetwork$ is a Lipschitz constant of $\network$ for the standard (unweighted) 1-norm, i.e.\ $\|\network(x) - \network(x')\| \leq \lipnetwork \|x - x'\|$ for all $x, x' \in \R^{m_0}$.
\end{lemma}
\begin{proof}
We first prove the inequality  $\|\network(x) - \network(x')\|^n_{\weightsys} \leq L_{T, \weightsys} \|x - x'\|^0_{\weightsys}$ for all $x, x' \in \R^{m_0}$ by induction on the number of layers $(n+1)$, where $n \geq 1$. 

For $n=1$, we note that 
\begin{align*} \|T(x) - T(x')\|^1_{\weightsys} &= \|R_1(A_1x+b_1) - R_1(A_1x'+b_1)\|^1_{\weightsys} \\ &\leq \|(A_1x+b_1) - (A_1x'+b_1)\|^1_{\weightsys} \\ &\leq \|A_1\|^{0,1}_{\weightsys} \|x - x'\|^0_{\weightsys} = L_{T, \weightsys} \|x - x'\|^0_{\weightsys},
\end{align*}
using \cref{prop:lipaffine,prop:actfuns}, and the fact that $R_1$ has Lipschitz constant 1.

\noindent For the induction step, consider a network with $(n+1)+1$ layers. Write \[\network(x_0) = R_{n+1}(A_{n+1} x_n + b_{n+1}) = R_{n+1}(A_{n+1} \networkleqn(x_0) + b_{n+1})\] where $\networkleqn$ represents the operator corresponding to the first $n+1$ layers of the network. Then
\begin{align*} \|\network(x) &- \network(x')\|^{n+1}_{\weightsys}  \\ &= \|R_{n+1}(A_{n+1} \networkleqn(x) + b_{n+1}) - R_{n+1}(A_{n+1} \networkleqn(x') + b_{n+1})\|^{n+1}_{\weightsys} \\ &\leq \|(A_{n+1} \networkleqn(x) + b_{n+1}) - (A_{n+1} \networkleqn(x') + b_{n+1})\|^{n+1}_{\weightsys} \\ &\leq \|A_{n+1}\|^{n,n+1}_{\weightsys} \|\networkleqn(x) - \networkleqn(x')\|^n_{\weightsys} \\ &\leq \|A_{n+1}\|^{n,n+1}_{\weightsys} L_{T_{\leq n}, \weightsys}  \|x - x'\|^0_{\weightsys}  \\ &=  \|A_{n+1}\|^{n,n+1}_{\weightsys} \prod_{\ell=1}^n \|A_\ell\|^{\ell-1, \ell}_{\weightsys}   \|x - x'\|^0_{\weightsys}  \\[-4pt] &= \prod_{\ell=1}^{n+1} \|A_\ell\|^{\ell-1, \ell}_{\weightsys}  \|x - x'\|^0_{\weightsys} = L_{T, \weightsys} \|x - x'\|^0_{\weightsys},
\end{align*}
which completes the induction.

We turn to the second part of the lemma. Suppose that $w_i^n = 1$ for $1 \leq i \leq m_n$. By definition of the weight system $\weightsys$, we have $\max_i w_i^0  = 1$, so $w_i^0 \leq 1$ for $1 \leq i \leq m_0$. Hence, if $\| \cdot \|$ denotes the unweighted 1-norm, then $\|x\| = \| x \|_{\weightsys}^n$ for $x \in \R^{m_n}$ and \[\|x\| = \sum_{i=1}^{m_0} |x_i| \geq  \sum_{i=1}^{m_0} w_i^0 |x_i| = \| x \|_{\weightsys}^0\] for $x \in \R^{m_0}$. Hence, we conclude that \[\|T(x) - T(x')\| = \|T(x) - T(x')\|^{n+1}_{\weightsys} \leq L_{T, \weightsys} \|x - x'\|^0_{\weightsys} \leq L_{T, \weightsys} \|x - x'\|,\] which completes the proof.
\end{proof}

\subsection{Proof of \cref{lem:optlem}} \label{proof:optlem}

\begin{lemma}
If $\weightsys$ is optimal for output weights $\weight^n$, then $\lipnetwork \leq \lipnetworkalt$ for all weight systems $\weightsysalt$ with output weights $\weight^n$.
\end{lemma}
\begin{proof}
Let $\weight^n$ be given output weights and let $\weightsys$ be optimal for output weights $\weight^n$. Let $\weightsysalt$ be any weight system with output weights $\weight^n$. Then we have $\lipnetwork\weight^0_j \leq \lipnetworkalt\weightalt^0_j $ for all $1 \leq j \leq m_0$. By definition of a weight system, we have $\max_j \weightalt^0_j = \max_j w^0_j = 1$. Take a $j^*$ such that $w^0_{j^*} = 1$. Then $\weightalt^0_{j^*} \leq 1$, so \[\lipnetworkalt\geq\lipnetworkalt\weightalt^0_{j^*} \geq \lipnetwork \weight^0_{j^*} = \lipnetwork,\] which is the required inequality.
\end{proof}

\subsection{Proof of \cref{thm:optweights}} \label{proof:optweights}

\begin{theorem}[Correctness of Algorithm \ref{alg:weights}]
Let output weights $\weight^n$ be given. 
Then the weights $\weight^\ell_j$ computed using  Algorithm \ref{alg:weights} are optimal for output weights~$\weight^n$.
\end{theorem}

\begin{proof}
We first note that the weights that  Algorithm \ref{alg:weights} computes indeed satisfy $\max_j w^\ell_j = 1$ for all $0 \leq \ell \leq n-1$. Let $j^* \in \argmax\limits_{1 \leq j \leq m_{\ell-1}} \sum_{i=1}^{m_\ell} \weight_i^\ell  \left\lvert (A_\ell)_{ij} \right\rvert$. Then \[K_\ell =\sum_{i=1}^{m_\ell} \weight_i^\ell  \left\lvert (A_\ell)_{ij^*} \right\rvert \geq  \sum_{i=1}^{m_\ell} \weight_i^\ell  \left\lvert (A_\ell)_{ij} \right\rvert,\]
so $\weight_j^{\ell-1} = \frac1{K_\ell} \sum_{i=1}^{m_\ell}  \weight_i^\ell  \left\lvert (A_\ell)_{ij} \right\rvert \leq 1$, while $\weight_{j^*}^{\ell-1} = 1$. Hence, $\max_j w^{\ell-1}_j = 1$ for all $1 \leq \ell \leq n$, so $\max_j w^\ell_j = 1$ for all $0 \leq \ell \leq n-1$.

We now prove the optimality by induction on the number of layers $(n+1)$, where $n \geq 1$. 
For $n=1$, the weight system $\weightsys$ with corresponding Lipschitz bound $K$ computed using \cref{alg:weights} satisfies $K w_j^0 = \sum_{i=1}^{m_0} \weight_i^1 \left\lvert (A_1)_{ij} \right\rvert $. Now let $\wsysalt$ be any weight system with the same output weights $\weight^n$, and let $\Kalt$ be the corresponding Lipschitz bound. Then we have \[\Kalt \weightalt_j^0 = \| A_1 \|^{0,1}_{\wsysalt} \weightalt_j^0 \geq \left(\frac1{\weightalt_j^0} \sum_{i=1}^{m_0} \weight_i^1 \left\lvert (A_1)_{ij} \right\rvert \right) \weightalt_j^0 = \sum\limits_{i=1}^{m_0}\weight_i^1 \left\lvert (A_1)_{ij} \right\rvert = K w_j^0\] by \cref{lem:matrixnorm}, showing the optimality of the weights $w$.

For the induction step, assume that we have a network with $(n+1)+1$ layers. Let $\weightsys$ denote the weight system computed using \cref{alg:weights}, and let $K$ be the corresponding Lipschitz bound. Let $K_{\geq2}$ denote the Lipschitz bound computed after all but one iteration of the outer loop, i.e.\ $K_{\geq2} = \prod_{\ell=2}^{n+1} \| A_\ell \|^{\ell-1, \ell}_{\weightsys}$. Let $K_1$ be the Lipschitz factor in the last iteration of the loop. Then $K = K_1 K_{\geq 2}$. Similarly, let $\weightsysalt$ be any other weight system with output weights $\weight^n$. Let $\Kalt_{\geq 2}$ denote the Lipschitz bound computed after all but one iteration of the outer loop, i.e. $\Kalt_{\geq2} = \prod_{\ell=2}^{n+1} \| A_\ell \|^{\ell-1, \ell}_{\wsysalt}$. Finally, let $\Kalt_1 = \| A_1 \|^{0,1}_{\wsysalt}$ be the Lipschitz factor in the last iteration of the loop. Then $\Kalt = \Kalt_1' \Kalt_{\geq 2}$. 

From \cref{alg:weights}, we note that $K_1 \weight_j^0 = \sum_{i=1}^{m_0} \weight_i^1 \left\lvert (A_1)_{ij} \right\rvert $, while \cref{lem:matrixnorm} implies that \[\Kalt_1 \weightalt_j^0 = \| A_1 \|^{0,1}_{\wsysalt} \weightalt_j^0 \geq \left(\frac1{\weightalt_j^0} \sum_{i=1}^{m_0} \weight^1_i \left\lvert (A_1)_{ij} \right\rvert \right) \weightalt_j^0 = \sum_{i=1}^{m_0} \weightalt^1_i \left\lvert (A_1)_{ij} \right\rvert.\]
By the induction hypothesis, we have $\Kalt_{\geq 2}\weightalt^1_j \geq K_{\geq 2} \weight^1_j$ for all $1 \leq j \leq m_1$. Combining yields
 \begin{align*}\Kalt\weightalt_j^0 &= \Kalt_1 \Kalt_{\geq 2} \weightalt_j^0 \geq \sum_{i=1}^{m_0} \Kalt_{\geq 2}\weightalt^1_i \left\lvert (A_1)_{ij} \right\rvert \geq \sum_{i=1}^{m_0} K_{\geq 2}\weight^1_i \left\lvert (A_1)_{ij} \right\rvert \\ &= K_1 K_{\geq 2} w_j^0 = K w_j^0,\end{align*}
which is the required inequality. 

This completes the induction and hence the proof.
\end{proof}

\subsection{Proof of \cref{thm:averagedactivation}}\label{proof:averagedactivation}

\begin{theorem}
Consider an $(n+1)$-layer neural network with $\tfrac12$-averaged activation operators $R_k$.  Let  $\weightsys$ be a corresponding weight system. Let \[S_n = \{(k_1, k_2, \ldots, k_r) \in \mathbb{N}_0^r \mid 0 \leq r \leq n-1, \, 1 \leq k_1 < k_2 < \dots < k_r \leq n-1\}.\] Then the Lipschitz constant $L_{\network}$ of the neural network operator $\network$ satisfies
\[
L_{\network} \leq \frac1{2^{n-1}} \sum_{(k_1, k_2, \ldots, k_r) \in S_n}\;\left[\prod_{\ell=1}^{r+1} \| A_{k_\ell} \ldots A_{k_{\ell-1}+1} \|^{k_{\ell-1}, k_\ell}_{\weightsys} \right],
\]
where we set $k_0 = 0$ and $k_{r+1} = n$.
\end{theorem}

\begin{proof}
We prove the statement by induction on the number of layers $(n+1)$, where $n \geq 1$. For the proof, we first omit the outermost activation operator. Note that the Lipschitz constant of each activation operator $R_k$ is 1 since it is a $\tfrac12$-averaged activation operator and using \cref{prop:actfuns}. In particular, \cref{prop:lipcomp} shows that a Lipschitz constant computed for the network without the outermost activation operator is also valid for the network where we do have the outermost activation operator.

For $n=1$, we have $\network(x_0) = A_1 x_0 + b_1$. Then $L_{\network} \leq \|A_1\|_{\weightsys}^{0,n}$ by \cref{prop:lipaffine}.

For the induction step, assume that we have a network with $(n+1)+1$ layers. Write \[\network(x_0) = A_{n+1} x_n + b_{n+1} = A_{n+1} R_n(\networkleqn(x_0)) + b_{n+1}\] where $\networkleqn$ represents the operator corresponding to the first $n+1$ layers of the network. Since $R_n$ is $\tfrac12$-averaged, we can write $R_n(x) = \tfrac12 x + \tfrac12 Q(x)$, where $Q$ has Lipschitz constant 1. Hence,
\begin{align*} \network(x_0) &= \tfrac12\left(A_{n+1} (A_n x_{n-1} + b_n) + b_{n+1}\right) + \tfrac12\left(A_{n+1} Q(\networkleqn(x_0)) + b_{n+1}\right) \\ &= \tfrac12\left( (A_{n+1} A_n) x_{n-1} + (A_{n+1} b_n + b_{n+1})\right) + \tfrac12\left(A_{n+1} Q(\networkleqn(x_0)) + b_{n+1}\right)
\\ &= \tfrac12 \networkleqnprime(x_0) + \tfrac12\left(A_{n+1} Q(\networkleqn(x_0)) + b_{n+1}\right),\end{align*}
where $\networkleqnprime$ represents the operator corresponding to a network with $n+1$ layers, where the first $n$ layers are as in the given network, but where the last layer is as the output layer of the given network, the matrix is $A_{n+1}A_n$ and the bias is $A_{n+1} b_n + b_{n+1}$. Define $A'_k = A_k$ for $k < n$ and $A'_n = A_{n+1}A_n$. The induction hypothesis informs us that 
\[
L_{\networkleqnprime} \leq \frac1{2^{n-1}} \sum_{(k_1, k_2, \ldots, k_r) \in S_n}\;\left[\prod_{\ell=1}^{r+1} \| A'_{k_\ell} \ldots A'_{k_{\ell-1}+1} \|_{\weightsys}^{k_{\ell-1}, k_\ell} \right],
\]
where in this case $k_{r+1}= n$. If we instead write $k_{r+1} = n+1$, we can also write 
\[
L_{\networkleqnprime} \leq \frac1{2^{n-1}} \sum_{(k_1, k_2, \ldots, k_r) \in S_n}\;\left[\prod_{\ell=1}^{r+1} \| A_{k_\ell} \ldots A_{k_{\ell-1}+1} \|_{\weightsys}^{k_{\ell-1}, k_\ell} \right].
\]
The induction hypothesis also informs us that 
\[
L_{\networkleqn} \leq \frac1{2^{n-1}} \sum_{(k_1, k_2, \ldots, k_r) \in S_n}\;\left[\prod_{\ell=1}^{r+1} \| A_{k_\ell} \ldots A_{k_{\ell-1}+1} \|_{\weightsys}^{k_{\ell-1}, k_\ell} \right].
\]
where $k_{r+1}= n$. Write $\network'(x) = A_{n+1} Q(\networkleqn(x)) + b_{n+1}$.
Since $Q$ has Lipschitz constant 1 and $x \mapsto A_{n+1} x + b_{n+1}$ has Lipschitz constant $\| A_{n+1} \|_{\weightsys}^{n, n+1}$ by \cref{prop:lipaffine}, we can bound the Lipschitz constant of $\network'$ as $L_{\network'} \leq \| A_{n+1} \|_{\weightsys}^{n, n+1}L_{\networkleqn}$ by \cref{prop:lipcomp}. We can also write this as 
\[
L_{\network'}  \leq \frac1{2^{n-1}} \sum_{(k_1, k_2, \ldots, k_r) \in S_{n+1}: k_r = n}\;\left[\prod_{\ell=1}^{r+1} \| A_{k_\ell} \ldots A_{k_{\ell-1}+1} \|_{\weightsys}^{k_{\ell-1}, k_\ell} \right].
\]
where $k_{r+1}= n+1$. 

Note that 
\begin{align*}
    S_{n+1} &= \{(k_1, k_2, \ldots, k_r) \mid 0 \leq r \leq n, \, 1 \leq k_1 < k_2 < \dots < k_r \leq n\} \\
    &= S_n \cup \{(k_1, k_2, \ldots, k_r) \mid 0 \leq r \leq n, \, 1 \leq k_1 < k_2 < \dots < k_r \leq n, k_r = n\} \\
    &= S_n \cup \{(k_1, k_2, \ldots, k_r) \in S_{n+1} \mid k_r = n\}.
\end{align*}
Using \cref{prop:lipadd}, we find that
\begin{align*}
    L_{\network}  &\leq \tfrac12 L_{\networkleqnprime}  + \tfrac12 L_{\network'} \\ &\leq \frac1{2^{(n+1)-1}}\Bigg( \sum_{(k_1, k_2, \ldots, k_r) \in S_n}\;\left[\prod_{\ell=1}^{r+1} \| A_{k_\ell} \ldots A_{k_{\ell-1}+1} \|_{\weightsys}^{k_{\ell-1}, k_\ell} \right] \\ &\qquad\qquad \qquad  \qquad + \sum_{(k_1, k_2, \ldots, k_r) \in S_{n+1}: k_r = n}\;\left[\prod_{\ell=1}^{r+1} \| A_{k_\ell} \ldots A_{k_{\ell-1}+1} \|_{\weightsys}^{k_{\ell-1}, k_\ell} \right] \Bigg) \\ 
    &= \frac1{2^{(n+1)-1}}\sum_{(k_1, k_2, \ldots, k_r) \in S_{n+1}}\;\left[\prod_{\ell=1}^{r+1} \| A_{k_\ell} \ldots A_{k_{\ell-1}+1} \|_{\weightsys}^{k_{\ell-1}, k_\ell} \right],
\end{align*}
which completes the induction step and hence the proof. 
\end{proof}

\section{Experiments}
\label{app:Experiments}

In this appendix, we provide the complete dynamics and reach-avoid specifications for the benchmarks used for the empirical evaluation in \cref{sec:Empirical}.  In addition, we provide an overview of the hyperparameters used. 

For each of the models, we also provide the loss functions used by the RL algorithms to train the initial policy. To further test the robustness of our method against different input policies, we used different loss functions for PPO and the other RL algorithms from \texttt{Stable-Baselines3}.

\subsection{Model specifications}
\label{app:Models}

For simplicity of the implementation, we use a triangular noise distribution for each model. The triangular distribution on $[-1, 1]$ is the continuous distribution with probability density function
\[
\text{Triangular}(x) = \begin{cases} 1-|x| &\text{if } |x|< 1 \\ 0 &\text{otherwise}.\end{cases}
\]

\subsubsection{Linear System.}
\label{app:linsys}
In Linear System (\texttt{linear-sys}), we have $\X = [-1.5, 1.5]^2 \subseteq \R^2$ and $\U = [-1, 1] \subseteq \R$ and $\mathcal{N} = [-1, 1]^2 \subseteq \R^2$. The system dynamics function $f \colon \X \times \U \times \mathcal{N} \to \X$ is given by 
\[ f(\state, \control, \noise) = A \state + B \control + W \noise,\]
where $A = \begin{pmatrix} 1 & 0.045 \\ 0 & 0.9 \end{pmatrix}$ and $B = \begin{pmatrix} 0.45 \\ 0.5\end{pmatrix}$ and $W = \begin{pmatrix} 0.01 & 0 \\ 0 & 0.005 \end{pmatrix}$.
The noise $\noise$ has two components which are independent and have the Triangular distribution.

The initial states are $\X_0 = ([-0.25, -0.2] \cup [0.2, 0.25]) \times [-0.1, 0.1]$, the target states are $\xTarget = [-0.2, 0.2]^2$, and the unsafe states are $\xUnsafe = ([-1.5, -1.4] \times [-1.5, 0]) \cup ([1.4, 1.5] \times [0, 1.5])$.
The resulting reach-avoid task is shown in \cref{fig:tasks}.

\begin{figure}[t!]
\centering
\includegraphics[height=4.2cm]{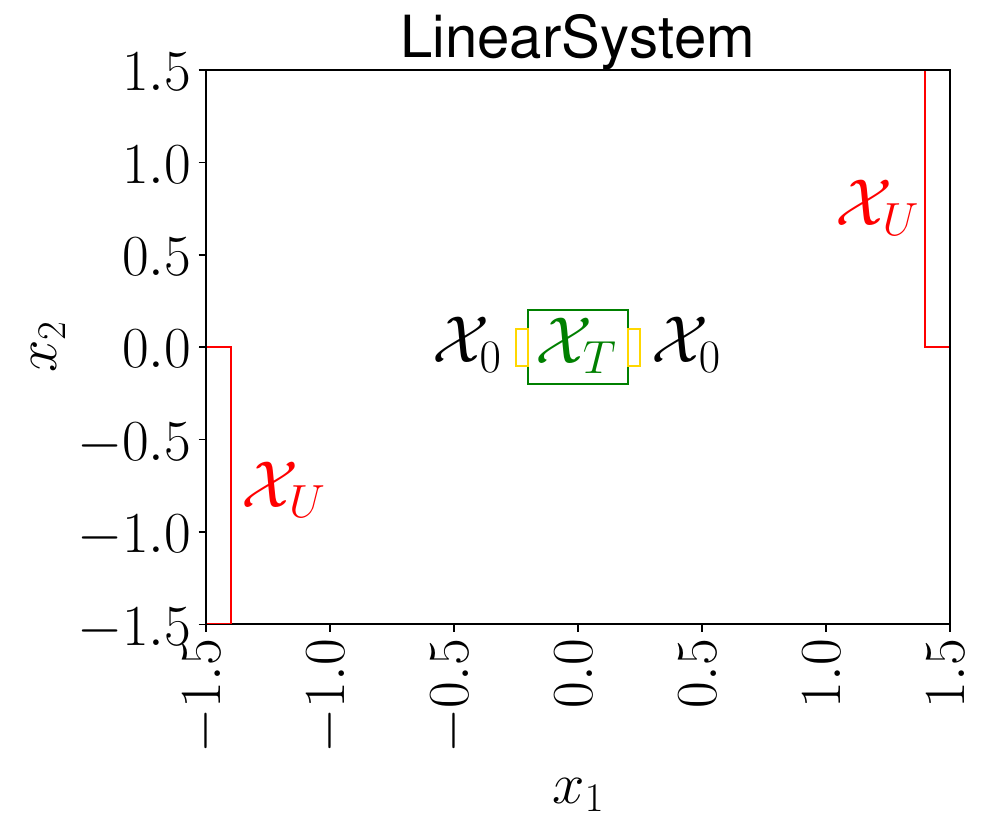}
\hspace{1cm}
\includegraphics[height=4.2cm]{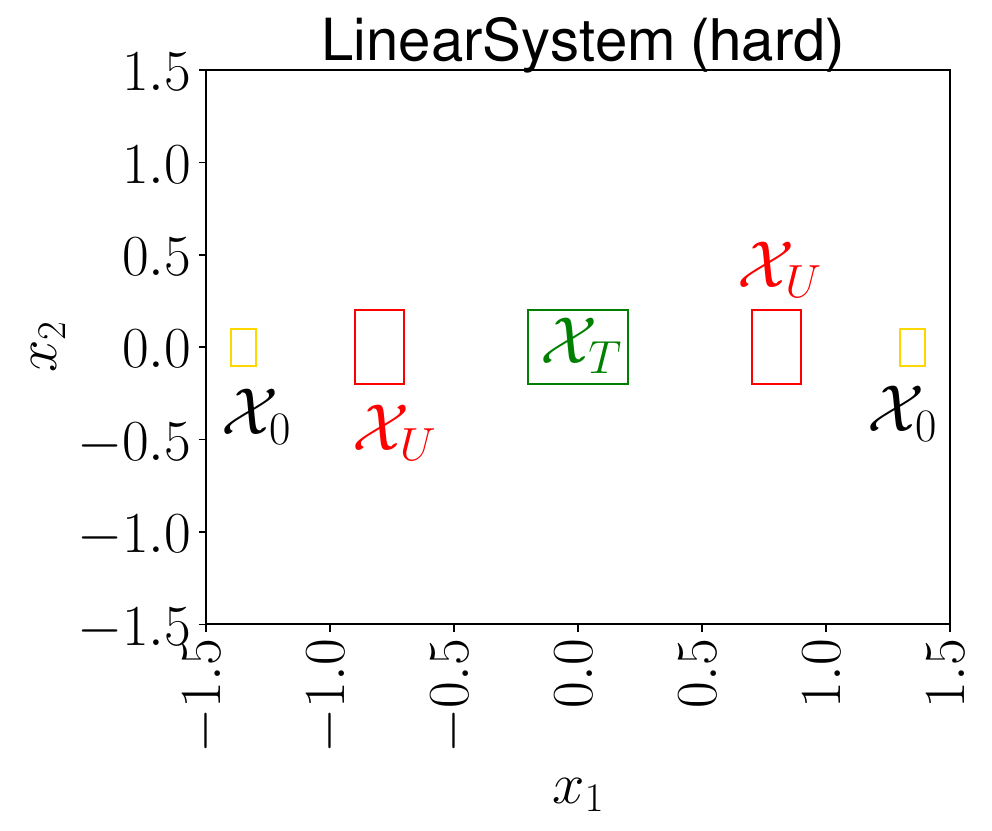}
\includegraphics[height=4.2cm]{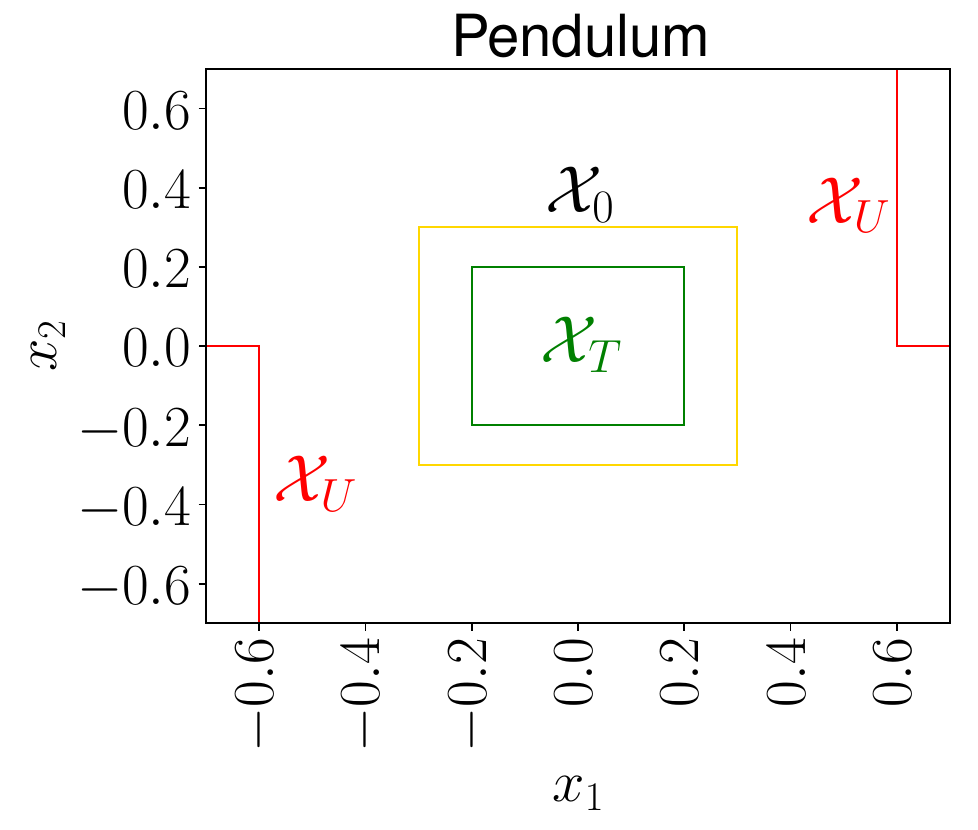}
\hspace{1cm}
\includegraphics[height=4.2cm]{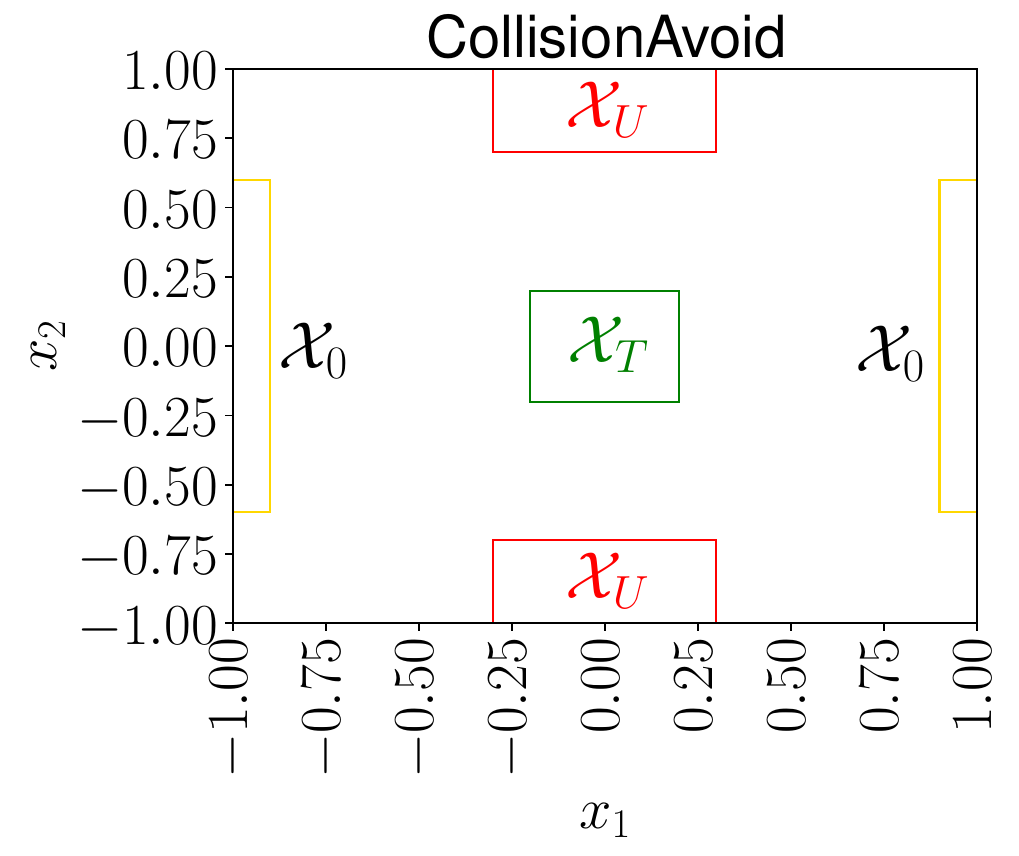}%
\caption{Reach-avoid specifications for the 2D benchmarks used for the experiments.}
\label{fig:tasks}
\end{figure}

The Lipschitz constants w.r.t.\ $\state$ and $\control$ are $L_{f, \state} = 1$ and $L_{f, \control} = 0.95$.

For all our main experiments, we use the loss function $\| \state\|_2 - 1 = \sqrt{\state_1^2 +\state_2^2} - 1$ to train input policies, which is a sensible loss function since the goal states are around the origin and the unsafe states are the farthest away from the origin. 
When training input policies with the \texttt{Stable-Baselines3} algorithms, we instead assign a loss of $5$ when entering the unsafe region and a loss of $-5$ when entering the goal region, and use the loss function $\| \state\|_2 - 1 = \sqrt{\state_1^2 +\state_2^2} - 1$ otherwise. 

\subsubsection{Pendulum.}
In Pendulum (\texttt{pendulum}), we have $\X = [-0.7, 0.7]^2 \subseteq \R^2$ and $\U = [-1, 1] \subseteq \R^2$ and $\mathcal{N} = [-1, 1]^2 \subseteq \R^2$. The system dynamics function $f \colon \X \times \U \times \mathcal{N} \to \X$ is given by 
\begin{align*} f&(\state, \control, \noise) = \begin{pmatrix} \state_1 + 0.01 \noise_1 \\ 0 \end{pmatrix} \\ &+ \begin{pmatrix} \delta \\ 1 \end{pmatrix} \text{clip}\left((1-b)\state_2 + \delta\left(\frac{-1.5 G \sin(\state_1 + \pi)}{2l} + \frac{6}{ml^2} \control\right) + 0.02\noise_2, -5, 5\right),\end{align*}
where $\delta = 0.05$, $G = 10$, $m = 0.15$, $l = 0.5$ and $b = 0.1$. 
The $\text{clip}$ function is defined as $\text{clip}(x,a,b) = \min(\max(x,a),b)$.
The noise $\noise$ has two components $\noise_1, \noise_2$ which are independent and have the Triangular distribution.

The initial states are $\X_0 = [-0.3, 0.3]^2$, the target states are $\xTarget = [-0.2, 0.2]^2$, and the unsafe states are $\xUnsafe = ([-0.7, -0.6] \times [-0.7, 0]) \cup ([0.6, 0.7] \times [0, 0.7])$.
The resulting reach-avoid task is shown in \cref{fig:tasks}.

The Lipschitz constants w.r.t.\ $\state$ and $\control$ are $L_{f, \state} = 1.7875$ and $L_{f, \control} = 8.4$.

For all our main experiments, we use the loss function $\state_1^2 + 0.1\state_2^2$ to train input policies. When training input policies with the \texttt{Stable-Baselines3} algorithms, we instead assign a loss of $5$ when entering the unsafe region and a loss of $-5$ when entering the goal region, and use the loss function $\state_1^2 + 0.1\state_2^2$ otherwise. 

\subsubsection{Collision Avoidance.}
In Collision Avoidance (\texttt{collision-avoid}), we have $\X = [-1, 1]^2 \subseteq \R^2$, $\U = [-1, 1]^2 \subseteq \R^2$ and $\mathcal{N} = [-1, 1]^2 \subseteq \R^2$. The system dynamics $f \colon \X \times \U \times \mathcal{N} \to \X$ is given by 
\begin{align*} f(\state, \control, \noise) &= \state + 0.2 \left(d_2 \left( d_1 \control + (1-d_1) \begin{pmatrix}0 \\ 1\end{pmatrix} \right) + (1-d_2)\begin{pmatrix}0 \\ -1\end{pmatrix}  \right) + 0.05 \noise,\end{align*}
where $d_1 = \min\left\{\dfrac{10}{3} \left\| \state - \begin{pmatrix}0 \\ 1\end{pmatrix} \right\|_2, 1\right\}$ and  $d_2 = \min\left\{\dfrac{10}{3} \left\| \state - \begin{pmatrix}0 \\ -1\end{pmatrix} \right\|_2, 1\right\}$.

The noise $\noise$ has two components $\noise_1, \noise_2$ which are independent and have the Triangular distribution.

The initial states are $\X_0 = ([-1, -0.9] \cup [0.9, 1]) \times [-0.6, 0.6]$, the target states are $\xTarget = [-0.2, 0.2]^2$, and the unsafe states are $\xUnsafe = [-0.3, 0.3] \cup ([-1, -0.7] \times [0.7, 1])$.
The resulting reach-avoid task is shown in \cref{fig:tasks}.

The Lipschitz constants w.r.t.\ $\state$ and $\control$ are $L_{f, \state} = 3$ and $L_{f, \control} = 0.2$.

We use the same loss functions for training input policies as for Linear System. 

\subsubsection{Linear System (hard layout).}
The hard layout of Linear System (\texttt{linear-sys}) has the same dynamics as the default version defined in \cref{app:linsys}, but we modify the reach-avoid specification.
Specifically, we use the reach-avoid specification with $\X_0 = ([-1.4, -1.3] \cup [1.3, 1.4]) \times [-0.1, 0.1]$ and $\xUnsafe = ([-0.9, -0.7] \cup [0.7, 0.9]) \times [-0.2, 0.2]$.
The resulting reach-avoid task is shown in \cref{fig:tasks}.

\begin{figure}[t!]
\centering
\includegraphics[height=4.2cm]{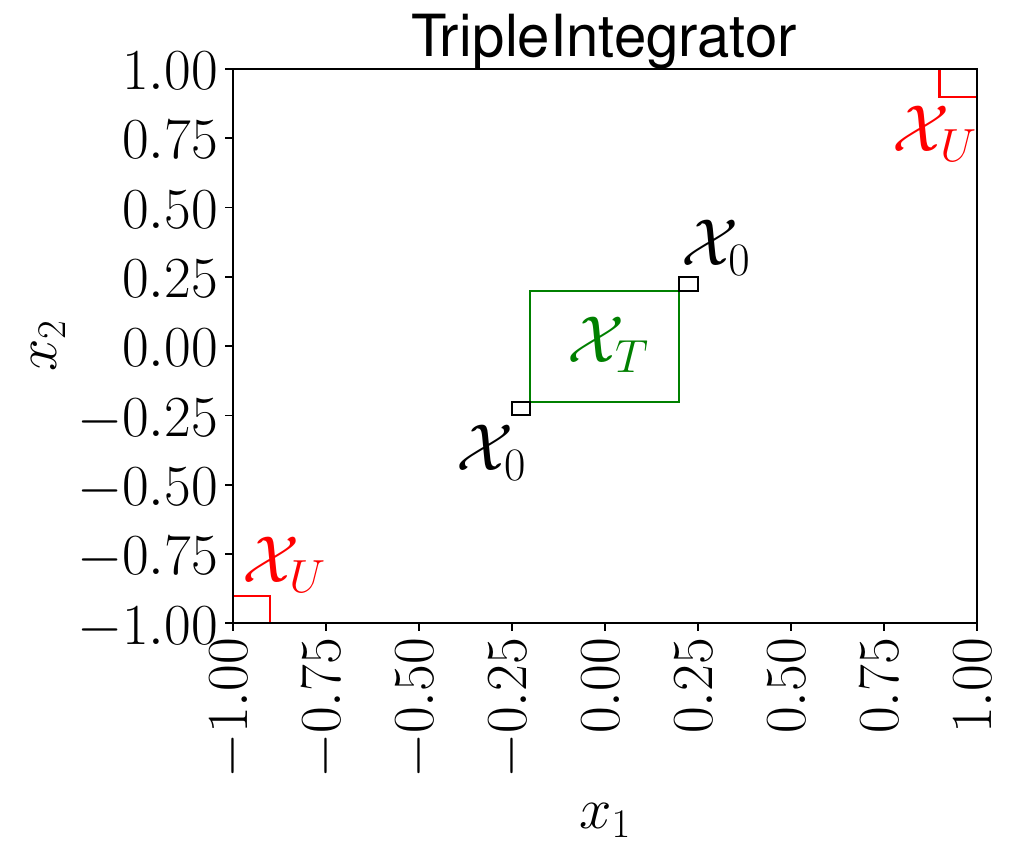}
\hspace{1cm}
\includegraphics[height=4.2cm]{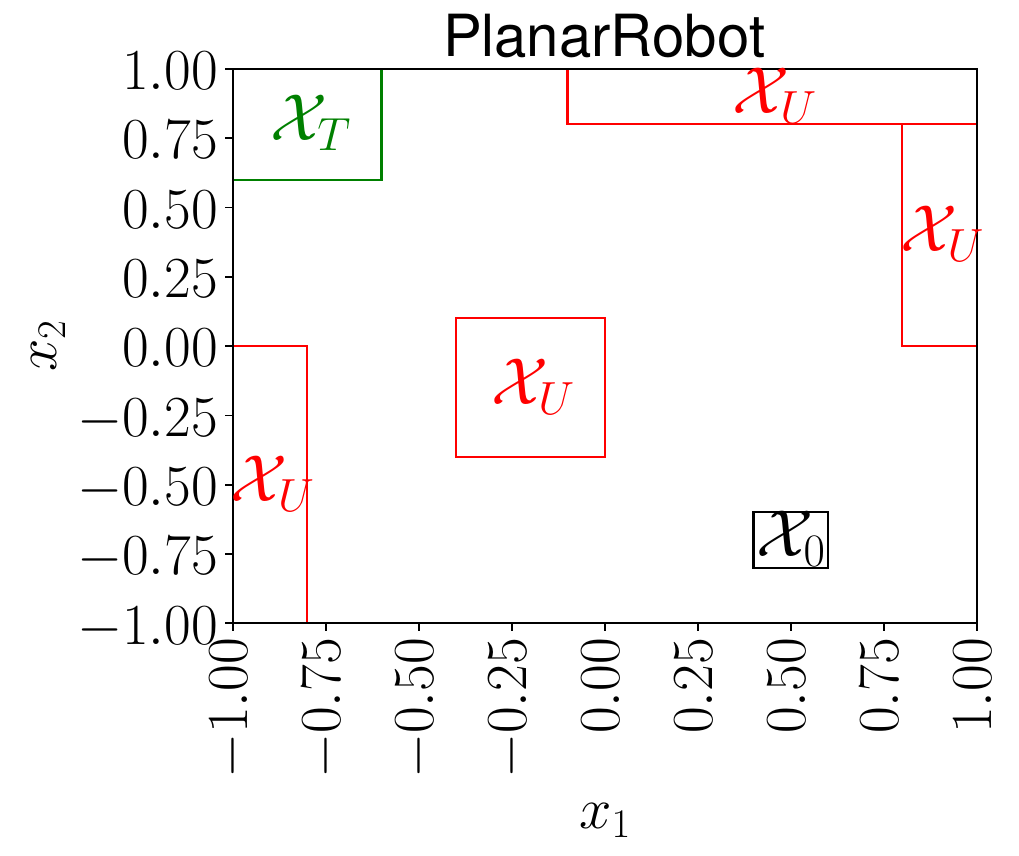}
\hspace{1cm}
\includegraphics[height=4.2cm]{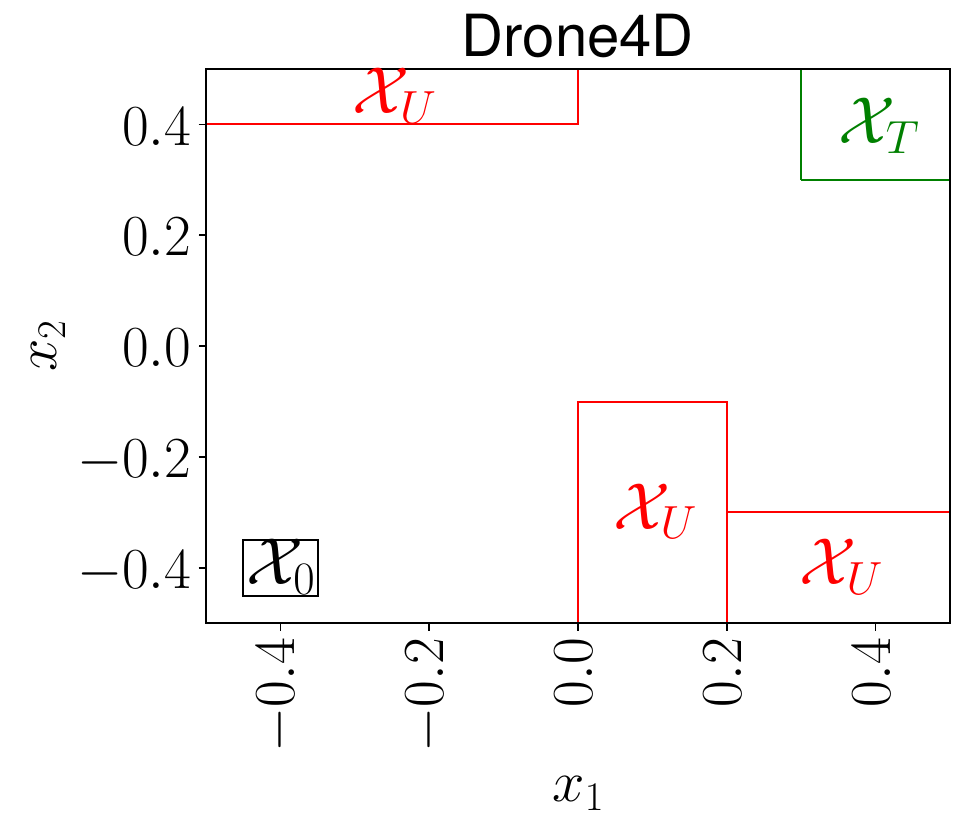}
\caption{Reach-avoid specifications for the 3D and 4D benchmarks used.}
\label{fig:challenging}
\end{figure}

\subsubsection{Triple Integrator.}
In Triple Integrator (\texttt{triple-integrator}), we have a 3D state space $\X = [-1, 1]^3 \subseteq \R^3$ and $\U = [-1, 1] \subseteq \R$ and $\mathcal{N} = [-1, 1]^3 \subseteq \R^3$. The system dynamics function $f \colon \X \times \U \times \mathcal{N} \to \X$ is given by 
\[ f(\state, \control, \noise) = A \state + B \control + W \noise,\]
where $A = \begin{pmatrix} 1 & 0.045 & 0 \\ 0 & 1 & 0.045 \\ 0 & 0 & 0.9  \end{pmatrix}$ and $B = \begin{pmatrix} 0.35 \\ 0.45 \\ 0.5 \end{pmatrix}$ and $W = \begin{pmatrix} 0.01 & 0 & 0 \\ 0 & 0.01 & 0 \\ 0 & 0 & 0.005 \end{pmatrix}$.

The noise $\noise$ has three components  which are independent and have the Triangular distribution.

The initial states are $\X_0 = ([-0.25, -0.2]^2 \cup [0.2, 0.25]^2) \times [-0.1, 0.1]$, the target states are $\xTarget = [-0.2, 0.2]^3$, and the unsafe states are $\xUnsafe = ([-1, -0.9]^2 \times [-1, 0]) \cup ([0.9, 1]^2 \times [0, 1])$. We show the 2D slice $\state_2 = 0$ of the reach-avoid specification in \cref{fig:challenging}. 

The Lipschitz constants w.r.t.\ $\state$ and $\control$ are is $L_{f, \state} = 1.045$ and $L_{f, \control} = 1.3$.

For training input policies, we use the loss function $\| \state\|_2^2 - 1 = \state_1^2 +\state_2^2 + \state_3^2 - 1$.

\subsubsection{Planar Robot.}
In Planar Robot (\texttt{planar-robot}), we have a 3D state space $\X = [-1, 1]^3 \subseteq \R^3$ and $\U = [-1, 1]^2 \subseteq \R^2$ and $\mathcal{N} = [-1, 1]^2 \subseteq \R^2$. The system dynamics function $f \colon \X \times \U \times \mathcal{N} \to \X$ is given by 
\[ f(\state, \control, \noise) = \state + \delta \begin{pmatrix} (\state_2+2\delta \control_1) \cos(\pi \control_2) \\  (\state_2+2\delta \control_1) \sin(\pi \control_2) \\ 2 \control_1 \end{pmatrix} + 0.01\begin{pmatrix} \omega_1 \\ \omega_2 \\ 0 \end{pmatrix},\]
where $\delta = 0.2$. The noise $\noise$ has two components which are independent and have the Triangular distribution.

The initial states are $\X_0 = [0.4, 0.6] \times [-0.8, -0.6] \times [-0.1, 0.1]$, the target states are $\xTarget = [-1, -0.6] \times [0.6, 1] \times [-1, 1]$, and the unsafe states are \begin{align*} \xUnsafe &= ([-1, -0.8] \times [-1, 0] \times [-1, 1]) \\ & \qquad \cup ([-0.1, 1] \times [0.8, 1] \times [-1, 1]) \\ & \qquad \cup ([0.8, 1] \times [0, 0.8] \times [-1, 1]) \\ &\qquad \cup ([-0.4, 0] \times [-0.4, 0.1] \times [-1, 1]). \end{align*} We show the 2D slice $\state_2 = 0$ of the reach-avoid specification in \cref{fig:challenging}. 

The Lipschitz constants w.r.t.\ $\state$ and $\control$ are is $L_{f, \state} = 1.4$ and $L_{f, \control} = 0.4\pi$.

For training input policies, we use the loss function 
\[10\sqrt{(\state_1 + 0.8)^2+(\state_2 - 0.8)^2} - 10 R.\] 
where $R=1$ if the goal is reached, and $R=0$ otherwise. This loss function measures the distance to the center of the target set (in the first two dimensions).

\subsubsection{Drone4D.}
In Drone4D (\texttt{drone4D}), we have a 4D state space $\X = [-0.5, 0.5]^4 \subseteq \R^4$ and $\U = [-0.5, 0.5]^2 \subseteq \R^2$ and $\mathcal{N} = [-1, 1]^2 \subseteq \R^2$. The system dynamics function $f \colon \X \times \U \times \mathcal{N} \to \X$ is given by 
\[ f(\state, \control, \noise) = \state + \delta \begin{pmatrix} \state_2 + \tfrac12 \delta \control_1 \\ -0.02 \state_2^3 + \control_1 \\ \state_4 + \tfrac12 \delta \control_2 \\ -0.01 \state_4^3 + \control_2 -0.1 \sin(\pi \state_1) \end{pmatrix} + 0.01\begin{pmatrix} 0 \\ \omega_1 \\ 0 \\ \omega_2 \end{pmatrix},\]
where $\delta = 0.5$. The noise $\noise$ has two components which are independent and have the Triangular distribution.

The initial states are $\X_0 = [-0.45, -0.35] \times [-0.1, 0.1] \times [-0.45, -0.35] \times [0.25, 0.35]$, the target states are $\xTarget = [0.3, 0.5] \times [-0.5, 0.5] \times [0.3, 0.5] \times [-0.5, 0.5]$, and the unsafe states are \begin{align*}\xUnsafe &= [0.2, 0.5] \times [-0.5, 0.5] \times [-0.5, 0.3] \times [-0.5, 0.5]\\ &\qquad \cup [0, 0.2] \times [-0.5, 0.5] \times [-0.5, 0.1] \times [-0.5, 0.5] \\ &\qquad \cup [-0.5, 0] \times [-0.5, 0.5] \times [0.4, 0.3] \times [-0.5, 0.5].\end{align*} We show the 2D slice $\state_2 = \state_4 = 0$ of the reach-avoid specification in \cref{fig:challenging}. 

The Lipschitz constants w.r.t.\ $\state$ and $\control$ are is $L_{f, \state} = 1.5$ and $L_{f, \control} = 0.5$.

For training input policies, we use the loss function 
\[\sqrt{(\state_1 + 0.4)^2+(\state_3 + 0.4)^2} - 10 R.\] 
where $R=1$ if the goal is reached, $R=-1$ if an unsafe state is entered, and 0 otherwise. This loss function measures the distance to the center of the target set (in the two dimensions corresponding to the drone position).

\subsection{Hyperparameters}
\label{app:Hyperparameters}

In this appendix, we give an overview of the hyperparameters of our algorithm. \Cref{tab:hyperparam} reports the hyperparameters. We now discuss the meaning of these parameters, the reason why we have chosen them, and the extent to which we have done hyperparameter tuning. 

\subsubsection{Hyperparameters for all benchmarks.}
We first discuss the hyperparameters common for all benchmarks, as mentioned in \cref{tab:hyperparam}. These parameters are at least used in all 2D benchmarks; for the more challenging benchmarks there are some exceptions mentioned later.

\paragraph{Samples and buffers.}
Recall from \cref{sec:Implementation} that the points used in the loss function of the learner are divided into randomly sampled points and counterexamples.
For efficiency reasons, our implementation keeps track of a buffer for both types of points and samples from these buffers when needed.
We use a buffer of size $90\,000$ for the random points and a buffer of size $30\,000$ for the counterexamples.
The number of epochs is the number of full passes made over these buffers in a single learner iteration. 
Within each epoch, the data is divided into batches of $4\,096$ points, and the counterexample fraction (we use $0.25$) gives the fraction of counterexamples in each batch.\footnote{Strictly speaking, each batch consists of 4096 points, $75\%$ of which is sampled from the buffer with random points, and $25\%$ of which is sampled from the counterexample buffer. Thus, it may happen that the same points are sampled more than once in an epoch, but given the size of the batches, this behavior only has a limited effect.} 
After each verifier iteration, a part of the counterexamples in the counterexample buffer is replaced by new counterexamples. 
The counterexample refresh fraction (we use $0.5$) gives the fraction of the counterexample buffer that is replaced with new counterexamples.
We did not tune these hyperparameters for our experiments.

\paragraph{Optimizer and learning rate.}
We take the optimizer (we use Adam~\cite{DBLP:journals/corr/KingmaB14}), learning rates for $V$ ($5 \cdot 10^{-4}$) and $\pi$ ($5 \cdot 10^{-5}$) from \cite{DBLP:conf/aaai/ZikelicLHC23} and have not performed any tuning on them. We also take the number of samples $N$ used in the expected decrease terms in the loss function of the learner (see \cref{sec:Implementation}) from \cite{DBLP:conf/aaai/ZikelicLHC23}.

\begin{table}[t]
\centering
\begin{tabular}{l|l} \toprule
number of random points & $90\,000$ \\
number of counterexamples & $30\,000$ \\
epochs & $25$ \\
batch size & $4096$ \\
counterexample fraction & $0.25$ \\
counterexample refresh fraction & $0.5$ \\
optimizer & Adam~\cite{DBLP:journals/corr/KingmaB14} \\
learning rate $V$ & $5 \cdot 10^{-4}$ \\
learning rate $\pi$ & $5 \cdot 10^{-5}$ \\
$N$ loss learner & 16 \\
pretraining steps (PPO) & $10^5$ \\
init. mesh verification grid & $0.01$ (for 2D), $0.04$ (for 3D), $0.06$ (for 4D) \\
max refine factor $C$ & $10$ (for 2D), $4$ (for 3D), $2$ (for 4D) \\
loss function: $\alpha$ & $10$ \\
loss function: $\varepsilon$ & $0.1$ \\
loss function: $\varepsilon'$ & $0.01$ \\
goal $L_\pi$ pretrain & $10$ \\
\bottomrule
\end{tabular}
\medskip
\caption{Hyperparameters common for all benchmarks, or only dependent on the dimension of the state space.}\label{tab:hyperparam}
\end{table}

\paragraph{PPO training.}
The number of steps for training input policies with PPO is $10^5$, except for \texttt{planar-robot} and \texttt{drone4D}, which use $10^7$ steps and $10^6$ steps, respectively.
We observed that this number of steps leads to adequate convergence of the policies. We did not perform  sophisticated tuning of the number of training steps.

\paragraph{Verifier grid and maximum refinement.}
We use $\tau_{\statedisc} = 0.01$ as initial mesh for each point $\statedisc$ in the discretization used in the verifier. We set the maximum refinement factor $C$ to 10, which seems to yield a good balance between (a) not having to refine too often and (b) not wasting too much time when hard violations are found after the first local refinement. 

\paragraph{Loss function.} Besides the loss mesh $\tau$, there are three hyperparameters in the loss function: the expected decrease multiplier $\alpha$, and the two terms $\varepsilon, \varepsilon'$. We set $\alpha$ to 10 (except one experiment for the robustness to input policies), since we observed that in general this gives a good balance between the three loss terms. The terms $\varepsilon, \varepsilon'$ ensure that the conditions amply hold for the points in the discretization if the loss is 0. We set a higher value ($\varepsilon = 0.1$) for the initial and unsafe conditions as these conditions are easier to satisfy and setting $\varepsilon$ to a higher value makes the algorithm more robust. On the other hand, the expected decrease condition is harder to satisfy, so we set a lower value $\varepsilon' = 0.01$. When setting this value too high, the loss function might introduce a larger margin on points that are no longer violations in favor of fixing violations. 

In addition,~\cite{DBLP:conf/aaai/ZikelicLHC23} also use losses penalizing the Lipschitz constant of the policy and the RASM candidate, and an auxiliary loss term that attempts to ensure that the global minimum of the candidate RASM is in the target set. We do not use these additional terms. The aim of the Lipschitz loss is to reduce the Lipschitz constant of (mainly) the certificate in fewer learner-verifier iterations. However, our method already substantially reduces the Lipschitz constant, so this additional loss has no substantial effect in combination with our method. 
In addition, the Lipschitz loss needs the desired Lipschitz constant as a hyperparameter, and setting it too small can actually hamper convergence. 
For the auxiliary loss, we observed that for high probability bounds this loss even worsens performance. 

\paragraph{Lipschitz constant pretraining.} The loss function used in PPO pretraining includes
an extra term that gives a loss if the Lipschitz constant of the policy is larger than
a prespecified value. We set this value to 10. For all learner-verifiers, we use the product upper bound from~\cite{DBLP:journals/corr/SzegedyZSBEGF13} to compute the Lipschitz bound, to ensure that all learner-verifiers use the same initial policy.

\subsubsection{Benchmark-specific hyperparameters.}
We now discuss the benchmark-specific hyperparameters, which can be divided into two categories.

\paragraph{Loss mesh.}
The main benchmark-specific parameter is the \emph{loss mesh}. 
Recall from \cref{sec:Implementation} that this loss mesh is contained in the loss function of the learner and is different from the mesh used by the verifier.
Intuitively, when the loss function is trained to zero for a particular loss mesh $\tau$, then a discretization of this same mesh $\tau$ in the verifier should suffice to successfully verify the certificate $V$.
In practice, the loss function is not trained to zero exactly, so we expect that the loss mesh should be higher than the (worst-case) mesh needed by the verifier.

We now explain the tradeoff for choosing a good loss mesh. A high loss mesh makes it harder to train, but easier to verify the certificate. Hence, we expect the learner-verifier to use more iterations, but the final verification to be fast. On the other hand, a lower loss mesh might lead to 0 loss fast, but lead to longer verification times since lower meshes are needed in the refinement. 

This tradeoff makes it difficult to select a good mesh. Therefore, we start with some initial loss mesh $\tau_{\text{init}}$ and decrease it by a factor 0.8 each iteration. In this way, it will become easier to learn a correct certificate as the learner-verifier iterations progress, at the cost of expecting higher verification times. 
For most 2D benchmarks we set  $\tau_{\text{init}} = 0.001$, but for \texttt{collision-avoid} we set $\tau_{\text{init}} = 0.01$.

\paragraph{Noise partition.}
Finally, we take the number of noise partition cells of $12^2 = 144$ used for computing upper bounds on expectations from~\cite{DBLP:conf/aaai/ZikelicLHC23}. For \texttt{collision-avoid}, which has relatively more noise,  we instead use $24^2 = 576$ cells. 

\subsubsection{Hyperparameters for $>2$D benchmarks.}
For the 3D and 4D benchmarks (\texttt{triple-integrator}, \texttt{planar-robot}, and \texttt{drone4D}), we have selected a number of hyperparameters differently. For \texttt{triple-integrator}, we increase the number of noise partition cells to $6^3 = 216$. We set the initial loss mesh $\tau_{\text{init}}$ to 0.005 for the 3D benchmarks and to 0.01 for \texttt{drone4D}, and refine with a factor 0.9 rather than a factor 0.8.

\subsection{Loss function RASMs} \label{app:lossrasm}

When the learner is learning a (discrete) RASM rather than a (discrete) logRASM, the learner uses a different loss function where each term instead models a differentiable version of a RASM condition:
\begin{align*}
    \loss_\Linit(V) &= \max_{\state \in P_0}\left\{\max\{V(\state) -1 + \varepsilon, 0\}\right\},
    \\ 
    \loss_\Lunsafe(V) &= (1-\rho) \max_{\state \in P_\Lunsafe}\left\{\max\big\{\tfrac{1}{1-\rho} - V(\state) + \varepsilon, 0\big\}\right\},
    \\
    \loss_\Lexp(\pi, V) &= \frac{1}{|P_\Lexp|} \!\sum_{\state \in P_\Lexp} \!\max\!\bigg\{
         \!\bigg[\frac1N\!\sum_{\noise_i \sim d} V(f(\state, \policy(\state), \noise_i))\bigg] {-} V(\state) {+} \tau K' {+} \varepsilon'
    , 0  \bigg\}.
\end{align*}
We use the same hyperparameters in this loss function.

\setlength{\tabcolsep}{2.2pt}
\begin{table}[t]
\caption{Lipschitz constants and timings for LipBaB in comparison to our method, averaged over 10 seeds (except any seeds that failed), for the policy network ($\policy$) and the certificate network ($V$). For LipBaB, we used a timeout of 600 seconds.}\label{lipbab}
\centering
\begin{tabular}{l|ll|rrrr|rr}
\toprule
model & $\rho$ & network & $L_{\text{Ours}}$ & $L_{\text{LipBaB}, 1}$ &  $L_{\text{LipBaB}}$ & lower & $t_{\text{better}}$ & $t_{\text{exact}}$ \\ \midrule
\multirow{2}{*}{\texttt{linear-sys}} & 0.8 & $\pi$ & 2.56 & 4.18 & 1.28 & 0.71 & 162.9 & \textgreater 600 \\
 & 0.999999 & $\pi$ & 2.56 & 4.18 & 1.29 & 0.71 & 162.6 & \textgreater 600 \\ \midrule
\multirow{2}{*}{\begin{tabular}{@{}l@{}}\texttt{linear-sys} \\ (hard layout)\end{tabular}} & 0.8 & $\pi$ & 2.52 & 4.15 & 1.27 & 1.05 & 159.2 & \textgreater 600 \\
 & 0.999999 & $\pi$ & 1.67 & 2.53 & 0.88 & 0.88 & 101.1 & 435 \\ \midrule
\multirow{2}{*}{\texttt{pendulum}} & 0.8 & $\pi$ & 0.70 & 0.92 & 0.43 & 0.39 & 75.1 & \textgreater 600 \\
 & 0.999999 & $\pi$ & 0.47 & 0.54 & 0.31 & 0.31 & 35.1 & \textgreater 588 \\ \midrule
\multirow{2}{*}{\texttt{collision-avoid}} & 0.8 & $\pi$ & 6.06 & 8.70 & 3.66 & 1.89 & 106.0 & \textgreater 600 \\
 & 0.999999 & $\pi$ & 5.25 & 6.92 & 3.62 & 3.17 & 90.4 & \textgreater 600 \\
\midrule
\multirow{2}{*}{\texttt{linear-sys}} & 0.8 & $V$ & 39.14 & 94.86 & 56.91 & 5.63 & \textgreater 600.0 & \textgreater 600 \\ 
 & 0.999999 & $V$ & 83.74 & 176.8 & 106.15 & 15.23 & \textgreater 600.0 & \textgreater 600 \\ \midrule
\multirow{2}{*}{\begin{tabular}{@{}l@{}}\texttt{linear-sys} \\ (hard layout)\end{tabular}} & 0.8 & $V$ & 93.78 & 148.05 & 85.08 & 21.39 & 471.2 & \textgreater 600 \\
 & 0.999999 & $V$ & 545.23 & 776.56 & 420.58 & 91.70 & 281.3 & \textgreater 600 \\ \midrule
\multirow{2}{*}{\texttt{pendulum}} & 0.8 & $V$ & 10.95 & 13.90 & 9.26 & 6.91 & 240.0 & \textgreater 600 \\
 & 0.999999 & $V$ & 120.06 & 126.73 & 81.23 & 66.78 & 15.9 & \textgreater 600 \\ \midrule
\multirow{2}{*}{\texttt{collision-avoid}} & 0.8 & $V$ & 8.05 & 9.89 & 7.43 & 6.19 & 367.0 & \textgreater 600 \\
 & 0.999999 & $V$ & 65.4 & 76.61 & 52.63 & 40.65 & 139.2 & \textgreater 600 \\
\bottomrule
\end{tabular}
\end{table}

\section{Comparison against LipBaB}\label{app:lipbab}
We compare our method for computing Lipschitz constants to the anytime algorithm LipBaB \cite{DBLP:conf/icann/BhowmickDR21}, a competitive algorithm for computing global Lipschitz constants. We use the final networks returned upon termination of the learner-verifier framework (using our method).
We then compute the Lipschitz constant for the policy network and the certificate network using our method and LipBaB.

\Cref{lipbab} shows results for Lipschitz computations using LipBaB, in comparison to our method (using both weighted norms (\cref{subsec:weighted}) and averaged activation operators (\cref{app:averagedactivationoperators})).  The first two columns specify the network we trained on by specifying the benchmark, probability bound $\rho$, and whether we consider the policy network ($\policy$) or the certificate network ($V$). The next column gives the Lipschitz constant computed using our method. For each setting, our method takes 0.2 seconds for the first call (due to JIT compilation by JAX) and 0.0002 seconds (0.2 milliseconds) for subsequent calls (if the code is already JIT-compiled). The column titled `$L_{\text{LipBaB}, 1}$' gives the first result computed by LipBaB.  Computing this first result takes 0.5 seconds for 3 layer networks. The next two columns (titled `$L_{\text{LipBaB}}$' and `lower') provide the final Lipschitz constant computed by LipBaB upon termination or after 600 seconds (whichever occurs first). The column titled `$t_{\text{better}}$' gives the first time after which LipBaB has computed a better Lipschitz constant than our method. The final column, titled `$t_{\text{exact}}$', gives the time until LipBaB returns the exact Lipschitz constant (counting a timeout as 600 seconds, and an inequality sign indicating that in at least one run the exact Lipschitz constant was not found in 600 seconds). %

\fi

\end{document}